\PassOptionsToPackage{table,xcdraw,dvipsnames}{xcolor}
 
\documentclass{article} %
 
\usepackage{iclr2027_conference,times}

\usepackage[utf8]{inputenc}
\usepackage[T1]{fontenc}
\usepackage[USenglish]{babel}

\usepackage{microtype}
\usepackage{xcolor}
\usepackage{graphicx}
\usepackage{wrapfig}
\usepackage{booktabs}
\usepackage{tabularx}
\usepackage{multirow}
\usepackage{caption}
\usepackage{subcaption}
\usepackage{enumitem}
\usepackage{csquotes}

\usepackage{amsmath,amsfonts,amssymb,amsthm}
\usepackage{mathtools}
\usepackage{bm}
\usepackage{bbm}
\usepackage{dsfont}
\usepackage[mathscr]{eucal}
\usepackage{nicefrac}
\usepackage{cancel}
\usepackage{centernot}

\usepackage{algorithmic}
\usepackage[linesnumbered,ruled]{algorithm2e}

\usepackage[toc,page]{appendix}
\usepackage{tcolorbox}
\usepackage{mdframed}
\usepackage{tikz}
\usepackage{circledsteps}

\usepackage{natbib}

\usepackage{todonotes}   %
\usepackage{hyperref}
\usepackage{url}
 
\hypersetup{
    colorlinks=true,
    linkcolor=blue,
    citecolor=green!80!black,
    filecolor=magenta,
    urlcolor=cyan
}
 
\usepackage[capitalize,noabbrev]{cleveref}

\usepackage{thmtools}

\RenewCommandCopy{\theHproposition}{\theproposition}
\RenewCommandCopy{\theHhyp}{\thehyp}
\RenewCommandCopy{\theHlem}{\thelem}
\RenewCommandCopy{\theHprop}{\theprop}
\RenewCommandCopy{\theHremark}{\theremark}
 
\crefname{proposition}{Proposition}{Propositions}
\Crefname{proposition}{Proposition}{Propositions}
 
\crefname{hyp}{Hypothesis}{Hypotheses}
\Crefname{hyp}{Hypothesis}{Hypotheses}
 
\crefname{lem}{Lemma}{Lemmas}
\Crefname{lem}{Lemma}{Lemmas}
 
\crefname{prop}{Proposition}{Propositions}
\Crefname{prop}{Proposition}{Propositions}
 
\crefname{remark}{Remark}{Remarks}
\Crefname{remark}{Remark}{Remarks}

\makeatletter
\def\@captype{table}
\makeatother

\newcommand{\bmf}[1]{\bm{\mathsf{#1}}}

\newcommand{\ve}{\bmf{e}}

\newcommand{\vh}{\bmf{h}}

\newcommand{\vq}{\bmf{q}}
\newcommand{\vw}{\bmf{w}}

\newcommand{\vz}{\bmf{z}}

\definecolor{cgpt}{HTML}{00E079}
\definecolor{cclaude}{HTML}{9B4400}

\newcommand{\dfv}{\texttt{DataInf}}

\newcommand{\hypi}{\texttt{HyperINF}}

\title{Learning Dynamics of Continual Learning: \\A Unified View of Data Attribution, Forgetting, and Plasticity Loss}

\author{
  \textbf{Yi Ren}$^{1,2}$ \quad
  \textbf{Wenlong Deng}$^{3}$ \quad
  \textbf{Guanzhe Hong}$^{2}$ \quad
  \textbf{Clare Lyle}$^{4,\dagger}$ \quad
  \textbf{Yarin Gal}$^{1,2}$ \\
  $^1$OATML \quad
  $^2$University of Oxford \quad
  $^3$UBC \quad
  $^4$Google DeepMind \\
}

\iclrfinalcopy %
\begingroup

\footnotetext{$^\dagger$ Participated only in an advisory capacity.}
\addtocounter{footnote}{-1}
\endgroup

\begin{document}

\maketitle

\begin{abstract}
Modern language models are likely to be updated throughout their lifetime rather than trained once and frozen.
Each update therefore participates in a recurring cycle: decide which experience to learn from, understand what that update changes, and remain capable of learning from what comes next.
We show that these challenges are governed by the same evolving update--behavior interaction.
We derive a token- and layer-wise decomposition of how learning from one token changes another prediction.
By separating the softmax force, shared readout geometry, and residual connections, it exposes two interaction channels and yields a forward-computable approximation.
Following this interaction through time reveals a unified picture of continual adaptation.
Positive interaction identifies useful experience; negative interaction produces either concentrated \emph{collision} or accumulated \emph{erosion}; over longer horizons, updates reshape the shared geometry mediating future learning signals, reducing their transmission.
These predictions lead to effective data selection, mechanism-specific controls for interference, and a readout-based diagnostic of future learnability whose degradation predicts the benefit of restoring the readout.
Across models and training regimes, the same local interaction thus explains both what an update changes now and how learning today changes what can be learned tomorrow.
This view connects data attribution, forgetting, and plasticity loss as distinct regimes of the same evolving learning dynamics.
\noindent\textbf{Project page:} \url{https://joshua-ren.github.io/learning-dynamics-cl/}
\end{abstract}

\section{Introduction}
\label{sec:intro}

Modern language models are unlikely to remain fixed after deployment. 
They will continue to accumulate experience from users, feedback, tool use, retrieved documents, domain-specific data, and model-generated trajectories \citep{ouyang2022training, sun2024learning, yin2025godel}. 
The challenge is therefore no longer only to train a capable model once, but to build a learner that can repeatedly turn new experience into useful updates over its lifetime.

Every such update raises three tightly coupled questions. 
\emph{Which experience should the model learn from?}
\emph{What existing behavior will change when it learns from that experience?}
And after many such updates, \emph{will the model remain able to learn effectively from what comes next?} 
These questions arise whenever experience is consolidated into model parameters, even if the surrounding system also uses external memory, in-context adaptation, or other faster learning mechanisms \citep{lewis2020retrieval, olsson2022context, behrouz2026nested}.

We argue that these requirements should not be studied as independent problems.
They are different views of the same evolving learning process.
At any time $t$, learning from an experience $u$ changes some existing behavior $o$.
We denote this local \emph{update--behavior interaction} by $\Delta_t(o,u)$.
Before an update, it indicates whether learning from \(u\) is likely to benefit a desired behavior \(o\). 
During adaptation, negative interactions identify which existing behaviors \(o\) are disrupted and which updates \(u\) are responsible for the interference.
Over longer horizons, the updates themselves reshape the model geometry that mediates future interactions, changing how effectively later learning signals propagate.
Continual learning therefore requires understanding not only what an update changes, but also how learning changes the learner itself. \cref{fig:system}(a) summarizes this view.

\begin{figure}[t]
    \vspace{-1em}
    \centering
    \includegraphics[width=0.9\linewidth]{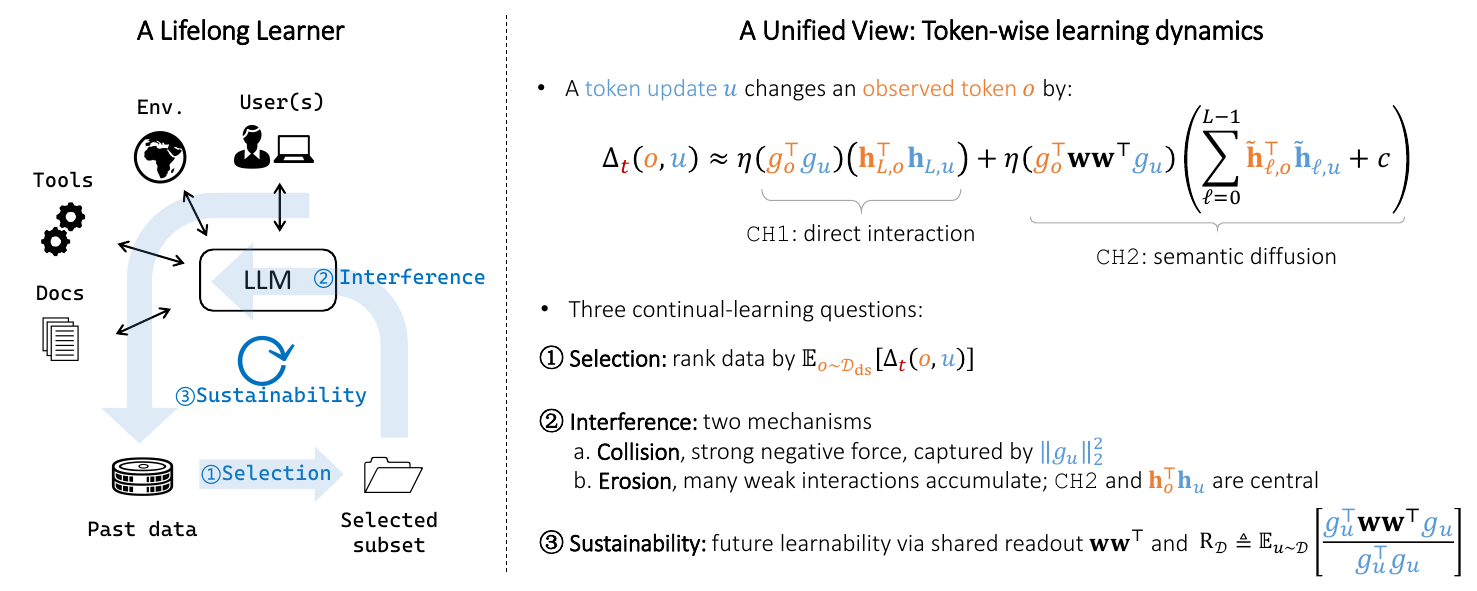}
    \vspace{-1em}
    \caption{A unified view of continual adaptation through the evolving update--behavior interaction underlying selection, interference, and future learnability.
    }
    \label{fig:system}
    \vspace{-1em}
\end{figure}

A natural starting point is the learning-dynamics framework of \citet{ren2025learning}, which studies how an update induced by one example changes the model's behavior on another.
However, a coarse example-level interaction is insufficient for continual adaptation: we need to know which output directions interact, how their learning signals propagate, and which parts of the interaction geometry evolve during training.
Exact evaluation also requires full-parameter Jacobians, making repeated measurement impractical at LLM scale.

We therefore develop a structured, token-level view of the same learning dynamics.
By explicitly separating the softmax learning signal, the shared readout, and the residual backbone, we expose two distinct pathways through which one token update can affect another behavior.
The first is a direct, token-aligned interaction in vocabulary space.
The second is a broader interaction mediated by the shared readout geometry and propagated through the residual stream.
Approximating the dominant residual path further turns this decomposition into a forward-computable quantity based only on logits, hidden states, and the shared readout.
This lets us probe the interaction repeatedly as the model adapts, rather than treating the interaction geometry as fixed over a training trajectory.

Following this interaction through the lifetime of the learner reveals several qualitatively different regimes, as illustrated in \cref{fig:system}.
\emph{Before learning}, positive interactions identify experience that is predicted to improve a desired behavior.
\emph{During learning}, negative interactions determine how existing behavior is disrupted.
Our decomposition reveals two distinct mechanisms: \emph{collision}, where a small number of strongly conflicting forces produce concentrated interference, and \emph{erosion}, where many individually weak interactions accumulate coherently over training.
\emph{After prolonged learning}, the interaction geometry itself changes.
In particular, the shared readout that transmits output-side learning signals into the backbone can become less responsive along directions required by future tasks, reducing the model's ability to adapt.

These regimes lead to distinct and testable predictions.
For selection, the interaction score directly ranks candidate experience by its predicted benefit, and improves both memory retrieval and downstream learning when the selected examples are actually used for adaptation.
For interference, the two mechanisms require different controls: filtering extreme-energy updates mitigates collision, whereas separating the response format of incoming supervision reduces erosion without preventing the new task from being learned.
For long-term adaptability, we derive a task-conditioned measure of readout transmission and find that it deteriorates together with future learnability during long-horizon training.
Restoring the readout partially recovers this lost plasticity, and the amount of transmission degradation predicts the benefit of the intervention.

Viewed through conventional terminology, the three stages in \cref{fig:system} connect to problems usually studied separately as \emph{data attribution}, \emph{forgetting}, and \emph{plasticity loss}, respectively
\citep{koh2017understanding,kirkpatrick2017overcoming,lyle2023understanding-82a}.
Our perspective places them within a single learning process: which update to make, what that update changes, and how accumulated updates affect the model's ability to learn from future experience.

\section{Structured Learning Dynamics of Continual Adaptation}
\label{sec:learning_dynamics}

\subsection{Token-Level Update--Behavior Interaction}
\label{sec:learning_dynamics:01}

Consider an LLM parameterized by $\theta$, with next-token distribution $\pi_\theta(\cdot\mid s)\in\mathbb{R}^V$.
Let $u=(s_u,y_u)$ and $o=(s_o,y_o)$ denote the updating and observing token--context pairs, respectively, where $s$ is the context and $y$ the target token.
We study how learning from $u$ changes the model's confidence in $o$:
\begin{equation} 
    \Delta_t(o, u)\triangleq \log \pi_{\theta_{t+1}}(y_o\mid s_o) - \log \pi_{\theta_{t}}(y_o\mid s_o), \quad
    \label{eq:delta_pi}
\end{equation}
For autoregressive sequences, these token-level changes add across output tokens,
and example-level interactions follow by aggregation over token pairs.
The token-wise view preserves localized interactions that sequence-level averaging can obscure.
Considering one gradient step on the token-level negative log-likelihood, $\theta_{t+1}-\theta_t = \eta\nabla_\theta\log\pi_{\theta_t}(y_u\mid s_u)$.
A first-order expansion gives
\begin{equation}
    \Delta_t(o,u) \approx \langle
        \nabla_\theta\log \pi_{\theta_t}(y_o\mid s_o), \theta_{t+1}-\theta_t
    \rangle
    =\eta\langle
        \nabla_\theta\log \pi_{\theta_t}(y_o\mid s_o), \nabla_\theta\log \pi_{\theta_t}(y_u\mid s_u)
    \rangle,
    \label{eq:def_confidence_change}
\end{equation}
This one-step gradient interaction is the basic object we study throughout the paper.
More general finetuning objectives can locally be expressed as weighted combinations
of token-level log-probability gradients, so the same model-dependent interaction
applies term-wise with objective-specific weights.

\paragraph{Exposing the token force.}
Following the learning-dynamics perspective of \citet{ren2025learning}, we now open the output side of this interaction by explicitly separating the final softmax:
\[
    s\xrightarrow{\text{LLM Blocks } \theta} \vz \xrightarrow{\sigma(\cdot)} \pi.
\]
For a target token $y$,
\[
    \nabla_{\vz} \log \pi_\theta(y\mid s)
    =
    \ve_y-\pi_\theta(\cdot\mid s)\triangleq g(y,s)\in\mathbb{R}^{V\times 1}
\]
where $g(y,s)\in\mathbb{R}^V$ is the signed output-side learning signal,
which we call the \emph{token force}.
Writing $g_u=g(y_u,s_u)$ and $g_o=g(y_o,s_o)$, the chain rule yields
\begin{equation}
    \Delta_t(o,u) \approx
    \eta
    g_o^\top \mathcal{K}_t(s_o,s_u) g_u;\quad \mathcal{K}_t(s_o,s_u) = \nabla_\theta \vz_o\ \nabla_\theta \vz_u^\top.
    \label{eq:def_confidence_change_2}
\end{equation}
This is the token-level specialization of the learning-dynamics interaction:
the updating force $g_u$ is propagated through the model-dependent kernel $\mathcal{K}_t$ and measured along the observing force $g_o$.
The kernel, however, still treats the model as a black box.
To expose how this interaction is realized inside a modern LLM, we next separate the shared readout from the residual backbone.

\begin{figure}[t]
    \centering
    \includegraphics[width=1\linewidth, trim=0 0 0 0, clip]{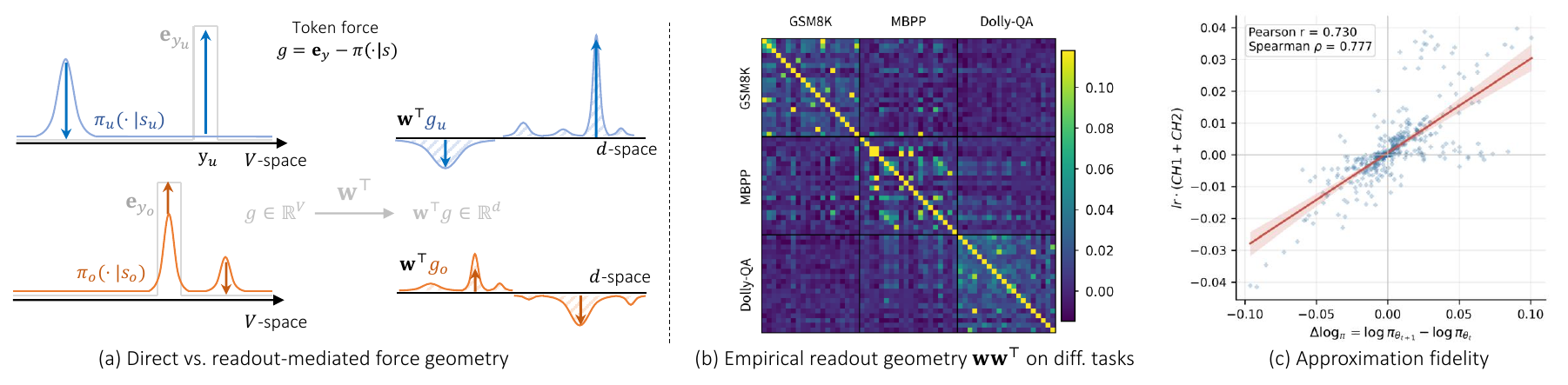}
    \caption{
            \textbf{Geometry and fidelity of the two-channel approximation.}
            (a) Direct token-force geometry in \texttt{CH1} versus readout-mediated geometry in \texttt{CH2}.
            (b) Empirical readout geometry $\vw\vw^\top$ across task-related tokens, showing substantial off-diagonal coupling, measured on Qwen2.5-1.5B.
            (c) The forward-computable approximation tracks the measured one-step change $\Delta_t(o,u)$.
    }
    \label{fig:gg_alignment}
\end{figure}

\subsection{Opening the Model: Readout and Residual Flow}
\label{sec:learning_dynamics:02}

\paragraph{Separating the readout from the backbone.}
The kernel in \cref{eq:def_confidence_change_2} still compresses the entire network into a single interaction operator.
We therefore further expose the model structure as
\[
    s\xrightarrow{f(s;\phi)} \vh \xrightarrow{\vw} \vz \xrightarrow{\sigma(\cdot)} \pi,
\]
where $\vh\in\mathbb{R}^d$ is the final hidden representation,
$\vw\in\mathbb{R}^{V\times d}$ is the shared linear readout, and $\theta=(\vw,\phi)$.
Separating the parameter gradients of the readout and backbone gives
\begin{equation}
    \Delta_t(o,u)
    \approx
    \eta
        \Big[
        \underbrace{
        (g_o^\top g_u)
        (\vh_o^\top \vh_u)
        }_{\text{readout contribution}} + 
        \underbrace{
        (\vw^\top g_o)^\top
        J_o J_u^\top
        (\vw^\top g_u)
        }_{\text{backbone contribution}}
        \Big]
\label{eq:a_kkk_g}
\end{equation}
where $J_o=\nabla_\phi \vh_o$ and $J_u=\nabla_\phi \vh_u$.
This decomposition is exact within the first-order approximation:
the first term arises from updating the shared readout $\vw$, whereas the second
captures how the projected forces $\vw^\top g_o$ and $\vw^\top g_u$ interact through
the backbone.

\begin{figure}[t]
    \centering
    \includegraphics[width=1\linewidth, trim=20 20 40 20, clip]{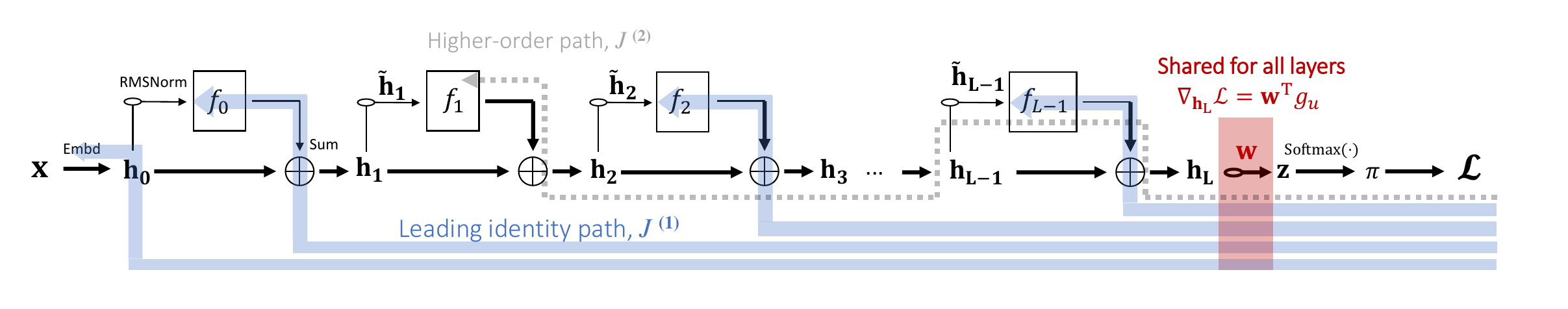}
    \caption{\textbf{Residual architecture and identity-path approximation.}
            All layers share the output-side signal $\vw^\top g$; we retain the direct residual-stream path $J_\ell^{(1)}$ and omit higher-order paths.
    }
    \vspace{-8pt}
    \label{fig:residual_network}
\end{figure}

The remaining obstacle is computational.
Evaluating $J_oJ_u^\top$ directly requires full-parameter Jacobians.
We instead exploit the residual structure, i.e., $\vh_{\ell+1} = \vh_{\ell} + f_\ell(\tilde{\vh}_{\ell};\phi_\ell)$ and $\tilde{\vh}_{\ell} = \mathsf{RMSNorm}(\vh_{\ell})$.
As illustrated in \cref{fig:residual_network}, the learning signal from the output
can reach an earlier block through the direct residual stream or through paths
that traverse one or more residual branches.
We retain the leading identity path through the residual stream and locally linearize each residual block.
Under these approximations, the contribution of each block reduces to the
similarity between its forward hidden representations.
In summary, we have the following proposition:
\begin{restatable}{proposition}{propscore}
    \label{prop:score}
    Under the identity-path and linear-block approximations above,
    \begin{equation}
        \Delta_t(o,u)\approx 
        \underbrace{\eta(g_o^\top g_u)(\tilde{\vh}_{L,o}^\top\tilde{\vh}_{L,u})}_\texttt{CH1} +
        \underbrace{\eta(\vw^\top g_o)^\top(\vw^\top g_u)
                        \left(\sum_{\ell=0}^{L-1}\tilde{\vh}_{\ell,o}^\top\tilde{\vh}_{\ell,u} +
                        \kappa_\text{embd}(s_o,s_u)
                        \right)}_\texttt{CH2},
    \label{eq:score}
    \end{equation}
    where $\kappa_\text{embd}(s_o,s_u)=\sum_{i,j}\mathbf{1}[s_{o,i}=s_{u,j}]$ denotes token overlap between the contexts $s_o$ and $s_u$.
\end{restatable}
It replaces the backbone Jacobians with forward-pass quantities;
full derivation in Appendix~\ref{app:proofs:delta}.

\textbf{Direct and diffuse interaction channels.}
The two terms expose qualitatively different interaction geometries, as demonstrated in \cref{fig:gg_alignment}(a).
In \texttt{CH1}, the output-side interaction is the direct force alignment $g_o^\top g_u$.
Because token forces are typically concentrated on a small number of vocabulary directions, \texttt{CH1} is dominated by relatively sparse, token-aligned interactions.
By contrast, \texttt{CH2} couples the forces through the shared readout geometry, $g_o^\top \vw\vw^\top g_u$.
The off-diagonal structure of $\vw\vw^\top$ allows different vocabulary directions to interact, so \texttt{CH2} can transmit influence even when direct force overlap is weak.
Thus, \texttt{CH1} captures concentrated direct alignment, whereas \texttt{CH2} provides a broader, more diffuse interaction channel.
Hidden-state similarity modulates the strength of both channels according to context.

\textbf{The shared readout broadens and transmits the interaction.}
Moreover, in \texttt{CH2}, both forces are projected through the same readout matrix $\vw$, giving $g_o^\top \vw\vw^\top g_u$.
Owing to the residual structure, all gradient paths to the residual blocks share the same output-side projection,
$\nabla_{\vh_L}\log \pi(y\mid s)=\vw^\top g$.
Hence, \(\|\vw^\top g\|_2^2 = g^\top \vw\vw^\top g\) measures how strongly an output-side learning signal is transmitted into the hidden space.
Because the readout is shared across residual blocks and evolves during training, the same geometry that broadens local interactions also governs how effectively future learning signals reach the parameters of each layer.

\paragraph{Approximation fidelity.} These structure-derived properties will recur throughout the following sections.
Many of our analyses depend primarily on the structural form of the interaction, rather than on exact numerical reconstruction of every local update: the direct and readout-mediated pathways arise from standard shared-readout and residual architectures, and their structural origin is preserved under common optimizer variants.
We examine approximation fidelity, cold-start behavior, and extensions beyond the canonical setting in Appendix~\ref{app:diagnose}.

\section{Selecting Beneficial Updates}
\label{sec:attribution}

The first question in continual adaptation is which experience should enter the next update.
Our interaction provides a direct criterion: an experience $u$ is useful when learning from it is predicted to improve the desired behavior represented by $o$.
Given a candidate pool $\mathcal D_{\mathrm{pool}}$ and a small probe set $\mathcal D_{\mathrm{ds}}$ representing the target behavior, we aggregate the token-level interactions in \cref{eq:score} as
\begin{equation}
    S_c(u;\mathcal D_{\mathrm{ds}})
    =
    \mathbb E_{o\sim\mathcal D_{\mathrm{ds}}}
    \left[
    \sum_{i,j} \texttt{CH}_c(o_i,u_j)
    \right],
    \quad
    c\in\{1,2\},
    \quad
    u\sim\mathcal{D}_\mathrm{pool}
    \label{eq:selection_score}
\end{equation}
and $S_{\texttt{CH1+2}}=S_{\texttt{CH1}}+S_{\texttt{CH2}}$.
Examples with higher score should be selected.
Unlike gradient-based attribution methods, these scores require only the forward quantities.\footnote{Forward-only at a fixed checkpoint; reliable signed use may require the brief warm-up in Appendix~\ref{app:diagnose}}

\paragraph{When is the direct channel sufficient?}
We first test whether positive interaction identifies related examples in controlled attribution benchmarks following \citet{deng2026value}.
Examples from the same task class are treated as mutually relevant, and candidate examples are ranked for each target using AUC and top-$K$ recall.
These settings contain substantial task and output-token overlap, so our analysis predicts that the direct force alignment in \texttt{CH1} should already be highly informative.
Table~\ref{tab:selection_main} confirms this: \texttt{CH1} nearly saturates both sentence-transformation and mathematical attribution, while adding \texttt{CH2} provides little additional benefit.

We next weaken this direct alignment while preserving semantic correspondence.
For each English GSM8K or MMLU example, we construct semantically equivalent candidates in Chinese, French, Korean, and Spanish and ask the score to recover the corresponding examples across languages.
Direct prediction-vocabulary overlap is substantially reduced in this setting.
Here the additional readout-mediated coupling becomes useful: on Qwen2.5-1.5B, adding \texttt{CH2} improves retrieval accuracy from $0.488$ to $0.625$ on MMLU and from $0.445$ to $0.600$ on GSM8K.
Together, the controlled and cross-lingual settings support the geometric distinction from \cref{sec:learning_dynamics}: direct force alignment can be sufficient when output-space overlap is strong, while the more diffuse readout-mediated channel contributes when that overlap weakens.

\paragraph{From predicted influence to useful experience.}
The same scores remain useful when moved beyond attribution labels.
In a multilingual agent-memory setting, an English query retrieves one experience from a shared multilingual memory before answering.
On Llama3.2-3B, incorporating \texttt{CH2} raises top-1 retrieval accuracy from $0.861$ to $1.000$, yielding $0.944$ downstream answer accuracy.

More importantly, predicted positive influence also translates into better learning when the retrieved experience is used for parameter updates.
We rank GSM8K training examples before fine-tuning and train on only the selected subset.
At a $5\%$ budget, \texttt{CH1+CH2} reaches $0.636$ accuracy on Qwen2.5-1.5B and $0.313$ on Llama3.2-3B, compared with $0.610/0.295$ for random selection and $0.587/0.301$ for \texttt{LESS} \citep{xia2024less}.
Thus, the local interaction is useful not only for identifying related examples, but also for deciding which experience should actually be learned.
Full benchmark results, uncertainty estimates, and experimental details are provided in Appendix~\ref{app:attribution}.

\begin{table*}[h]
  \centering
  \vspace{-5pt}
    \caption{
    \textbf{Forward-computable interactions identify useful experience.}
    On controlled tasks with strong direct alignment, \texttt{CH1} nearly saturates attribution. When alignment is weakened cross-lingually, the readout-mediated \texttt{CH2} provides substantial complementary signal, with \texttt{CH1+2} markedly outperforming gradient-based baselines. Both scores use only forward-pass quantities at a fixed checkpoint.
    }
  \label{tab:selection_main}

  \resizebox{0.85\textwidth}{!}{
\begin{tabular}{c cc cc cc cc}
\toprule
\textbf{Method}
& \multicolumn{2}{c}{\textbf{Sentence Transform.}}
& \multicolumn{2}{c}{\textbf{Controlled Math}}
& \multicolumn{2}{c}{\textbf{Cross-lingual Retrieval}}
& \multicolumn{2}{c}{\textbf{Efficiency}} \\
\cmidrule(lr){2-3}
\cmidrule(lr){4-5}
\cmidrule(lr){6-7}
\cmidrule(lr){8-9}
(Qwen2.5-1.5B-Inst.)

& \textbf{AUC $\uparrow$}
& \textbf{Recall $\uparrow$}
& \textbf{AUC $\uparrow$}
& \textbf{Recall $\uparrow$}
& \textbf{MMLU $\uparrow$}
& \textbf{GSM8K $\uparrow$}
& \textbf{Bwd.}
& \textbf{Complexity} \\
\midrule

\texttt{Random}
& 0.500 & 0.100
& 0.500 & 0.100
& 0.025 & 0.020
& $\times$ & $\mathcal{O}(n)$ \\

\texttt{Embd}
& 0.546 & 0.148
& 0.555 & 0.146
& 0.125 & 0.085
& $\times$ & $\mathcal{O}(nd)$ \\

\texttt{DataInf}
& 0.981 & 0.826
& 0.985 & 0.878
& 0.419 & 0.070
& \checkmark & $\mathcal{O}(nd_{\mathrm{in}}L)$ \\

\texttt{HyperINF}
& 0.993 & 0.934
& 0.986 & 0.942
& 0.694 & 0.370
& \checkmark & $\mathcal{O}(nd^3L)$ \\

\texttt{LESS}
& 0.785 & 0.370
& 0.835 & 0.592
& 0.469 & 0.140
& \checkmark & $\mathcal{O}(np_{\mathrm{proj}})$ \\

\midrule

\texttt{CH1}
& 1.000 & 0.989
& 1.000 & 0.998
& 0.488 & 0.445
& $\times$ & $\mathcal{O}(nd)$ \\

\texttt{CH1+2}
& 0.998 & 0.963
& 1.000 & 0.999
& 0.625 & 0.600
& $\times$ & $\mathcal{O}(ndL)$ \\

\bottomrule
\end{tabular}
  }
\end{table*}

\section{Negative Interactions: Collision and Erosion}
\label{sec:forgetting}

\begin{figure}[t]
    \centering
    \includegraphics[width=0.95\linewidth]{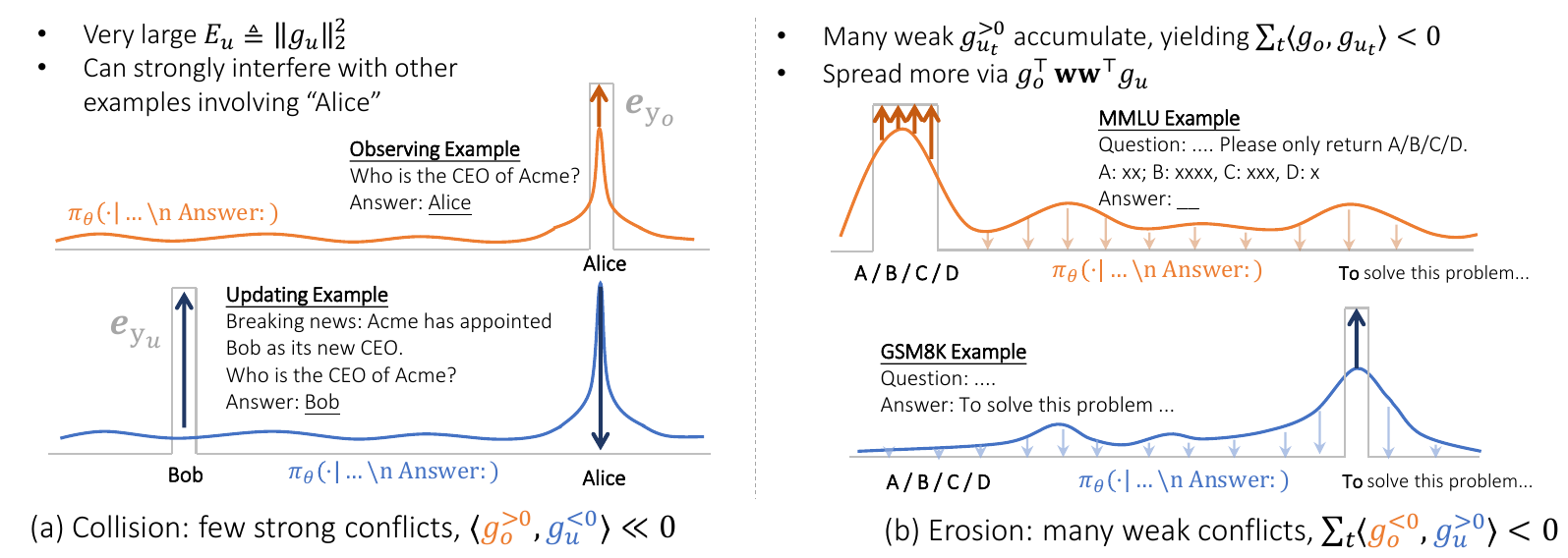}
    \vspace{-5pt}
    \caption{
    \textbf{Collision and erosion as two mechanisms of forgetting.}
    (a) Collision arises from a few strong negative force interactions, often with high-energy updates.
    (b) Erosion arises from many weak conflicts that accumulate and push behavior toward alternative continuations.
    }
    \label{fig:collision_erosion}
\end{figure}

The same interaction can also become negative: when $\Delta_t(o,u)<0$, learning from $u$
decreases confidence in an existing behavior $o$, corresponding to the regime conventionally
studied as catastrophic forgetting \citep{kirkpatrick2017overcoming,li2024revisiting}.
Equation~\ref{eq:score} reveals two qualitatively different regimes.
\emph{Collision} is concentrated: one or a few updates exert strong negative forces on existing
token directions.
\emph{Erosion} is distributed: many weak interactions accumulate coherently and gradually shift
behavior (\cref{fig:collision_erosion}), suggesting different controls.

\paragraph{Collision: Concentrated Negative Forces.} Abrupt forms of catastrophic forgetting can be viewed as a rapid loss of confidence in existing behaviors after new learning.
In our interaction view, this corresponds locally to one or a few updates producing a large negative $\Delta_t(o,u)$, causing the model's confidence in an existing prediction to drop sharply.
Through \cref{eq:score}, we hypothesize that the force alignment $g_o^\top g_u$ plays a central role: hidden-state similarities are typically positive\footnote{This phenomenon is often referred to as the ``narrow cone'' effect \citep{ethayarajh2019contextual,gao2018representation}. We verify the same pattern across settings in Appendix~\ref{app:diagnose:hh}.}, while the strong diagonal structure of $\vw\vw^\top$ causes the largest \texttt{CH2} interactions to often align with the same direct force conflicts, as illustrated in \cref{fig:collision_erosion}(a).

We refer to this regime of concentrated negative force interaction as \emph{collision}.
Because collision is driven by a strong mismatch between the model's current prediction and the incoming supervision, it can be detected directly from the update force.
\begin{restatable}{proposition}{propnegative}
    \label{prop:g_negative}
    \textbf{Extreme energy implies a large negative force, and vice versa.}
    Define
    \[
    E_u \triangleq \|g_u\|_2^2 = \|\ve_{y_u}-\pi_\theta(\cdot\mid s_u)\|_2^2 = 1-2\pi_\theta(y_u\mid s_u)+\|\pi_\theta(\cdot\mid s_u)\|_2^2
    \]
    as token energy. Under one-hot supervision, for any $\delta>0$, if $E_u>1+\delta$, there exists a non-target token $j\neq y_u$ such that $[g_u]_j<-\delta$.
    Conversely, if $[g_u]_j<-\delta$ for some $j\neq y_u$, then $E_u>2\delta^2$.
\end{restatable}

Thus, extreme update energy is not merely a norm appearing in the Cauchy--Schwarz bound
$|g_o^\top g_u|\leq\|g_o\|_2\|g_u\|_2$.
It certifies the presence of a strong negative vocabulary-space direction that can collide with existing high-confidence behavior.
Moreover, the expansion of $E_u$ clarifies when such risky updates arise.
The term $1-2\pi_\theta(y_u\mid s_u)$ increases when the supervised target is unlikely, while
$\|\pi_\theta(\cdot\mid s_u)\|_2^2$ increases as the predictive distribution becomes more concentrated.
Indeed, $-\log \|\pi_\theta(\cdot\mid s_u)\|_2^2$ is the order-2 R\'enyi, or collision, entropy \citep{renyi1961measures}.
Proof in Appendix~\ref{app:proofs:negative}.

\begin{figure}[t]
    \centering
    \includegraphics[width=1\linewidth]{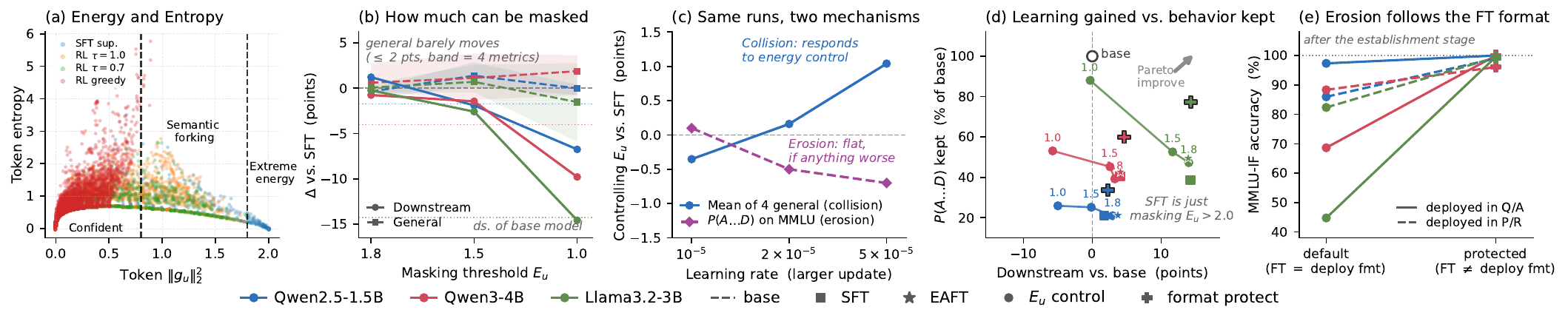}
    \vspace{-15pt}
        \caption{\textbf{Collision and erosion respond to different controls.}
        (a) Extreme-energy updates mark confident conflicts.
        (b--c) Energy masking mitigates collision but has little effect on erosion.
        (d) For erosion, energy control improves retention mainly by sacrificing downstream learning; format separation shifts this trade-off.
        (e) Erosion follows the fine-tuning format, so separating fine-tuning from deployment redirects interference away from the deployed behavior.
        More in Appendix~\ref{app:forgetting}.}
    \label{fig:results_collision_erosion}
    \vspace{-15pt}
\end{figure}

\paragraph{Erosion: Weak Interactions Accumulate.}
Figure~\ref{fig:collision_erosion} contrasts a second regime, \emph{erosion}, in which many
individually weak negative interactions accumulate coherently over training.
Unlike collision, erosion need not involve any extreme-energy update: each individual
$g_o^\top g_u$ may be small, while repeated updates consistently push probability mass toward
competing continuations.
The semantically diffuse \texttt{CH2} further spreads this interaction, so a per-token magnitude
criterion such as $E_u$ need not identify the responsible updates.

The GSM8K example in \cref{fig:collision_erosion}(b) illustrates the mechanism.
For an MMLU prompt that should be answered with \texttt{A/B/C/D}, the observing force assigns small negative components to alternative continuations such as \texttt{To solve this ...}.
GSM8K supervision repeatedly reinforces such reasoning-style continuations, giving them positive update forces.
Each interaction $\langle g_o^{<0},g_u^{>0}\rangle$ can therefore be weak, yet their repeated alignment gradually shifts probability mass away from the desired format.
This distinction suggests different controls: collision should respond to the magnitude of individual updates, whereas erosion depends on how many weak interactions are repeatedly routed toward the same competing behavior.
We next test these contrasting predictions experimentally.

\textit{Fig.~\ref{fig:results_collision_erosion}(a): Why does off-policy SFT forget more?}
On-policy RL forgets less than off-policy SFT.
Our framework attributes this to extreme-energy updates: common under off-policy SFT but rare in on-policy rollouts ($\tau$: sampling temperature; greedy: deterministic).
This motivates energy control; Entropy-Adaptive Fine-Tuning (EAFT) \citep{diao2026entropy} is one example, with other on-policy methods discussed in Appendix~\ref{app:forgetting:gkd_opd}.
We next test whether such control prevents collision.

\textit{Fig.~\ref{fig:results_collision_erosion}(b): How much high-energy masking is safe?}
Energy control is effective only in the extreme-energy regime, since learning itself also requires energy.
Masking updates with $E_u>1.8$ or $1.5$ removes this tail at little downstream cost, changing GSM8K by at most $2.6$ points across models.
In contrast, a threshold of $1.0$ sharply degrades downstream learning, sometimes nearly returning the model to pre-fine-tuning performance.
As Fig.~\ref{fig:results_collision_erosion}(a) shows, $E_u=1.0$ already masks many ordinary semantic tokens (including many semantic-forking tokens), rather than only extreme conflicts.

\textit{Fig.~\ref{fig:results_collision_erosion}(c): Energy control separates collision from erosion.}
Standard MMLU accuracy can mask erosion, so we track
$P(A\ldots D)\triangleq\sum_{c\in\{A,B,C,D\}}\pi_\theta(c\mid s)$,
the total answer-token probability mass.
As the learning rate increases, $E_u$ control increasingly improves collision-side retention
(the mean of four general-capability metrics; Appendix~\ref{app:forgetting:gkd_opd}),
while $P(A\ldots D)$ remains nearly unchanged.
Thus, controlling update magnitude addresses collision but leaves erosion largely intact.

Our decomposition suggests a different control for erosion.
Because erosion accumulates through repeated force alignment and hidden-state similarity in Equation~\eqref{eq:score}, changing the response format can perturb both pathways that repeatedly route interference toward the same behavior.

\textit{Fig.~\ref{fig:results_collision_erosion}(d): The two controls exhibit different gain--retention trade-offs.}
Energy masking recovers answer-position behavior mainly by moving back toward the base model and sacrificing downstream learning.
Response-format protection instead preserves substantially more instruction-following behavior while retaining downstream gains, consistent with our analysis.

\textit{Fig.~\ref{fig:results_collision_erosion}(e): Erosion follows the fine-tuning format.}
To understand how format separation achieves this retention gain, we start from the \texttt{base} model to avoid inherited response-format preferences.
We first establish MMLU behavior in one format, then fine-tune GSM8K using either the same or an alternative format.
With matched formats, instruction following (IF) degrades substantially; when the formats are separated, deployed behavior is largely preserved while erosion shifts toward the fine-tuning format.
Swapping the formats reverses the pattern.
GSM8K is learned in all settings, showing that format separation redirects interference rather than suppressing downstream learning.

\textit{In summary,} collision stems from a few high-energy conflicts and responds to energy control, whereas erosion accumulates through many weak interactions and responds to changing alignment.

\section{Evolving Geometry and Future Learnability}
\label{sec:plasticity}

\begin{figure}[t]
    \centering
    \includegraphics[width=1\linewidth]{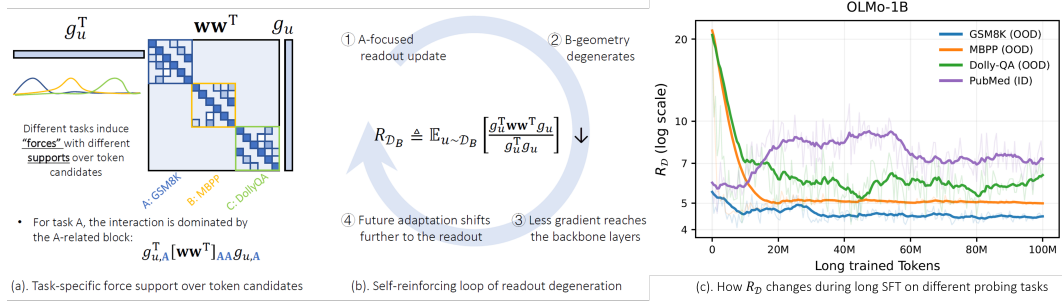}
    \caption{\textbf{Readout transmission governs future learnability.}
            (a--b) The shared readout gates task forces into the hidden space and can form a self-reinforcing bottleneck.
            (c) The corresponding score $R_{\mathcal D}$ declines during long-horizon training for OOD tasks.}
    \label{fig:plastic_theory}
    \vspace{-5pt}
\end{figure}

So far, we have asked how a new update interacts with existing behavior under the model's current geometry.
A lifelong learner faces a further question: after many such updates, will future learning remain equally effective?
This failure mode is commonly studied as \emph{plasticity loss} \citep{lyle2023understanding-82a}.
Our interaction view suggests that the problem is not separate from the analysis above.
The same geometry that determines how an update changes current behavior is itself modified by learning, and therefore determines how strongly future updates can act.

To see this connection, consider the self-interaction $o=u$.
Equation~\ref{eq:score} then measures how effectively a token's own learning signal changes its prediction.
In particular, the backbone-mediated channel contains
$g_u^\top \vw\vw^\top g_u=\|\vw^\top g_u\|_2^2$.
Importantly, $\vw^\top g_u$ is the output-side learning signal.
Every backbone layer receives gradients through it.
The shared readout therefore forms a bottleneck between token-space forces and backbone learning.
If training reshapes $\vw\vw^\top$ to attenuate future-task directions, the backbone receives weaker signals even when $g_u$ remains strong.

\paragraph{Readout transmission as signal gain.}
This suggests viewing plasticity as a signal-transmission problem:
$g_u$ is the incoming signal and $\vw^\top g_u$ the signal reaching the backbone.
The relevant question is not the input magnitude, but how much of it is transmitted.
To isolate this gain from the raw force magnitude, we define the task-conditioned readout transmission
\begin{equation}
    R_\mathcal{D}\triangleq \mathbb{E}_{u\sim\mathcal{D}}\left[
    \frac{\|\vw^\top g_u\|_2^2}{\|g_u\|_2^2}
    \right]
    =\mathbb{E}_{u\sim\mathcal{D}}\left[
    \frac{g_u^\top\vw\vw^\top g_u}{g_u^\top g_u}.
    \right]
    \label{eq:R}
\end{equation}
This Rayleigh quotient of $\vw\vw^\top$ along task-induced force directions measures directional signal gain:
the numerator is the transmitted strength, while the denominator normalizes the input \citep{oppenheim1996signals}.
Thus, $R_{\mathcal D}$ measures how effectively signals from task $\mathcal D$ enter the backbone.

\paragraph{Readout transmission evolves with learning.}
Crucially, $R_{\mathcal D}$ is itself dynamic because the shared readout is updated during training.
Different tasks induce forces along different vocabulary directions and therefore probe different regions of $\vw\vw^\top$ (Fig.~\ref{fig:plastic_theory}a).
Training on the current task preferentially reinforces its supported directions, while transmission along weakly supported future-task directions can deteriorate.
The same output-level force can then induce a weaker update to the backbone.

This degradation can further amplify itself (Fig.~\ref{fig:plastic_theory}b).
As $R_{\mathcal D}$ decreases for a future task, less of its learning signal reaches the backbone, so subsequent adaptation can rely more heavily on the readout itself.
These updates may further concentrate the readout toward currently supported directions, forming a self-reinforcing bottleneck.
We view this loop as a possible amplification mechanism rather than a necessary condition for plasticity loss, and formally prove it in Appendix~\ref{app:proofs:plasticity}.

\paragraph{Experimental verification.}
We test these predictions along a long-horizon PubMed training trajectory.
At each checkpoint, we measure $R_{\mathcal D}$ on three held-out OOD tasks
(GSM8K, MBPP, and Dolly-QA) and in-distribution PubMed.
We predict $R_{\mathcal D}$ to decay on OOD tasks but remain comparatively stable in-distribution;
if this degeneration causes plasticity loss, lower $R_{\mathcal D}$ should also imply slower adaptation.
Broader results across models and training settings are in Appendix~\ref{app:exp_plasticity}.

\textit{Fig.~\ref{fig:plastic_theory}(c): Future-task transmission declines during long-horizon training.}
Over approximately 100M PubMed tokens (results on 300M hybrid tokens in Appendix~\ref{app:exp_plasticity}), \(R_{\mathcal D}\) progressively decreases for held-out OOD tasks, while the in-distribution task shows no comparable systematic decay. This suggests that the effect is not uniform readout shrinkage, but task-dependent reshaping toward currently supported directions. Sequential SFT results in Appendix~\ref{app:exp_plasticity} further support this interpretation.

\begin{figure}[t]
    \centering
    \includegraphics[width=1\linewidth]{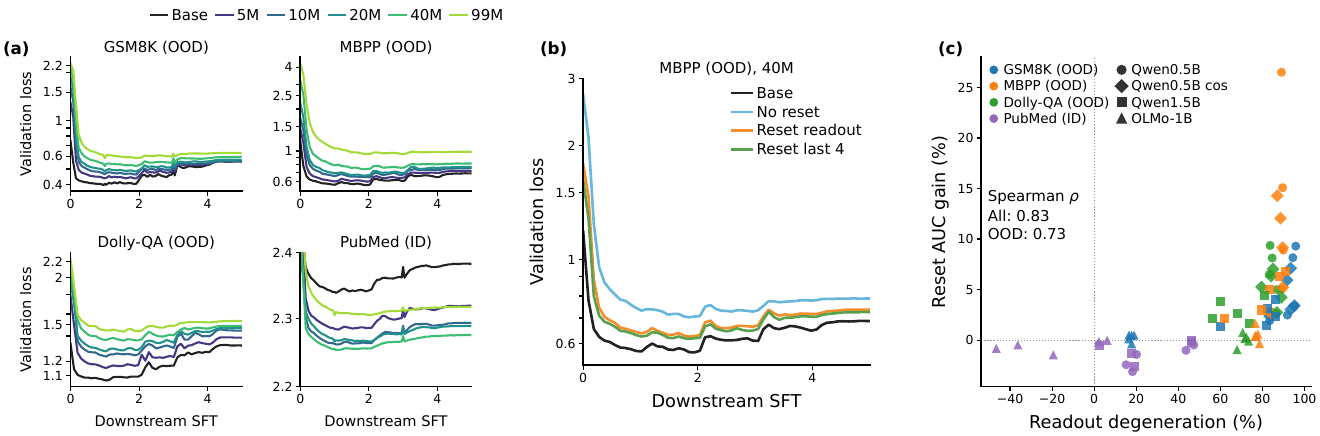}
            \caption{\textbf{Readout degeneration predicts future learnability.}
            (a) Later checkpoints adapt more slowly on OOD tasks but not ID PubMed.
            (b) Readout restoration partially recovers downstream learning.
            (c) More degeneration predicts larger reset gains across models, tasks, and checkpoints.}
    \label{fig:plasticity_results}
    \vspace{-5pt}
\end{figure}

\textit{Fig.~\ref{fig:plasticity_results}(a): Transmission loss predicts plasticity loss.}
We next ask whether declining $R_{\mathcal D}$ corresponds to reduced future learnability.
As predicted, later checkpoints adapt progressively more slowly on three OOD tasks, while in-distribution PubMed shows no comparable degradation.
Thus, the task-dependent decline in $R_{\mathcal D}$ is mirrored by task-dependent loss of future learnability.

\textit{Fig.~\ref{fig:plasticity_results}(b): Restoring the readout recovers plasticity.}
Because the decline in $R_{\mathcal D}$ localizes degeneration to the shared readout, we restore it to its base value before downstream adaptation, echoing reset-based plasticity interventions in classical RL \citep{lyle2023understanding-82a}.
Readout reset substantially improves validation loss, while resetting the last four layers provides only modest additional recovery.
This turns the diagnostic into a functional test of the proposed bottleneck.
Thus, the readout is an important, though not exclusive, source of lost plasticity.
We use readout restoration as a mechanistic intervention rather than a deployment strategy; preserving capabilities acquired during long-horizon training while recovering plasticity is outside the scope of this diagnostic.

\textit{Fig.~\ref{fig:plasticity_results}(c): Readout degeneration predicts reset benefit.}
Finally, we aggregate all settings above and directly compare the decrease in $R_{\mathcal D}$ with the gain from readout reset.
Across models, tasks, and checkpoints, the two quantities are strongly correlated
(Spearman $\rho=0.83$ overall; $0.73$ on OOD tasks).
Thus, the diagnostic and intervention tell a consistent story: settings with greater readout degeneration are precisely those that benefit more from restoring the readout.

\section{Conclusion and Discussion}

We studied continual learning through a single evolving update--behavior interaction.
The same local geometry determines which experiences produce useful updates, when updates interfere with existing behavior, and how accumulated learning reshapes the model's ability to learn next.
Opening this interaction into token forces, shared-readout geometry, and residual flow exposes distinct mechanisms behind these behaviors: direct and diffuse influence for selection, collision and erosion for forgetting, and readout transmission for future learnability.
The decomposition yields both diagnostics and mechanism-derived interventions, from energy control and format separation to readout restoration.

Our analysis is intentionally local and approximate.
It prioritizes a forward-computable view of interaction geometry over exact reconstruction of training trajectories, and the shared readout is important but not the sole source of long-term plasticity loss.
Higher-order residual paths, optimizer-dependent dynamics, longer-term feedback, and broader architectures remain important directions.
We discuss these limitations, related work, and additional evidence in Appendix~\ref{app:sec:related_works}.
More broadly, continual learning requires understanding not only what models learn, but how each act of learning reshapes the learner and its capacity to learn next.

\section*{Acknowledgements and Author contributions}
Acknowledgements: We thank Yihong Chen for feedback on the manuscript, Ruotian Peng for helpful discussions on the experiments, and Hamed Shirzad for comments on the manuscript.
We also thank Hen Davidov, Anushka Nair, Zihuiwen Ye, Lin Li, Hao Fei, Yonatan Gideoni, Xander Davies, Lucciano Carvalho Melo, and Sergio Calvo Ordoñez for helpful discussions during group meetings. 
We are grateful to Christos Thrampoulidis for facilitating access to computational resources used in part of the experiments.
YR acknowledges funding support from GenAI. We also acknowledge computational support from the OATML group and Isambard-AI.

Author contributions: \textbf{YR} conceived and led the project, developed the theoretical framework and analyses, designed and conducted the majority of the experiments, and led the writing and revision of the manuscript. \textbf{WD} contributed to the development and discussion of the ideas, and was a core contributor to the data attribution and forgetting sections, including experiments.
\textbf{GH} contributed to discussions throughout the project and to the experimental studies in the forgetting section.
\textbf{CL} contributed to discussions throughout the project and especially on plasticity loss.
\textbf{YG} supervised the project, contributed to conceptual discussions and research direction, and provided extensive feedback on the manuscript and its revisions.

\bibliography{iclr2027_conference}
\bibliographystyle{iclr2027_conference}

\clearpage

\appendix
\section{Related Work and Discussion}
\label{app:sec:related_works}

\subsection{Continual Adaptation and Learning Dynamics in Agentic LLM Systems}
\label{sec:related_works:CL}
Continual learning is a classical problem in machine learning \citep{mccloskey1989catastrophic,kirkpatrick2017overcoming}, but emerging agentic and self-improving LLM systems give it renewed relevance. 
Recent proposals span recursively self-improving systems \citep[RSI,][]{yin2025godel, zhang2026darwin}, test-time training \citep[TTT,][]{sun2024learning}, and iterated learning \citep{gulcehre2023reinforced, ren2024bias}, and may combine several adaptation mechanisms, including in-context learning \citep{olsson2022context}, external memory \citep{lewis2020retrieval}, lightweight adapters \citep{pfeiffer2022lifting}, parameter resetting \citep{chen2023improving}, and direct parameter updates \citep{ramkumar2023learn}. 
Such systems need not rely on a single learning loop or timescale. 
Nested Learning \citep{behrouz2026nested} provides a useful abstraction for this view: information may be incorporated through a hierarchy of learning processes, with in-context adaptation operating rapidly, memory and lightweight adapters at intermediate timescales, and parameter updates more slowly.

Our framework focuses on one stage of this hierarchy: the consolidation of experience into model parameters. 
Parameter updating may be only one component of continual adaptation, but whenever transient experience is eventually consolidated into model weights, the cycle studied in \cref{fig:system} reappears: the learner must decide what experience to consolidate, understand how the resulting update affects existing behavior, and preserve its ability to incorporate future experience. 
Our analysis provides a concrete account of this cycle for gradient-based parameter updates.

The same adaptation logic has conceptual counterparts in systems that rely primarily on in-context learning or external memory. 
A memory-based learner must still determine which past experience to retrieve, whether newly introduced information conflicts with existing knowledge or instructions, and whether its growing memory remains useful for future adaptation. 
The underlying mechanisms differ substantially from gradient-based learning, but the broader cycle of selection, interference, and sustained adaptability remains. 

We therefore view selection, interference, and sustainability as a broader way of organizing continual adaptation. 
The learning-dynamics framework developed in this paper provides one concrete realization for parameter updates, while analogous questions may arise under other adaptation mechanisms.

Although our concrete analysis focuses primarily on SFT-style updates, the learning-dynamics perspective has already been applied more broadly across LLM post-training. 
Recent work has used related update-level analyses to study refusal-aware instruction tuning, mixed reasoning data, and unlearning-induced probability redistribution \citep{zhu2025grait, deng2025inverse, li2026llm}, as well as negative-gradient effects, exploration and exploitation, and support preservation in RLVR \citep{deng2025on, ren_thesis, peng2025beyond, deng2026token}. 
This perspective has further been extended to tool-integrated reinforcement learning and agent search \citep{deng2026on}, as well as to safety alignment \citep{huang2026alignment}. 
These developments suggest that fine-grained learning dynamics may provide a useful analytical language beyond supervised adaptation, while a unified treatment of SFT, RL, tool use, and memory-based adaptation remains an open direction.

\subsection{Related Work on Selection, Interference, and Plasticity}
\label{sec:related_works:three_tasks}

Data attribution, forgetting, and plasticity loss have largely developed as separate research areas, but each concerns a different aspect of how model updates interact with behavior.
Data attribution asks which updates are beneficial, forgetting studies harmful interactions with existing behavior, and plasticity concerns how previous updates alter the model's responsiveness to future learning. 
We briefly relate our framework to these three literatures from this shared perspective.

\textbf{Beneficial interactions and data attribution.}
Data attribution studies how individual training examples affect model predictions or downstream behavior, with influence functions providing a classical formulation \citep{koh2017understanding}.
Recent work has adapted this idea to modern LLMs through scalable influence estimation \citep{grosse2023studying}, efficient approximations for parameter-efficient fine-tuning \citep{kwon2024datainf}, gradient-based data selection \citep{xia2024less,zhao2026rethinking}, and forward-only attribution \citep{deng2026value}.
Our perspective is complementary: rather than treating attribution only as a ranking problem, we resolve the underlying update–observation interaction into token-level forces, readout geometry, and layer-wise propagation. 
The resulting approximation yields practical attribution scores, but more importantly exposes the same interaction structure that later governs interference and future learnability.

\textbf{Harmful interactions and forgetting.} 
Forgetting has long been recognized as a central challenge in continual learning \citep{kirkpatrick2017overcoming}, and has recently been revisited in LLMs through continual instruction tuning, rehearsal, and analyses of fine-tuning-induced degradation \citep{scialom2022fine, li2024revisiting}.
Earlier theoretical accounts often relate forgetting to representation or task-subspace overlap, for example through frozen-feature similarity or NTK overlap \citep{ramasesh2020anatomy, doan2021theoretical}.
Our analysis instead focuses on the signed local interactions that produce these behavioral changes.
By separating output forces, readout geometry, and layer-wise propagation, the framework distinguishes two negative-interaction regimes: \textit{collision}, where concentrated negative forces induce sharp interference, and \textit{erosion}, where many individually mild interactions accumulate into gradual behavioral drift. 
This connects abrupt capability loss and subtler forms of forgetting within the same local learning dynamics.

\textbf{Evolving interactions and plasticity loss.} 
Plasticity loss concerns whether a trained model remains responsive to new learning signals.
It has been studied extensively in deep and reinforcement learning, where repeated adaptation can progressively reduce responsiveness to new experience, and resetting upper layers can partially restore plasticity \citep{lyle2023understanding-82a, nikishin2022primacy}.
Related phenomena are now emerging in LLMs: extended pretraining can make models harder to fine-tune \citep{springer2025overtrained}, while post-training strategies exhibit distinct stability–plasticity trade-offs across continual natural-language training \citep{hernandez2026can}.
From our interaction view, plasticity introduces a qualitatively different question from the two regimes above: previous updates change the geometry through which future updates act. 
Our analysis identifies the higher-layer and shared readout geometry as part of this evolving interaction operator, whose reshaping modulates the effective learning signal reaching the backbone.
This provides a structural link between degraded future learnability and the empirical effectiveness of resetting the readout or other upper layers.

Across these literatures, our distinction is a finer unified mechanistic analysis.
Existing methods often summarize updates through example-level influence, aggregate forgetting, or global plasticity measures.
Our decomposition instead retains signed token-level forces, shared readout geometry, and layer-wise propagation. 
This finer resolution connects beneficial influence, harmful interference, and evolving future learnability through the same local object.

\subsection{Limitations and Assumptions of the Learning-Dynamics Framework}
\label{sec:related_works:limitations}
Theory needs assumptions and approximations. 
The relevant question is whether they preserve the structure needed to explain the system of interest. 
Throughout this work, we trade some fidelity to the full training dynamics for tractability, interpretability, and computational efficiency.
These approximations arise from a common goal rather than being introduced separately to fit individual observations: retaining the dominant structure of model updates while keeping the framework analytically transparent and practical at LLM scale. 
Most of our results build on \cref{eq:score}, which relates behavioral change to the local gradient interaction. 
We summarize below where the main approximations enter and how they constrain the interpretation of our results.

\textbf{First-order local approximation.}
The starting point of our analysis is \cref{eq:def_confidence_change}, which follows from a first-order Taylor expansion and neglects terms of order \(\mathcal{O}(\|\theta_{t+1}-\theta_t\|_2^2)\). 
It is therefore most reliable for sufficiently small parameter updates.
This assumption is related to, but distinct from, the local approximations used by influence functions: influence functions characterize how a nearby empirical-risk optimum changes under a small data perturbation, whereas our analysis asks how one actual local update changes an observed behavior. 
This leads directly to the gradient interaction \(\nabla_\theta \log \pi_o^\top \nabla_\theta \log \pi_u\), without modeling re-optimization around a nearby optimum.

\textbf{Optimizer dynamics.}
For clarity, our derivation assumes an SGD update, while LLM training typically uses adaptive optimizers such as AdamW \citep{loshchilov2017decoupled}.
For a fixed optimizer state, the framework extends naturally to a preconditioned update.
If $\theta_{t+1}-\theta_t=\eta P_t\nabla_\theta\log\pi_\theta(y_u\mid s_u)$, then $\Delta_t(o,u)\approx\eta\nabla_\theta\log\pi_\theta(y_o\mid s_o)^\top P_t\nabla_\theta\log\pi_\theta(y_u\mid s_u)$.
The main complication is therefore not preconditioning itself, but its dependence on the training history.
In AdamW, momentum and second-moment estimates make the effective preconditioner $P_t$ evolve with previous gradients, while decoupled weight decay introduces an additional update component not induced by the current token.
Thus, incorporating a fixed optimizer state into the one-step analysis is straightforward, whereas jointly modeling the evolution of optimizer and model states would require extending the framework beyond local update dynamics.
We briefly discuss our framework in the AdamW case in Appendix~\ref{app:proofs:adamw} and leave the more principled analysis to future work.

\textbf{Residual-path approximation.}
The forward-computable form in \cref{eq:score} relies on two related but distinct approximations.
First, we locally linearize each residual block over the scale of the update, replacing its parameter-gradient geometry with that of the corresponding local linear map.
Second, when propagating the learning signal across layers, we retain only the leading identity path through the residual stream and omit paths that traverse one or more additional residual branches.
These two approximations control different sources of error: the former determines how accurately each block is represented locally, while the latter determines how much cross-layer Jacobian composition is ignored.
Importantly, the omitted higher-order residual paths remain part of the same first-order gradient rather than higher-order derivatives.
Retaining them would systematically improve fidelity, but would also reintroduce increasingly expensive Jacobian products.
See Appendix~\ref{app:diagnose} for more.

\textbf{From one-step interactions to accumulated learning.}
Perhaps the most important limitation is that the theory is fundamentally local in time. 
Our interaction analysis is local to a single update, whereas forgetting and plasticity loss emerge through many updates.
These mechanistic arguments therefore extrapolate the structure revealed by local interactions along a training trajectory, with the resulting predictions validated empirically rather than derived from an exact multi-step theory. 
A complete trajectory-level treatment would need to track the joint evolution of gradients, representations, and interaction geometry throughout optimization. 
Recent continuous-time formulations \citep{litman2026theory} provide one possible route by integrating time-varying kernel interactions; extending our structured decomposition in this direction would allow the interaction channels, force geometry, and layer-wise propagation to evolve explicitly over time.

\textbf{Empirical scope.}
Our experiments cover only a limited portion of the space of possible continual-learning systems.
Token forces, representations, and readout geometry can vary with model architecture, data distribution, training objective, and optimizer; our empirical results should therefore be viewed as mechanistic patterns rather than universal quantitative laws.
The current decomposition is also tailored to Transformer architectures with residual connections and a shared readout, and does not explicitly resolve finer-grained structures such as MoE, QKV projections, or circuit-level pathways.
Its signed fidelity can also depend on the representation regime, particularly at cold start; we characterize this behavior and the effect of brief warm-up in \cref{app:diagnose}.
We therefore view the framework as a set of mechanistic tools and testable hypotheses for continual adaptation, rather than a closed theory specific to the models and settings studied here.

\section{Proofs and Derivations}
\label{app:proofs}

\subsection{Derivation of One-step Token-wise Confidence Change}
\label{app:proofs:delta}
We now formally derive \cref{prop:score} in \cref{sec:learning_dynamics}.
We use the convention that Jacobians of vector outputs are output-by-parameter,
so $J=\partial h/\partial\phi\in\mathbb R^{d\times p}$.
\propscore*

\begin{proof}
We start from calculating \cref{eq:a_kkk_g}
\[
    \Delta_t(o,u)
    \approx
    \eta
        \Big[
        \underbrace{
        (g_o^\top g_u)
        (\vh_o^\top \vh_u)
        }_{\text{readout contribution}} + 
        \underbrace{
        (\vw^\top g_o)^\top
        J_o J_u^\top
        (\vw^\top g_u)
        }_{\text{backbone contribution}}
        \Big],
\]
which considers the following model:
\[
s\xrightarrow{f(s;\phi)} \vh \xrightarrow{\vw} \vz \xrightarrow{\sigma(\cdot)} \pi
\]
Recall that the logits are written as $\vz = \vw \vh,$ and $\theta=(\vw,\phi)$.
Since the readout and backbone form two disjoint parameter groups, their contributions to the logit eNTK can be separated directly:
\[
    \nabla_{\theta}\vz_o \nabla_{\theta}\vz_u^\top
    = \nabla_{\vw}\mathbf \vz_o \nabla_{\vw}\vz_u^\top + \nabla_{\phi}\vz_o \nabla_{\phi}\vz_u^\top.
\]
For the readout term, each logit only depends on its corresponding row of $\vw$.
Therefore, $\nabla_{\vw}\mathbf \vz_o\nabla_{\vw}\mathbf \vz_u^\top=(\vh_o^\top\vh_u)I_{V\times V}$.
For the backbone term, the chain rule gives $\nabla_{\phi}\vz=\nabla_{\vh}\vz\nabla_\phi \vh=\vw\nabla_{\phi}\vh=\vw J$ and hence $\nabla_{\phi}\vz_o\nabla_{\phi}\vz_u^\top=\vw J_oJ_u^\top\vw^\top$.
Starting from Equation~\eqref{eq:def_confidence_change_2}, and substituting these two terms into \cref{eq:a_kkk_g} and sandwiching the result between $g_o^\top$ and $g_u$ gives \cref{eq:a_kkk_g}.

The decomposition is therefore simply induced by the two parameter groups: updating $\vw$ gives the readout contribution, while updating $\phi$ gives the backbone contribution.
No additional approximation beyond the first-order update is required for this parameter-group decomposition.

\cref{eq:score} further considers the residual architecture used by modern LLMs, as in \cref{fig:residual_network}.
Under this setting, we further partition the backbone parameters according to the residual blocks and the embedding layer $\phi=(\phi_0,\ldots,\phi_{L-1},\phi_{\text{embd}})$.
Accordingly, its eNTK can be written as the sum of the contributions from these parameter blocks:
\begin{equation}
J_oJ_u^\top = \nabla_\phi\vh_{L,o} \nabla_\phi\vh_{L,u}^\top = 
\sum_{\ell=0}^{L-1}
\nabla_{\phi_\ell}\vh_{L,o}
\nabla_{\phi_\ell}\vh_{L,u}^\top +
\nabla_{\phi_{\text{embd}}}\vh_{L,o}
\nabla_{\phi_{\text{embd}}}\vh_{L,u}^\top .
\label{app:eq:JJ}
\end{equation}

Here and below, we formulate the backbone eNTK with respect to the pre-normalization state $\vh_L$ and omit the Jacobian of the final RMSNorm. 
Retaining this Jacobian replaces the transmitted force $\vw^\top g$ by $J_{\mathrm{RMS}}(\vh_L)^\top \vw^\top g$, while preserving the same force--kernel--force structure.
We now consider residual blocks of the form
\[
    \vh_{\ell+1} = \vh_{\ell} + f_\ell(\tilde{\vh}_{\ell};\phi_\ell), \qquad
    \tilde{\vh}_{\ell} = \mathsf{RMSNorm}(\vh_{\ell}).
\]

Since $\vh_\ell$ is independent of $\phi_\ell$, $\nabla_{\phi_\ell}\vh_{\ell+1}=\nabla_{\phi_\ell}f_\ell(\tilde{\vh}_\ell;\phi_\ell)$.
By the chain rule, $\nabla_{\phi_\ell}\vh_L=\nabla_{\vh_{\ell+1}}\vh_L \nabla_{\phi_\ell}\vh_{\ell+1} = \nabla_{\vh_{\ell+1}}\vh_L\,\nabla_{\phi_\ell}f_\ell(\tilde h_\ell;\phi_\ell)$.

The next step is the identity-path approximation.
For block $\ell$, its contribution to the final representation satisfies (after distributively expanding the product and grouping terms by the number of residual transformations, similar as derivations in \citet{chen2026decomposing})
\[
    J_\ell= %
    \left[\prod_{k=\ell+1}^{L-1}(I+A_k)\right]B_\ell
    =
    \underbrace{B_\ell}_{J_\ell^{(1)}}
    +
    \underbrace{\sum_{k=\ell+1}^{L-1}A_k B_\ell}_{J_\ell^{(2)}}
    +
    \underbrace{
    \sum_{\ell+1\leq k_1<k_2\leq L-1}
    A_{k_2}A_{k_1}B_\ell
    }_{J_\ell^{(3)}}
    +\cdots,
\]
where $A_k \triangleq \nabla_{\vh_{k}} f_k(\tilde{\vh}_{k};\phi_k)$ and $B_\ell\triangleq\nabla_{\phi_\ell}\vh_{\ell+1}$.
Under the identity-path approximation, i.e., $\nabla_{\vh_{\ell+1}}\vh_L\approx I$, the only problem is $\nabla_{\phi_\ell}f_\ell(\tilde h_\ell;\phi_\ell)$.

We then approximate this geometry by the linear block $f_\ell(\tilde{\vh}_\ell)\approx M_\ell\tilde{\vh}_\ell$.
For this linear surrogate, 
\[
\nabla_{\phi_\ell}\vh_{L,o}
\nabla_{\phi_\ell}\vh_{L,u}^\top \approx
\nabla_{M_\ell}f_{\ell,o}
\nabla_{M_\ell}f_{\ell,u}^\top
=
\left(
\tilde{\vh}_{\ell,o}^\top
\tilde{\vh}_{\ell,u}
\right)I_{d\times d}.    
\]
Thus, each residual block contributes its forward-pass hidden-state similarity to the backbone eNTK.

The embedding layer has the same structure.
An embedding parameter is shared only when the corresponding input tokens are identical, giving 
\[
\nabla_{\phi_{\text{embd}}}\vh_{L,o}
\nabla_{\phi_{\text{embd}}}\vh_{L,u}^\top
\approx
\kappa_{\text{embd}}(s_o,s_u)I_{d\times d},
\]
where $\kappa_{\mathrm{embd}}(s_o,s_u)=\sum_{i=1}^{|s_o|}\sum_{j=1}^{|s_u|}\mathbf 1[s_{o,i}=s_{u,j}]$.
Then, by substituting the two parts above back into \cref{app:eq:JJ} and \cref{eq:score}, we can get the desired expression in \cref{prop:score}.

Note that the embedding term should be viewed as a context-level approximation rather than a strict consequence of the identity path, since full-context overlap implicitly captures non-direct residual routing such as attention-mediated token mixing. 
In practice, this contribution is typically small and does not materially affect our results (rank correlation with and without this term usually $>0.999$), while the context-overlap form better reflects the full model than target-position overlap alone. 
A more faithful treatment of token routing, particularly for long contexts, is left to future work.

The key consequence is that the identity path removes the layer-dependent output-side Jacobian, while the linear-block approximation reduces the remaining parameter geometry to hidden-state similarity.
This is what makes the all-layer backbone contribution estimable from forward-pass quantities.
\end{proof}

\subsection{Energy and Negative Force}
\label{app:proofs:negative}
We now formally prove \cref{prop:g_negative} in \cref{sec:forgetting}.
\propnegative*

We here have a more detailed expression for this proposition.

\paragraph{Detailed form of Proposition~\ref{prop:g_negative}.}
\textit{
    Extreme update energy implies a large negative force. 
    Let
    \[
        g_u=\ve_{y_u}-\pi_\theta(\cdot\mid s_u)\in\mathbb{R}^{V\times 1};\quad E_u\triangleq\|g_u\|_2^2,
    \]
    where $\pi_\theta(\cdot\mid s_u)$ is the model's prediction and $y_u$ is the supervision.
    Define the largest non-target probability as
    \[
        m_u\triangleq \max_{j\neq y_u} \pi_\theta(j\mid s_u).
    \]
    Then, we have
    \[
        m_u\geq\frac{E_u}{1-\pi_\theta(y_u\cdot\mid s_u)}-(1-\pi_\theta(y_u\mid s_u))\geq E_u - 1.
    \]
    Consequently, whenever $E_u > 1+\delta$, there necessarily exists a
    non-target token $j \neq y_u$ such that
    \[
    \exists j\neq y_u,\quad [g_u]_j = -\pi_\theta(j\mid s_u) < -\delta.
    \]
    Therefore, extremely large update energy is not merely a large gradient norm: it certifies the existence of a large negative force along at least
    one non-target token direction.
    In particular, $E_u > 1.9$ implies that some non-target token receives a negative force with magnitude greater than $0.9$.
}

\begin{proof}
Let
\[
    q_u \triangleq 1-\pi_\theta(y_u\mid s_u)= \sum_{j \neq y_u} \pi_\theta(j\mid s_u).
\]
By the structure of the SFT force,
\[
    E_u = \|\ve_{y_u}-\pi_\theta(\cdot\mid s_u)\|_2^2 = 1-2\pi_\theta(y_u\mid s_u) + \sum_j\pi_\theta(j\mid s_u)^2 = q_u^2 + \sum_{j \neq y_u} \pi_\theta(j\mid s_u)^2.
\]

Since $\pi_\theta(j\mid s_u) \leq m_u$ for every $j \neq y_u$ by definition, we have
\[
    \sum_{j \neq y_u} \pi_\theta(j\mid s_u)^2\leq  \sum_{j \neq y_u}m_u \pi_\theta(j\mid s_u)= m_u\sum_{j \neq y_u} \pi_\theta(j\mid s_u)=m_u q_u.
\]

Therefore,
\[
    E_u \leq q_u^2 + m_u q_u,
\]
which implies
\[
    m_u\geq\frac{E_u-q_u^2}{q_u}=\frac{E_u}{q_u}-q_u=\frac{E_u}{1-\pi_\theta(y_u\mid s_u)}-(1-\pi_\theta(y_u\mid s_u)).
\]
Moreover, because $q_u \leq 1$, we have
\[
    E_u\leq q_u^2 + m_u q_u \leq 1+m_u.
\]

Thus, if $E_u>1+\delta$, then $m_u>\delta$. By the definition of $m_u$,
there exists some $j \neq y_u$ such that
\[
[g_u]_j=-\pi_\theta(j\mid s_u)=-m_u<-\delta.
\]

In other words, if the energy exceeds 1 by the amount of $\delta$, the $V\times 1$ vector $g_u$ must have a negative force greater than $\delta$ on the token $j=\underset{j\neq y_u}{\text{argmax }}\pi_\theta(\cdot\mid s_u)$.

Furthermore, we can also have a bound for the inverse claim, i.e., if $g_u$ has a very negative force, the resulting energy must also be large.
Specifically, if there exists a non-target token $j \neq y_u$ such that
$[g_u]_j < -\delta$, then
\[
    E_u = \|g_u\|_2^2 > 2\delta^2.
\]
Indeed, recalling that $m_u = \max_{j \neq y_u}\pi_\theta(j\mid s_u)$ and
$q_u = 1-\pi_\theta(y_u\mid s_u)$, we have $q_u \geq m_u$, and therefore
\[
E_u=q_u^2+\sum_{j\neq y_u}\pi_\theta(j\mid s_u)^2\geq q_u^2+m_u^2\geq 2m_u^2.
\]
The bound is tight when all non-target probability mass is concentrated
on a single token.
Note that this bound is not as tight as that in the forward direction:
if $\delta=0.9$, which means the largest negative component in $g_u$ is $-0.9$, we can only conclude $E_u\geq 2m_u^2=2(0.9)^2=1.62$.
\end{proof}

\subsection{Formal Analysis of Cross-Task Readout Degeneration}
\label{app:proofs:plasticity}

We now formally describe the degeneration of the readout layer mentioned in \cref{sec:plasticity}.

\begin{restatable}{proposition}{propplasticity}
    \label{prop:plasticity}
    \textbf{Readout degeneration limits future learning.} 
    Under mild norm control and sufficiently limited overlap between task-induced force directions, training on one task can concentrate the shared readout geometry toward its task-relevant directions while weakening transmission along others. For a subsequently encountered task \(\mathcal D_B\), this can reduce \(R_{\mathcal D_B}\), attenuating effective updates to the backbone and hence its one-step learnability.
\end{restatable}

\begin{proof}
We formalize the cross-task attenuation induced by readout updates. 
Consider one updating example $u\sim\mathcal{D}_A$.
Under a first-order SGD update with mild norm control, the readout update can be written as
\[
    \vw_{t+1} = \alpha \vw_t + \eta g_u \vh_u^\top,
\]
where $0<\alpha<1$ denotes the multiplicative contraction. 
One example is decoupled weight decay, which gives $\alpha=1-\eta\lambda$.
But $0<\alpha<1$ can also come from a budget constraint on the readout: if the readout can only grow in some directions at the cost of others, the effect on any fixed direction is the same. Related work on the implicit bias of gradient descent also supports such mild norm-control assumptions in simplified settings \citep{gunasekar2018characterizing}. 
The argument below only requires such a contraction and is not specific to weight decay.

For a probing example $o\sim\mathcal{D}_B$, define its readout transmission as the Rayleigh quotient
\[
    r_o(\vw)\triangleq \frac{g_o^\top\vw\vw^\top g_o}{g_o^\top g_o}= \frac{\|\vw^\top g_o\|_2^2}{\|g_o\|_2^2}.
\]
As throughout our one-step analysis, we evaluate $g_o$ at the current parameters and keep it fixed within this local update. 
This isolates the change induced by the readout geometry itself.

After updating on $u$, we have
\[
    \vw_{t+1}^\top g_o=\alpha\vw^\top_{t} g_o + \eta \vh_u(g_u^\top g_o).
\]
Therefore, by applying triangle inequality, we have
\[
    \sqrt{r_o(\vw_{t+1})} = \frac{\|\vw_{t+1}^\top g_o\|_2}{\|g_o\|_2} \leq 
    \alpha\frac{\|\vw_{t}^\top g_o\|_2}{\|g_o\|_2} + \eta \frac{|g_u^\top g_o|}{\|g_o\|_2}\|\vh_u\|_2.
\]

Let
\[
\rho_{u,o}
\triangleq
\frac{g_u^\top g_o}{\|g_u\|_2\|g_o\|_2}
\]
denote the normalized overlap between the updating and probing logit-gradient directions.
Using the definition of \(\rho_{u,o}\), the triangle-inequality bound becomes
\begin{equation}
\sqrt{r_o(\vw_{t+1})}
\le
\alpha\sqrt{r_o(\vw_t)}
+
\eta |\rho_{u,o}|\,\|g_u\|_2\|h_u\|_2.
    \label{app:plasticity:a1}
\end{equation}

\Cref{app:plasticity:a1} exposes the competition underlying readout degeneration. 
The first term contracts the transmission already available to task B, while the second term measures how much the update from task A replenishes the probing direction of task B. If the two tasks have little overlap in logit-gradient space, this compensating term is small.

The limiting case is relatively straightforward. 
Suppose that task A and task B have disjoint gradient supports, or more generally $g_u^\top g_o=0$.
Then the task A update contributes nothing along the probing direction $g_o$, and $\vw_{t+1}^\top g_o = \alpha \vw^\top_{t} g_o$.
Hence, in this extreme case, we have
\begin{equation}
    r_o(\vw_{t+1})=\alpha^2 r_o(\vw_{t}).
\end{equation}
Thus, any direction that is not supported by the current task is strictly attenuated whenever $\alpha<1$.
The same argument extends to partially overlapping tasks. Define
\[
\epsilon_{u\to B}
\triangleq
\left(
\mathbb E_{o\sim\mathcal D_B}
\left[\rho_{u,o}^2\right]
\right)^{1/2},
\]
and recall the dataset-level transmission score $R_{\mathcal D_B}(\vw)=\mathbb E_{o\sim\mathcal D_B}[r_o(\vw)]$.
Taking the $L_2(\mathcal D_B)$ norm of both sides of \cref{app:plasticity:a1} and applying Minkowski's inequality gives

\begin{align}
    \sqrt{R_{\mathcal D_B}(\vw_{t+1})}
    &=
    \left\|
    \sqrt{r_o(\vw_{t+1})}
    \right\|_{L_2(\mathcal D_B)}\nonumber\\
    &\le
    \alpha
    \left\|
    \sqrt{r_o(\vw_t)}
    \right\|_{L_2(\mathcal D_B)}
    +
    \eta\|g_u\|_2\|h_u\|_2
    \left\|
    \rho_{u,o}
    \right\|_{L_2(\mathcal D_B)}\nonumber\\
    &=
    \alpha\sqrt{R_{\mathcal D_B}(\vw_t)}
    +
    \eta\|g_u\|_2\|h_u\|_2\epsilon_{u\to B},
    \label{app:plasticity:a4}
\end{align}
where \(\|f(o)\|_{L_2(\mathcal D_B)} \triangleq (\mathbb E_{o\sim\mathcal D_B}[f(o)^2])^{1/2}\).

Consequently, the readout transmission associated with task B decreases whenever
\begin{equation}
  \eta\|g_u\|_2\|h_u\|_2\epsilon_{u\to B}<(1-\alpha)\sqrt{R_{\mathcal D_B}(w_t)}. 
\end{equation}

This is a sufficient condition for cross-task readout degeneration. 
\textit{It states that when task B is only weakly supported by the current task-A gradient, the update received by the B-relevant directions cannot compensate for their contraction under norm control}.

To isolate the evolution of the readout geometry, we keep the probing
force directions $\{g_o:o\sim\mathcal D_B\}$ fixed over the following
multi-step bound.
Allowing the probing forces themselves to evolve would require a
trajectory-level analysis beyond this local bound.
The effect accumulates under repeated training. 
Consider a sequence of updates
$u_1,\ldots,u_T\sim\mathcal D_A$, and suppose
\(
\|g_{u_\tau}\|_2\|h_{u_\tau}\|_2
\epsilon_{u_\tau\to B}
\le \delta
\)
for all $\tau$. Repeatedly applying \cref{app:plasticity:a4} gives
\begin{equation}
\sqrt{R_{\mathcal D_B}^{(T)}}
\le
\alpha^T\sqrt{R_{\mathcal D_B}^{(0)}}
+
\eta\delta
\sum_{\tau=0}^{T-1}\alpha^\tau
=
\alpha^T\sqrt{R_{\mathcal D_B}^{(0)}}
+
\eta\delta
\frac{1-\alpha^T}{1-\alpha}.
\end{equation}

For negligible cross-task overlap, $\delta\approx0$, and the transmission approximately follows
\begin{equation}
    R^{(T)}_{\mathcal{D}_B}\lesssim \alpha^{2T}R^{(0)}_{\mathcal{D}_B}.
\end{equation}

Hence, repeated updates on task A provide task-dependent replenishment to directions supported by A, while directions weakly supported by the current task progressively lose readout transmission.
This provides the Rayleigh-quotient counterpart of the block-wise interpretation in the main text.

Finally, we connect the degeneration of $R_{\mathcal D_B}$ to plasticity on task B.
After training on task A, consider a subsequent update on an example
$u\sim\mathcal D_B$.
As in \cref{sec:plasticity}, plasticity concerns the confidence change of the updating example itself, and we therefore set $o=u$.
The \texttt{CH2} contribution to its one-step improvement contains the factor
\(
g_u^\top \vw\vw^\top g_u
=
\|g_u\|_2^2 r_u(\vw).
\)
Hence,
\[
\Delta^{\texttt{CH2}}_u
=
\eta
\|g_u\|_2^2
r_u(\vw)
\left(
\sum_{\ell=0}^{L-1}
\|\tilde \vh_{\ell,u}\|_2^2
+
\kappa_{\mathrm{embd}}(s_u,s_u)
\right),
\]
up to the same approximation used in our learning-dynamics decomposition.

Thus, $r_u(\vw)$ is exactly the multiplicative readout-transmission factor governing the CH2 contribution for example $u$, while
\[
R_{\mathcal D_B}
=
\mathbb E_{u\sim\mathcal D_B}[r_u(\vw)]
\]
summarizes this transmission across task B.
A reduction in $R_{\mathcal D_B}$ therefore indicates weaker average readout transmission for the new task.
When the output-level signal and residual-state factors remain of comparable scale and do not systematically increase enough to compensate for this attenuation, the available CH2 update is correspondingly reduced.
This connects cross-task attenuation of the readout Rayleigh quotient to a reduction in effective one-step adaptation on the subsequent task, yielding the plasticity-loss mechanism described in \cref{prop:plasticity}.
\end{proof}

\textbf{Scope of the bound.}
The contraction factor $\alpha$ need not arise from explicit weight decay. 
It only requires some budget constraint on the readout: if the readout can only grow in some directions at the cost of others, the same attenuation follows. 
This explains why we still see degeneration without weight decay.
When the training distribution is narrow, a few directions take most of the budget, and the directions needed by new tasks are squeezed out. 
We do not derive this constraint here, and leave the exploration of other constraining mechanisms as an open question.

\subsection{Beyond SGD: What Changes under Adaptive Optimization?}
\label{app:proofs:adamw}
Our analysis uses SGD as the canonical optimization setting because it exposes the interaction structure in a particularly simple form.
In practice, however, LLM fine-tuning commonly uses adaptive optimizers such as Adam or AdamW; indeed, all experiments in this work are performed with AdamW.
We therefore briefly discuss which conclusions depend on the SGD assumption and which arise from the model structure itself.

The main distinction is \textit{structural} versus \textit{quantitative}.
Adaptive optimization reweights parameter-coordinate interactions and can therefore change exact attribution values, relative channel strengths, and pairwise rankings.
In contrast, the force-kernel-force form, the decomposition into the same parameter-induced interaction channels, and the role of the shared readout remain structurally unchanged.
Decoupled weight decay further introduces a candidate-independent local drift and therefore does not affect the ranking over candidate updates for a fixed observation.

\textbf{A preconditioned view of Adam.}
Recall \cref{eq:def_confidence_change} where we track the model's confidence change
\[
    \Delta_t(o,u)\approx \langle \nabla_\theta\log\pi(y_o\mid s_o), \theta_{t+1}-\theta_t\rangle
\]
Note that the $\nabla_\theta\log\pi(y_o\mid s_o)$ part does not depend on the optimizer: it is just the gradient of $(y_o,s_o)$ given to the model $\theta_{t}$, originating from the 1st-order Taylor expansion.
The difference arise in the calculation of $\Delta\theta_t\triangleq\theta_{t+1}-\theta_t$.
Let
$G_u \triangleq \nabla_\theta \log \pi(y_u\mid s_u)$.
For SGD, we have
\[
\Delta\theta_t^{\mathrm{SGD}}(u)=\eta_t G_u.
\]

However, different optimizers will make $\Delta\theta_t$ different.
For the adaptive learning rate mechanisms applied in Adam and AdamW, we need a \textit{precondition matrix} to capture it.
Following standard analyses of adaptive optimization \citep{malladi2022sdes,malladi2023kernel,xia2024less}, we consider the local preconditioned approximation
\(
\Delta\theta_t^{\mathrm{Adam}}(u)
\approx
\eta_t P_t G_u
\),
where $P_t \succ 0$ is a diagonal, optimizer-state-dependent preconditioner.
AdamW additionally applies decoupled weight decay \citep{loshchilov2017decoupled}: 
\(
\Delta\theta_t^{\mathrm{AdamW}}(u)
\approx
\eta_t P_t G_u
-\eta_t\lambda\theta_t.
\)
As a result, we have
\begin{align}
    \Delta_t^{\mathrm{SGD}}(o,u)
    &\approx
    \eta_t G_o^\top G_u,\nonumber\\
    \Delta_t^{\mathrm{Adam}}(o,u)
    &\approx
    \eta_t G_o^\top P_tG_u,\nonumber\\
    \Delta_t^{\mathrm{AdamW}}(o,u)
    &\approx
    \eta_t G_o^\top P_tG_u
    -\eta_t\lambda G_o^\top\theta_t.
\end{align}
The first-moment accumulator contributes an additional history-dependent drift shared across candidate updates, which we omit to isolate the interaction induced by the incoming gradient. 
Accordingly, the signed score reflects this local gradient interaction rather than the exact sign of the full AdamW.

\textbf{The structural decomposition is preserved.}
Recall that by separating the softmax layer, 
the parameter gradient can be written as
$G_{o/u}=\mathcal{J}_{o/u}^{\top}g_{o/u}$,
where $\mathcal{J}_{o/u}\triangleq\nabla_{\theta} \vz_{o/u}$ is the full-parameter logit Jacobian
and $g_{o/u}$ is the token-space force used throughout the main text.
Under Adam,
\[
G_o^\top P_tG_u
=
g_o^\top
\underbrace{\mathcal{J}_oP_t\mathcal{J}_u^\top}_{\mathcal{K}^{(P_t)}(o,u)}
g_u.
\]
Hence adaptive preconditioning preserves the same force-kernel-force form, with the Euclidean parameter-space kernel replaced by an optimizer-weighted kernel.

More importantly, Adam does not introduce new cross-parameter paths.
Partition the parameters into the shared readout $\vw$ and residual-block parameters $\{\phi_\ell\}_{\ell=0}^{L-1}$.
Since the Adam preconditioner is block-diagonal, 
\[
P_t
=
\operatorname{diag}
\left(
P_t^{\vw},
P_t^{0},\ldots,P_t^{L-1}
\right),
\]
and therefore
\[
G_o^\top P_tG_u
=
(G_o^{\vw})^\top P_t^{\vw}G_u^{\vw}
+
\sum_{\ell=0}^{L-1}
(G_o^{\ell})^\top P_t^{\ell}G_u^{\ell}.
\]
The same parameter paths that produce \texttt{CH1} and \texttt{CH2} under SGD therefore remain present under adaptive optimization.
The optimizer changes their weighting, but not their computational origin.

The role of the shared readout is similarly preserved.
For a residual-block weight matrix, the gradient has the outer-product form
\[
G_{\ell,o/u}
=
\operatorname{vec}(a_{o/u} \vh_{o/u}^\top),
\qquad
a_{o/u}=\vw^\top g_{o/u}.
\]
Under SGD, its interaction factorizes as
\[
I_{\ell}^{\mathrm{SGD}}
=
(a_o^\top a_u)(\vh_o^\top \vh_u)
=
(g_o^\top \vw\vw^\top g_u)(\vh_o^\top \vh_u).
\]
Under a diagonal Adam preconditioner
$P_{\ell,t}=\operatorname{diag}(p_{\ell,ij})$,
the same interaction becomes
\[
I_{\ell}^{\mathrm{Adam}}
=
\sum_{i,j}
p_{\ell,ij}\,
a_{o,i}a_{u,i}
h_{o,j}h_{u,j} = (a_o \odot a_u)^\top D_\ell (\vh_o \odot \vh_u),
\]
where $D_\ell\in\mathbb{R}^{d_{out}\times d_{in}}$ and $(D_\ell)_{ij}=p_{\ell,ij}$ is the reshaped version of $P_\ell$.
For SGD case, $D_\ell = \mathbf{1}\mathbf{1}^\top$, and hence the element-wise expression can be simplified to the inner product.

Thus Adam reweights the existing parameter-coordinate interactions rather than creating a new interaction mechanism.
In particular, since $p_{\ell,ij}>0$, the sign of each individual
coordinate contribution is unchanged, although their relative strengths and
therefore the sign or magnitude of the aggregate interaction may change.
The shared $\vw$ remains the source of the projected coupling
$a_{o/u}=\vw^\top g_{o/u}$ that underlies \texttt{CH2}.

\textbf{Decoupled weight decay does not affect local candidate ranking.}
Comparing AdamW with Adam
\[
\Delta_t^{\mathrm{AdamW}}(o,u)
=
\Delta_t^{\mathrm{Adam}}(o,u)
-\eta_t\lambda G_o^\top\theta_t.
\]
For a fixed observation $o$ and checkpoint $\theta_t$, the second term is
independent of the candidate update $u$.
Consequently,
\[
\operatorname{rank}_{u}
\Delta_t^{\mathrm{AdamW}}(o,u)
=
\operatorname{rank}_{u}
\Delta_t^{\mathrm{Adam}}(o,u)
\]
under this local analysis.
The same conclusion holds when weight decay is applied only to a subset of
parameters by replacing $\theta_t$ with the corresponding masked parameters.

\paragraph{What does change?}
The preceding invariances are structural rather than numerical.
Adaptive preconditioning changes the weights assigned to individual parameter
coordinates.
It can therefore change the exact attribution magnitude, the relative strength
of \texttt{CH1} and \texttt{CH2}, and the ranking or aggregate sign of pairs
whose contributions are comparable.
These quantities should therefore be regarded as optimizer dependent.
In contrast, our main mechanistic conclusions concern the origin and geometry of
the interaction channels, rather than their exact optimizer-specific magnitudes.

This distinction is also consistent with prior optimizer-aware attribution work.
For example, LESS \citep{xia2024less} explicitly compares data selection using
SGD, SignGD, and Adam influence formulations.
Their average downstream scores are $49.7$, $47.8$, and $50.5$, respectively,
compared with $46.0$ for random selection.
Thus the Adam-aware formulation improves attribution fidelity, while the simpler
SGD formulation already retains most of the downstream utility of the
optimizer-aware score.
This motivates our use of SGD as a clean analytical lens for exposing the
interaction structure.

Finally, the fixed-preconditioner approximation does not capture all details of
Adam, in particular the history dependence introduced by the first- and
second-moment estimates.
A fully trajectory-aware treatment of adaptive optimization is beyond the scope
of this work.
Our goal here is narrower: to clarify that practical optimizers primarily modify
the quantitative weighting of the interactions analyzed in the main text, while
leaving their structural origin intact.
\section{Understanding the Validity Regime of the Approximation}
\label{app:diagnose}

\cref{sec:related_works:limitations} summarizes the assumptions used in deriving our framework.
Here we examine these assumptions more systematically.
We distinguish two questions. 
First, how robust is the structural decomposition in \cref{eq:score} to common changes in optimization and model architecture?
Second, given this structure, under what training regimes does the forward-computable score faithfully approximate the exact local interaction? The latter reveals a particularly interesting behavior: interaction magnitude is reliable even at a cold start, whereas signed fidelity can initially fail and then recover rapidly after only a small amount of in-distribution adaptation. 
Together, these results characterize the practical operating regime of our approximation and clarify which conclusions in Section~\ref{sec:attribution}--\ref{sec:plasticity} depend on which aspects of its fidelity.

\subsection{How Robust Is the Structure of \cref{eq:score}}
\label{app:diagnose:01}

\textbf{The \texttt{CH1}+\texttt{CH2} structure is robust to common optimizer choices.}
\cref{eq:score} is derived under SGD, but introducing an adaptive optimizer does not remove the two-channel decomposition.
A parameter-wise preconditioner simply reweights the corresponding parameter-space interaction, while \texttt{CH1} remains the direct force-alignment channel and \texttt{CH2} remains the backbone-mediated channel coupled through the shared readout. 
Such reweighting can change individual scores and, consequently, their ranking.
By contrast, the decoupled weight-decay term in AdamW is independent of the selected updating example at a fixed model state, and therefore does not affect candidate ranking. 
We provide the corresponding derivation in Appendix~\ref{app:proofs:adamw}.

More generally, the decomposition follows from the basic residual-flow structure rather than a particular Transformer implementation.
Standard single-stream residual architectures retain the same two-channel form, while architectures with multiple residual streams, such as Hyper-Connections \citep{zhu2025hyperconnections} and the four-branch Gated Residual design in \texttt{Qwen3.8-Next} \citep{qwen2026design}, require the corresponding residual paths to be incorporated explicitly.
Thus, architectural changes modify the path structure of \texttt{CH2} rather than the underlying force-kernel-force view.

\textbf{Tied embeddings introduce additive coupling terms.}
Many language models tie the output readout to the input embedding matrix. 
In this case, the shared parameter receives gradients from both roles. 
Writing the gradient on the tied matrix as \(G=G^{\mathrm{out}}+G^{\mathrm{in}}\), its interaction contains
\[
    \langle G_o, G_u \rangle = \langle G_o^\text{out}, G_u^\text{out} \rangle + \langle G_o^\text{out}, G_u^\text{in} \rangle
                              +\langle G_o^\text{in}, G_u^\text{out} \rangle + \langle G_o^\text{in}, G_u^\text{in} \rangle.
\]
The first term is already captured by the standard readout contribution, while the remaining terms arise from tying the input and output parameterizations. 
Thus, tied embeddings add corrections rather than replacing the \texttt{CH1}+\texttt{CH2} structure.
Across the settings we tested, explicitly including these terms changes the resulting scores only modestly, so we use the simpler \cref{eq:score} throughout the main experiments. 
We leave a more complete treatment of tied parameterizations to future work.

\subsection{When Is the Approximation Numerically Reliable?}
\label{app:diagnose:02}

\begin{figure}[htbp]
    \centering
    \begin{subfigure}[b]{0.32\textwidth}
        \centering
        \includegraphics[width=\textwidth]{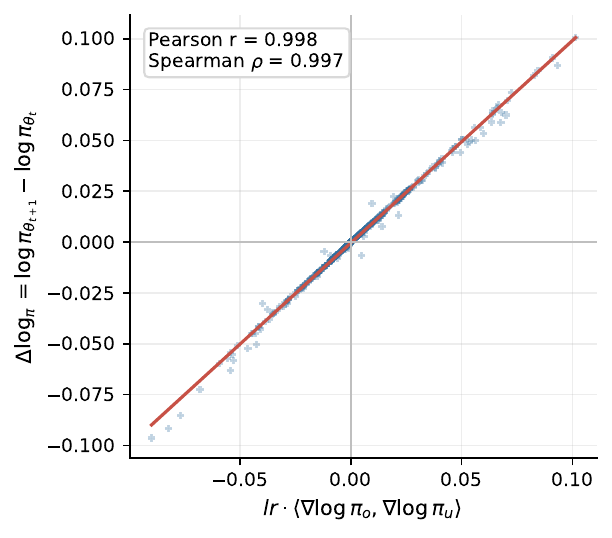}
        \caption{$\Delta\log\pi$ vs. first-order influence}
        \label{fig:sub1}
    \end{subfigure}
    \hfill
    \begin{subfigure}[b]{0.32\textwidth}
        \centering
        \includegraphics[width=\textwidth]{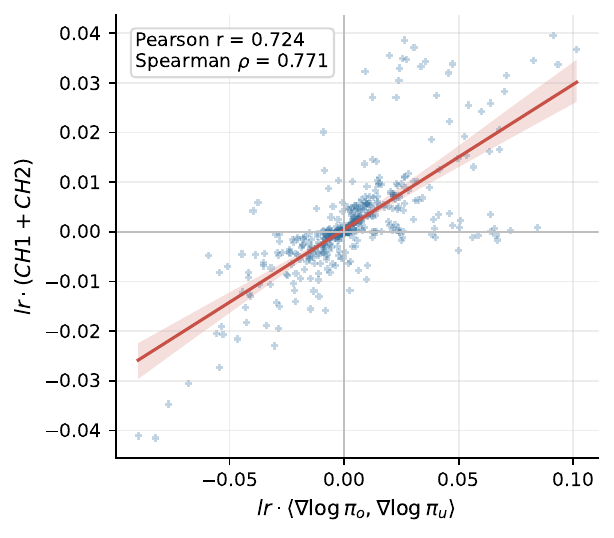}
        \caption{Eq.\ref{eq:score} vs. first-order influence}
        \label{fig:sub2}
    \end{subfigure}
    \hfill
    \begin{subfigure}[b]{0.32\textwidth}
        \centering
        \includegraphics[width=\textwidth]{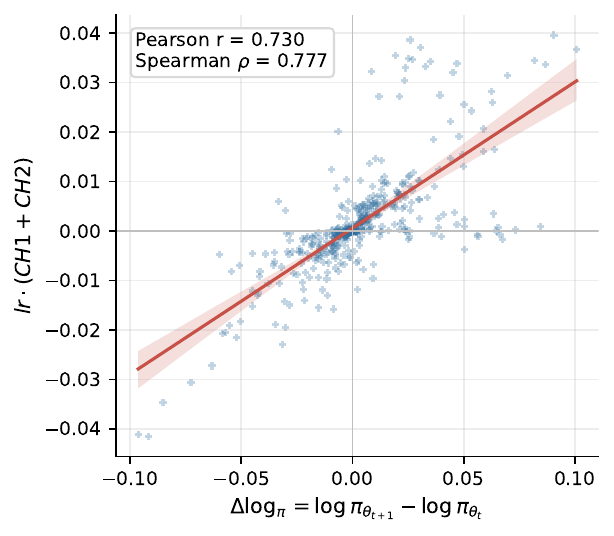}
        \caption{Eq.\ref{eq:score} vs. $\Delta\log\pi$}
        \label{fig:sub3}
    \end{subfigure}
    \caption{Empirical validation and geometry of the two-channel approximation.
            (a–c) Token-level validation over sampled \texttt{GSM8K} updating and \texttt{MMLU} observing token pairs, comparing the measured one-step log-probability change, the exact first-order gradient interaction, and our forward-computable approximation at an early in-distribution checkpoint obtained after a brief adaptation on the updating distribution. Correlations are reported in each panel.
            }
    \label{fig:validation}
\end{figure}

\cref{eq:score} is obtained through several approximations with different numerical consequences.
The main ones are the first-order expansion of the observed log probability, the separation of the direct readout contribution, and the approximation of backbone propagation through the residual flow. 
\cref{fig:validation} shows that their composition is accurate in the operating regime used throughout the paper. 
We now isolate these approximations and characterize where fidelity is gained or lost.

\begin{figure}[t]
    \centering
    \begin{subfigure}[t]{0.48\linewidth}
        \centering
        \includegraphics[width=\linewidth]{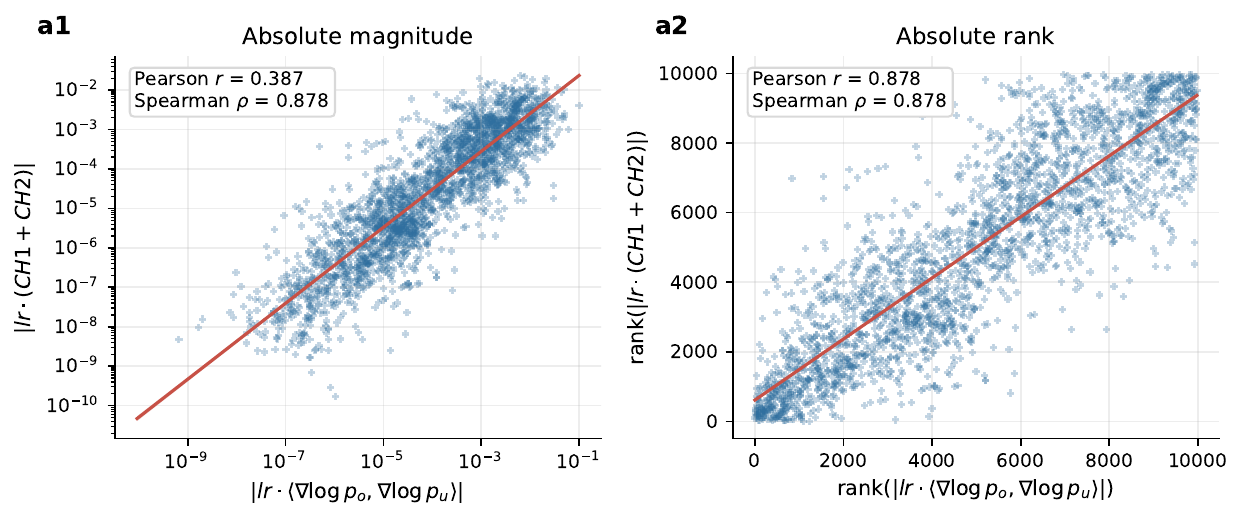}
        \caption{Unsigned approximation quality.}
    \end{subfigure}
    \hfill
    \begin{subfigure}[t]{0.48\linewidth}
        \centering
        \includegraphics[width=\linewidth]{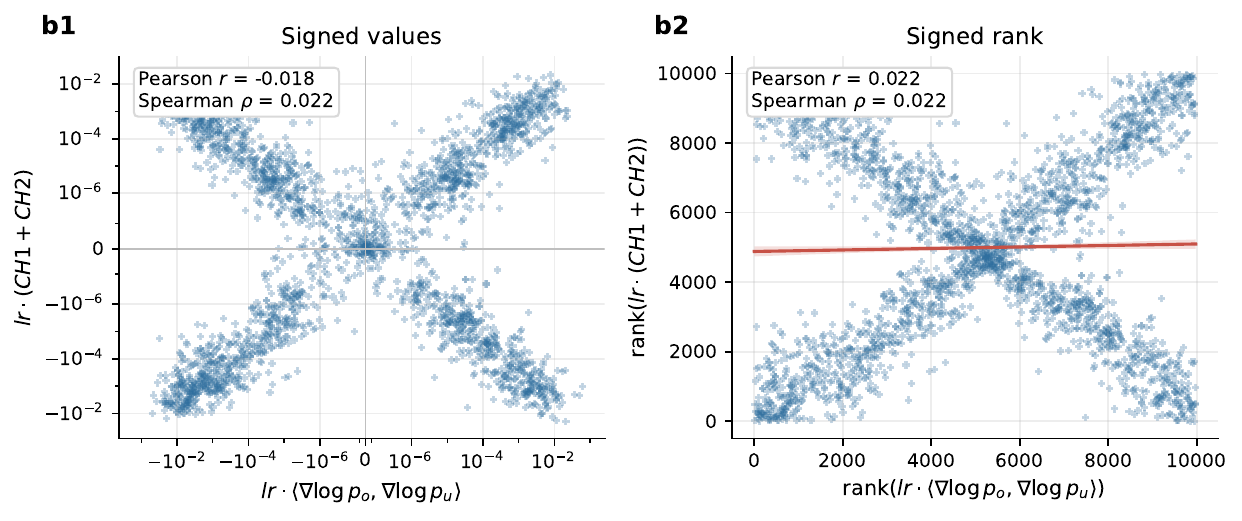}
        \caption{Signed approximation failure.}
    \end{subfigure}
    \caption{Cold start: magnitude is preserved more reliably.
            (a) The approximation tracks the absolute gradient interaction well in both value and rank. (b) In the signed setting, the characteristic X-shaped pattern reveals frequent sign flips despite preserved magnitude structure. The same behavior is observed across multiple models and dataset pairs; one representative setting is shown here.
            (\texttt{GSM8K} as $u$-side and \texttt{MMLU} as $o$-side).}
    \label{app:fig:scatter_x}
\end{figure}

\begin{figure}[t]
    \centering
    \includegraphics[width=0.8\linewidth]{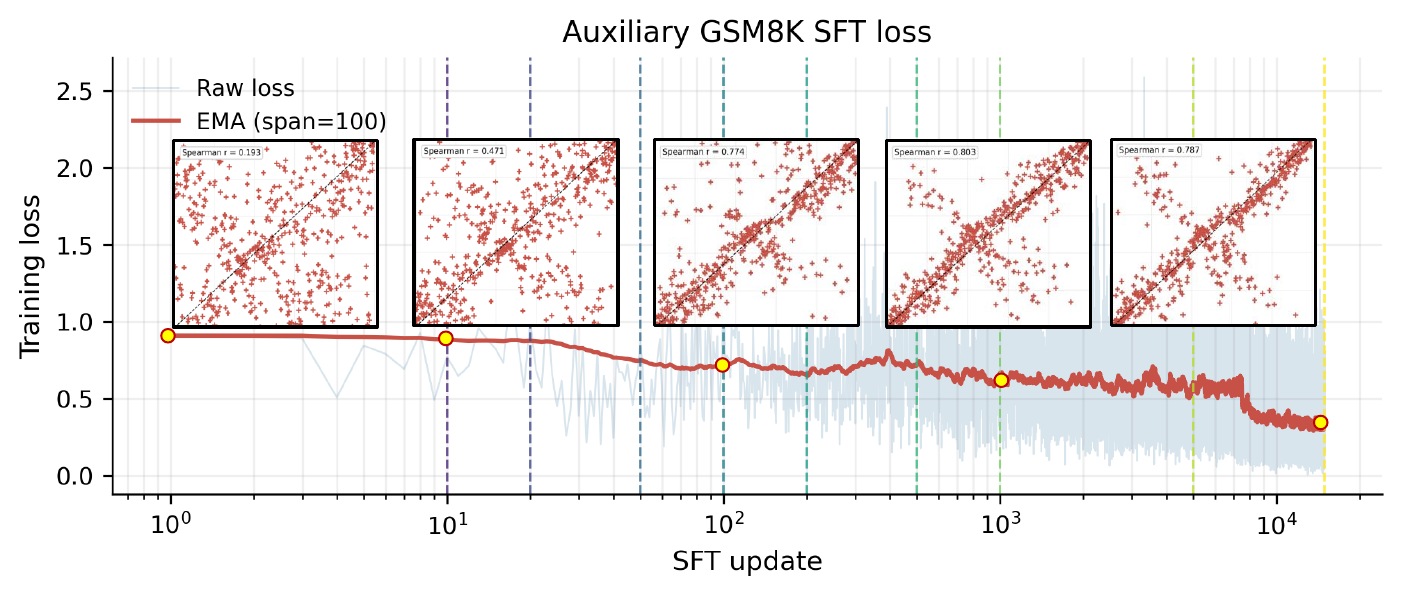}
    \caption{Signed approximation quality recovers within the first 1\% of updates. 
            Auxiliary SFT experiment on \texttt{GSM8K} for 2 epochs; insets show scatter plots of the exact first-order interaction against \cref{eq:score} at the marked checkpoints. At initialization the scatter also exhibits the characteristic X-shape, indicating frequent sign flips despite preserved magnitude structure. The signed geometry aligns rapidly and is largely recovered before the training loss shows systematic improvement. (\texttt{GSM8K} as $u$ and \texttt{Dolly} as $o$). Similar trends across different models and update--observation pairs.}
    \label{fig:app:loss_and_scatter}
\end{figure}

\begin{figure}[t]
    \centering
    \includegraphics[width=0.8\linewidth]{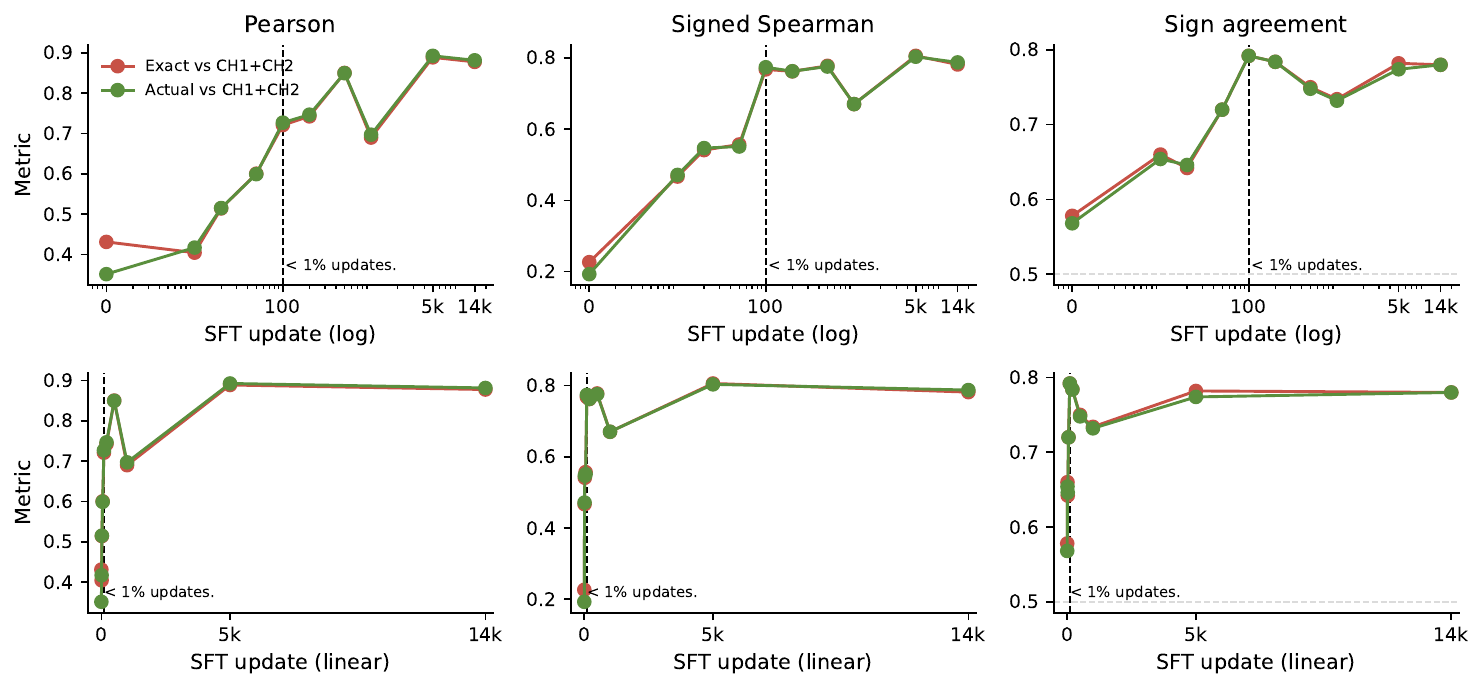}
    \caption{Signed metrics improve early and remain stable thereafter. Pearson correlation, signed Spearman correlation, and sign agreement between \cref{eq:score} and both the exact first-order interaction (``Exact'') and the measured log-probability change (``Actual''), evaluated at checkpoints along the auxiliary SFT run. Top: log-scaled x-axis; bottom: linear x-axis. Most of the improvement occurs within the first 100 updates, below 1\% of the run (dashed line), after which all three metrics stay flat through 14k updates. The two curves nearly coincide, indicating that the residual error comes from the structural approximations rather than the first-order expansion.}
    \label{fig:app:step1plus_checkpoint_metric_comparison}
\end{figure}

\begin{figure}[t]
    \centering
    \includegraphics[width=0.8\linewidth]{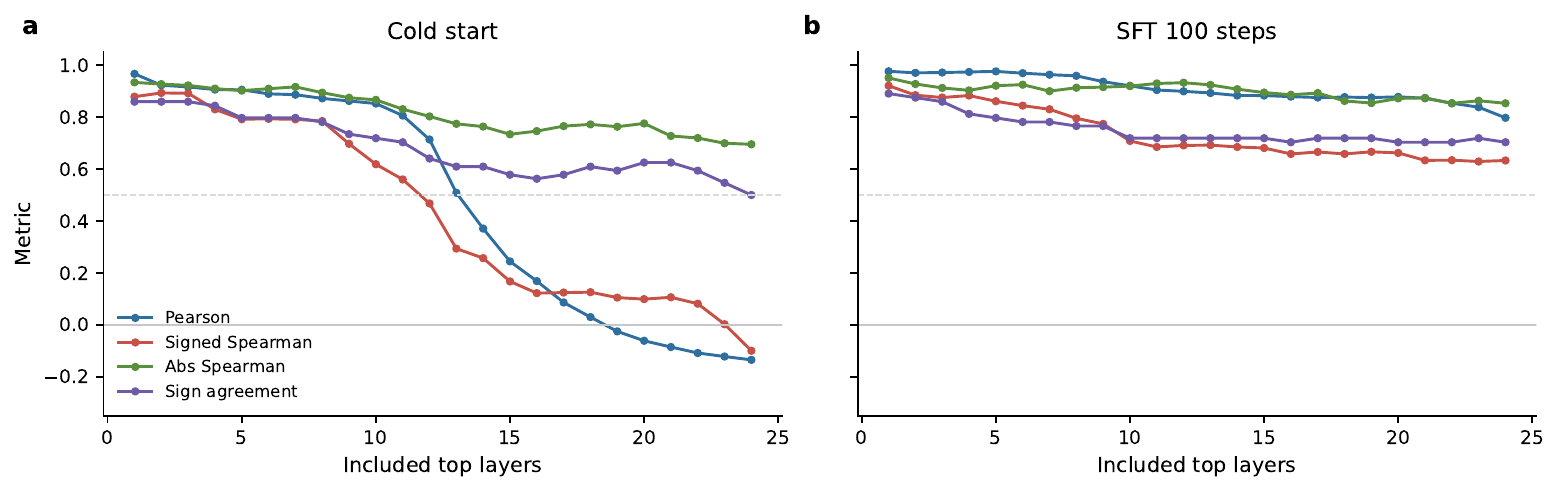}
    \caption{A brief warm start restores signed fidelity across depth. Approximation quality for \texttt{CH2} as lower layers are progressively included, starting from the output. (a) At initialization, signed metrics degrade sharply once middle layers enter the sum, reaching chance-level sign agreement and near-zero signed correlation when all 24 layers are aggregated, while absolute Spearman correlation remains around 0.7. (b) After 100 auxiliary SFT updates, the same aggregation retains signed Spearman above 0.6 and sign agreement near 0.7 at full depth. Dashed line marks chance-level sign agreement (0.5). Magnitude structure is thus preserved regardless of depth, whereas signed fidelity at depth requires a warm start.}
    \label{fig:app:layer_accuracy}
    \vspace{-10pt}
\end{figure}

\textbf{The first-order expansion contributes little error.}
Across the fine-tuning trajectory, the exact first-order gradient interaction remains extremely close to the measured one-step change in log probability.
This is visible in \cref{fig:app:step1plus_checkpoint_metric_comparison}, where the curves comparing \cref{eq:score} with the exact first-order interaction and with the measured change nearly coincide. 
Thus, the dominant approximation error studied below does not arise from the first-order Taylor expansion.

\textbf{\texttt{CH1} is essentially exact under readout-only updates.}
Freezing the backbone and updating only the readout yields nearly perfect agreement between \texttt{CH1} and the measured log-probability change (Pearson/Spearman \(\approx1.0\), MAE \(=2.2\times10^{-9}\)).
The nontrivial approximation error therefore arises primarily from \texttt{CH2} and residual-backbone propagation.

\textbf{Magnitude is reliable even at a cold start.}
At initialization, \texttt{CH2} exhibits a striking asymmetry between magnitude and sign.
As shown in \cref{app:fig:scatter_x}-(a), the approximation preserves the absolute interaction remarkably well in both value and rank.
In the representative setting shown here, absolute-value Spearman correlation reaches \(0.878\). 
However, the signed scatter in \cref{app:fig:scatter_x}-(b) forms a characteristic X-shaped pattern: many interactions have approximately correct magnitudes but reversed signs, reducing signed rank correlation to nearly zero. 
We observe the same qualitative behavior across multiple models and dataset pairs.

\textbf{Signed fidelity recovers rapidly and remains stable.}
The cold-start mismatch is highly transient. 
We perform an auxiliary SFT (same settings with most of our downstream SFT experiments, only batch-size is one) on data drawn from the updating distribution and recompute the score throughout training. 
As shown in \cref{fig:app:loss_and_scatter} and \ref{fig:app:step1plus_checkpoint_metric_comparison}, Pearson correlation, signed Spearman correlation, and sign agreement all improve sharply during the first \(\sim100\) updates, less than \(1\%\) of the full run, and remain stable thereafter. 
In contrast, the unsigned correlation is already high at initialization. 
The recovery therefore does not require convergence or substantial task learning. 
In our experiments, a brief adaptation is sufficient to enter the stable regime.

\textbf{Cold-start sign error accumulates with residual depth.}
\cref{fig:app:layer_accuracy} localizes this failure by progressively aggregating \texttt{CH2} contributions from the output side toward earlier Transformer layers. 
At a cold start, shallow approximations remain accurate, but signed Pearson correlation, signed Spearman correlation, and sign agreement deteriorate as deeper layers enter the sum. 
With all 24 layers included, the signed metrics approach chance even though absolute-rank correlation remains around \(0.7\).
After only 100 auxiliary SFT updates, the same aggregation remains reliable across the full depth of the network.
This suggests that individual local contributions are not the main source of failure; rather, signed error emerges when many residual paths are composed and partially cancel.

\textbf{Large interactions are more sign-stable.}
Signed errors are also concentrated among weaker interactions. 
\cref{app:fig:bin_accuracy} bins token pairs by the magnitude of the exact interaction. 
At a cold start, sign agreement remains near chance over much of the distribution, while the strongest interactions are already more reliable. After the brief warm start, sign agreement rises substantially and reaches around \(0.9\) for large-magnitude interactions.
Thus, when signed predictions must be used near initialization, confidence can be improved by restricting attention to sufficiently strong interactions.

\begin{wrapfigure}{r}{0.39\linewidth}
    \centering
    \vspace{-5mm}
    \includegraphics[width=\linewidth]{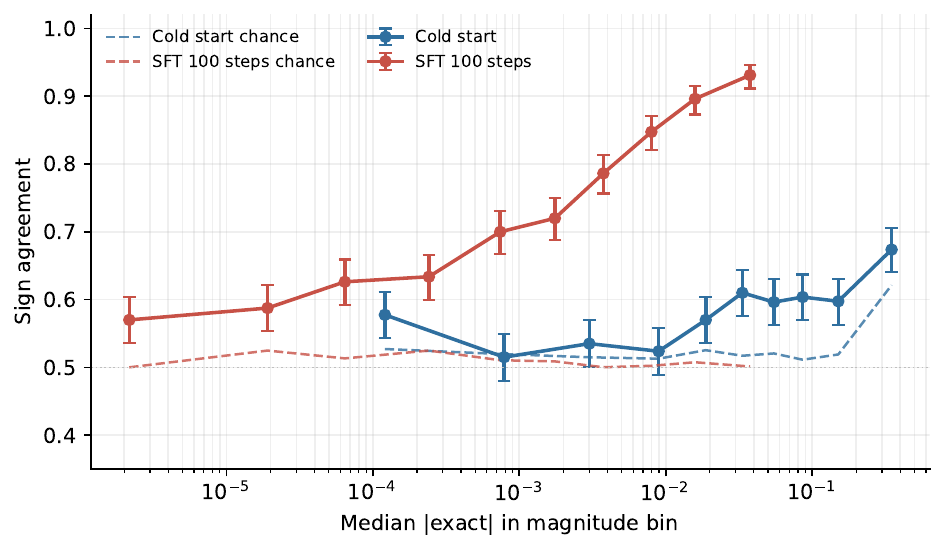}
    \vspace{-5mm}
    \caption{Sign agreement as a function of interaction magnitude before and after a brief warm start.}
    \label{app:fig:bin_accuracy}
    \vspace{-4mm}
\end{wrapfigure}

\subsection{Why Do We Stop at the Identity Path?}
\label{app:diagnose:03}
Our approximation retains only the first-order residual path \(J_{\ell}^{(1)}\), which is the key to keeping the resulting score forward-computable.
Under the same approximation \(f_\ell(h)\approx M_\ell h\), \(J_{\ell}^{(1)}=B_\ell\) contains only the identity residual propagation and therefore reduces to interactions between forward representations.
Moving beyond this path, however, rapidly changes the computational structure of the approximation.

Consider the \(k\)-th block from the output.
Its residual expansion contains one \(J^{(1)}\) path, \(k-1\) different \(J^{(2)}\) paths, \(\binom{k-1}{2}\) different \(J^{(3)}\) paths, \(\binom{k-1}{3}\) different \(J^{(4)}\) paths, and so forth.
More generally, the number of \(J^{(r)}\) paths grows as
\(
    \binom{k-1}{r-1}.
\)
However, even the second-order term,
\[
    J_{\ell}^{(2)}
    =
    \sum_{k>\ell} A_k B_\ell
    \approx
    \sum_{k>\ell} W_k B_\ell,
\]
requires explicitly introducing layer-dependent transformations, while higher-order terms involve products of multiple such transformations along different residual paths.
Consequently, higher-order paths introduce both combinatorial path growth and explicit layer-dependent propagation, progressively eroding the computational advantage over direct gradient-inner-product evaluation.
As more orders are retained, the advantage over direct gradient-inner-product computation quickly diminishes.

We therefore retain \(J^{(1)}\) as a deliberate computational choice.
It preserves a particularly simple structure based on forward quantities, making the score inexpensive to evaluate and recompute throughout training.
This scalability is important for our goal of using the score as an online probe of learning dynamics rather than as an exact reconstruction of the full gradient interaction.

\subsection{Practical Implications and Usage Guidelines}
\label{app:diagnose:04}

The results above clarify how the score should be used in practice.
Quantities that depend primarily on interaction magnitude can be evaluated reliably even near a cold start, whereas signed interactions are more faithful after a brief in-distribution adaptation.
Since our score is lightweight and forward-computable, an alternative is to recompute it periodically as training proceeds.
This is particularly natural in continual learning, where the model state and its local interaction geometry evolve together.
Unlike the earlier eNTK-based framework \citep{ren2025learning}, this strategy does not require the interaction geometry to remain approximately fixed over a long optimization trajectory.

\textbf{Data attribution.}
Data attribution is the application most directly affected by cold-start sign errors, since selecting examples by positive influence requires reliable signed ranking.
Two simple strategies are available.
First, when only interaction strength is required, examples can be ranked by absolute influence, whose fidelity remains high even at initialization.
Second, signed attribution can be performed after a brief in-distribution warm start.
The adaptation data need not contain the exact examples subsequently scored; matching the updating distribution is sufficient in our experiments.
This is also compatible with the short warm-start stages commonly used by gradient-based methods.

\textbf{Forgetting.}
The erosion mechanism in \cref{sec:forgetting} describes weak interactions that accumulate throughout training.
Since the score can be refreshed along the trajectory, this analysis mainly operates in the stable regime identified above.
Collision depends more directly on strong negative interactions.
Its energy criterion, $E_u=\|g_u\|_2^2$, is computed directly from the output force and does not rely on the residual-flow approximation, although the predicted transmission of a negative force through \texttt{CH2} can be affected by cold-start sign errors.
For settings that require signed collision estimates near initialization, we therefore recommend either a brief warm start or online recomputation as training begins.

\textbf{Plasticity loss.}
The analysis in \cref{sec:plasticity} depends primarily on the magnitude and geometry of readout transmission rather than the signed \texttt{CH2} approximation.
Its main conclusions are therefore largely insensitive to the cold-start sign phenomenon.

\textbf{Cost of online evaluation.}
Computing \texttt{CH2} over all Transformer blocks costs $\mathcal{O}(ndL)$, where $n$ is the number of token interactions, $d$ the hidden dimension, and $L$ the number of layers.
The structural decomposition allows this cost to be reduced when only part of the network is relevant.
If only $\ell$ layers are updated or evaluated, the cost becomes $\mathcal{O}(nd\ell)$.
Top-$K$ approximations in vocabulary space can likewise reduce the cost of evaluating $\vw^\top g$.
These properties make periodic recomputation practical and allow the score to serve as an online probe of learning dynamics rather than a fixed approximation computed only at initialization.
For sequence-level scoring, the token-wise contributions are additive, so the observation-side forward quantities can be aggregated before scoring without changing the resulting sequence-level score.
Thus, caching need only retain the aggregated per-example quantities rather than all token-level hidden states.

\textbf{Open questions.}
Two aspects of this transition remain unexplained. 
First, why does signed error accumulate so strongly across residual depth while interaction magnitude remains stable?
Second, why does a very small amount of in-distribution adaptation restore the signed geometry? 
One possible explanation is that a model arriving from pretraining or instruction tuning is locally adapted to its previous data distribution, while a new, relatively off-policy supervision distribution initially induces poorly aligned residual updates. 
A short adaptation phase may reorganize these local directions before substantial task learning occurs. 
This behavior may also relate to the early plateaus often observed in SFT loss curves.
We leave a formal characterization of this to future work.

\subsection{Representation overlap mainly modulates interaction magnitude}
\label{app:diagnose:hh}
Several analyses in the main text rely on the observation that hidden representations in Transformer models often exhibit a narrow-cone geometry \citep{ethayarajh2019contextual, gao2018representation}. 
We therefore empirically verify this property across different models and layers.
Specifically, we randomly select non-overlapping examples as update and observing examples, and calculate their token-level pairwise similarities.
As shown in \cref{fig:app:hh}, pairwise representation overlaps are predominantly positive, with most hidden states lying within a relatively narrow angular region rather than being distributed isotropically. 
This suggests that the representation-overlap term in our decomposition primarily modulates the magnitude of an interaction, while its sign is more strongly determined by the force-side geometry. 
We use this observation repeatedly in the subsequent analyses.

\begin{figure}[htbp]
    \centering
    \begin{subfigure}[b]{0.8\textwidth}
        \centering
        \includegraphics[width=\textwidth]{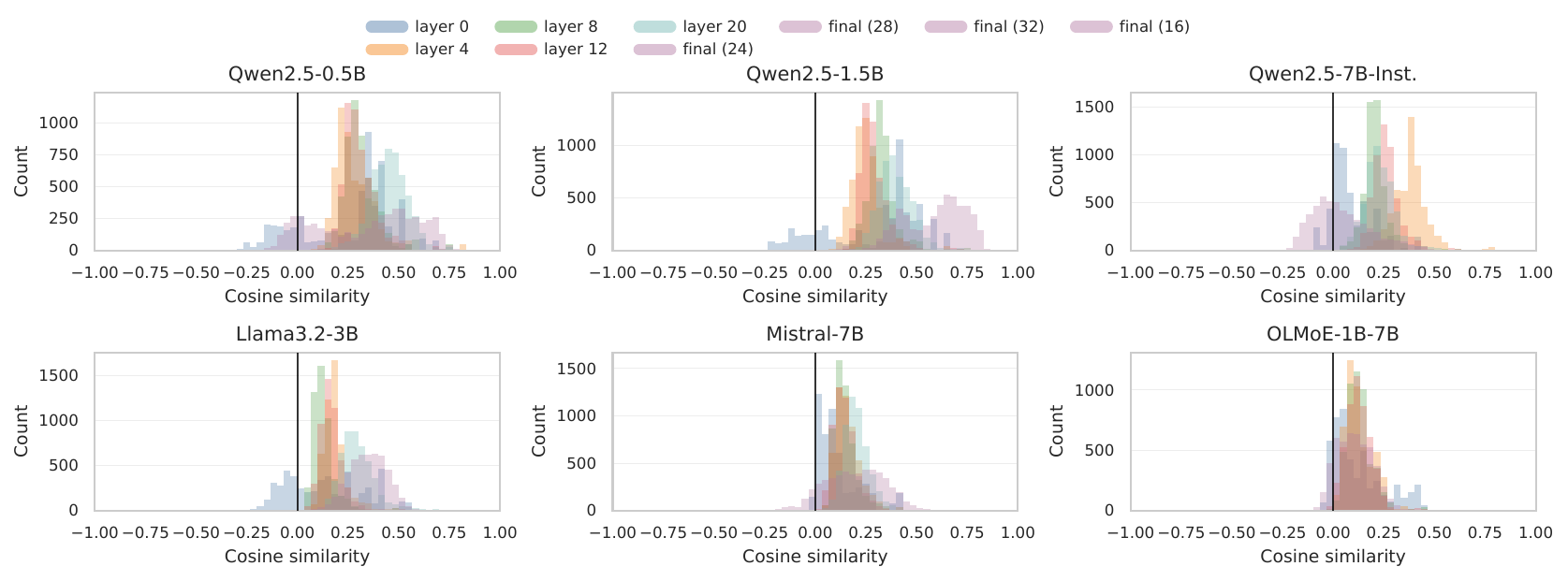} 
        \caption{Layer-wise distributions of pairwise hidden-state cosine similarity.}
    \end{subfigure}
    
    \vspace{1em} 
    
    \begin{subfigure}[b]{0.8\textwidth}
        \centering
        \includegraphics[width=\textwidth]{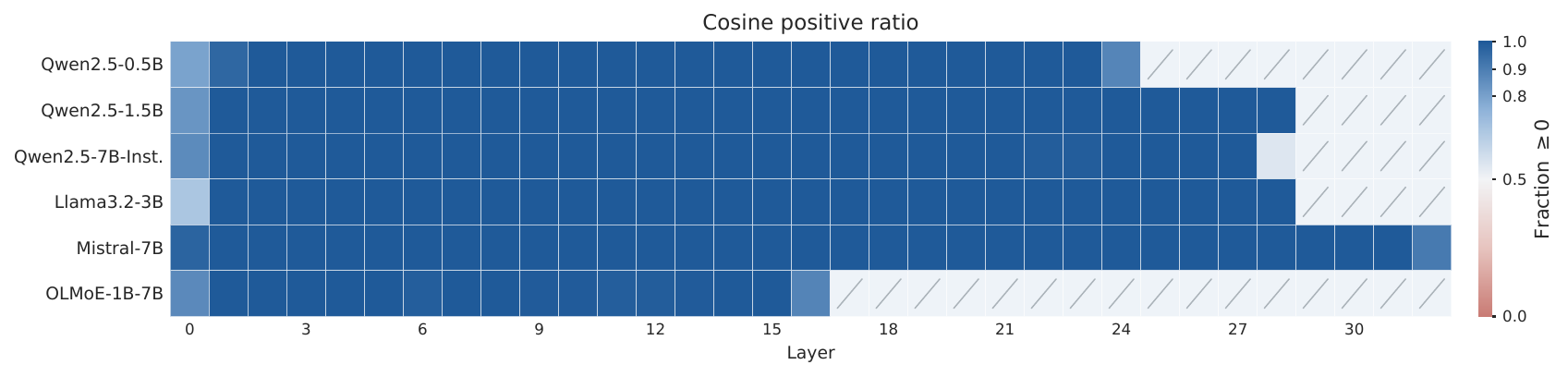} 
        \caption{Fraction of positive pairwise cosine similarities across layers.}
    \end{subfigure}
    \caption{Narrow-cone geometry of hidden representations across models.
            We compute token-level pairwise hidden-state alignment between 10 randomly sampled \texttt{GSM8K} examples and 10 \texttt{MMLU} examples, leading to roughly 5k $u$-$o$ pairs. For this analysis, we randomly sample approximately 5k token pairs from the full set of token pairs induced by the selected example pairs. We evaluate \texttt{Qwen} models across multiple scales, together with representative dense \texttt{Llama}, \texttt{Mistral} and MoE \texttt{OLMoE} architectures. 
            (a) Distributions of pairwise cosine similarity $\langle \vh_i,\vh_j\rangle/\|\vh_i\|_2\|\vh_j\|$ at representative layers, with the vertical line marking zero. (b) Fraction of pairwise cosine similarities that are non-negative at each layer; blank cells denote unavailable layers. Positive alignment dominates across model families, scales, and architectures, supporting a pervasive narrow-cone structure in the residual representations.}
    \label{fig:app:hh}
\end{figure}

\section{More on Data Attribution}
\label{app:attribution}

This appendix provides additional experimental details for the data-attribution studies in \cref{sec:attribution}. 
We first give the full results contributing to \cref{tab:selection_main} in the main context, including the details of those experimental settings.
We then provide further details on the controlled cross-lingual experiment used to distinguish the information captured by \texttt{CH1} and \texttt{CH2}, followed by the training setup for subset-based fine-tuning.

\subsection{Experimental Details for Data Selection}
\label{app:attribution:exp_details}
\textbf{Controlled influential-data identification.}
We follow prior data-attribution benchmarks and consider three settings: sentence transformation, mathematical problems without reasoning annotations, and mathematical problems with reasoning annotations \citep{deng2026value}. 
Each dataset contains 10 task classes with 100 examples per class, divided into 90 candidate examples and 10 downstream examples. 
The classes correspond to different transformation types for sentence transformation and different arithmetic operations for the mathematical tasks.
For each downstream example, candidates from the same class are treated as relevant examples.
We rank all candidates using each attribution score and evaluate the resulting rankings using AUC and top-$K$ recall.

We compare against three representative families of attribution methods. 
\texttt{Embedding} ranks examples using cosine similarity between their mean-pooled final-layer representations.
\texttt{DataInf} \citep{kwon2024datainf} and \texttt{HyperINF} \citep{zhou2024hyperinf} are influence-function-based methods that approximate the inverse-curvature term using different structural assumptions. 
\texttt{LESS} \citep{xia2024less} instead uses gradient alignment computed in a LoRA parameter subspace, providing a lower-cost approximation to full-parameter gradient interaction. 
In contrast, our \texttt{CH1} and \texttt{CH1+2} scores are computed directly from the forward quantities defined in Eq.~\ref{eq:score}, without per-example backpropagation.
\vspace{-10pt}
\begin{table}[h]
\centering
\caption{
Influential-data identification and efficiency on \texttt{Qwen2.5-1.5B} across three controlled benchmarks (5 seeds). Backward indicates whether per-example backpropagation is required, and Complexity reports the asymptotic attribution cost. Here, \(n\) is the number of candidate-target pairs, \(d\) the hidden dimension, \(d_{\mathrm{in}}\) the layer-input dimension, \(L\) the number of layers, and \(d_{\mathrm{proj}}\) the projected-gradient dimension used by LESS. Complexity denotes the pairwise scoring cost after caching per-example forward quantities; one-time feature construction costs are excluded.
}
\label{tab:ch1_res}
\resizebox{\textwidth}{!}{
\begin{tabular}{c*{4}{cc}}
\toprule
\multirow{2}{*}{\textbf{Method}}
& \multicolumn{2}{c}{\textbf{Sentence transformations}}
& \multicolumn{2}{c}{\textbf{Math problems (w/o reasoning)}}
& \multicolumn{2}{c}{\textbf{Math problems (w/ reasoning)}} 
& \multicolumn{2}{c}{\textbf{Efficiency}}\\
\cmidrule(lr){2-3}
\cmidrule(lr){4-5}
\cmidrule(lr){6-7}
\cmidrule(lr){8-9}
& AUC $\uparrow$
& Recall $\uparrow$
& AUC $\uparrow$
& Recall $\uparrow$
& AUC $\uparrow$
& Recall $\uparrow$ 
& Backward
& Complexity \\
\midrule

\texttt{Random}
& $0.50 $
& $0.10 $
& $0.50 $
& $0.10 $
& $0.50 $
& $0.10 $ 
& $\times$
& $\mathcal{O}(n)$\\

\texttt{Embd}
& $0.546 \pm 0.306$
& $0.148 \pm 0.205$
& $0.555 \pm 0.298$
& $0.146 \pm 0.295$
& $0.560 \pm 0.310$
& $0.198 \pm 0.311$ 
& $\times$
& $\mathcal{O}(nd)$\\

\dfv
& $0.981 \pm 0.019$
& $0.826 \pm 0.121$
& $0.985 \pm 0.032$
& $0.878 \pm 0.154$
& $0.987 \pm 0.030$
& $0.892 \pm 0.155$ 
& $\checkmark$
& $\mathcal{O}(nd_{in}dL)$\\

\hypi
& ${0.993 \pm 0.013}$
& ${0.934 \pm 0.063}$
& ${0.986 \pm 0.024}$
& ${0.942 \pm 0.080}$
& ${0.988 \pm 0.023}$
& ${0.950 \pm 0.060}$ 
& $\checkmark$
& $\mathcal{O}(nd^3L)$\\

\texttt{LESS}
& $0.785 \pm 0.096$
& $0.370 \pm 0.139$
& $0.835 \pm 0.235$
& $0.592 \pm 0.291$
& $0.829 \pm 0.172$
& $0.524 \pm 0.350$ 
& $\checkmark$
& $\mathcal{O}(nd_\text{proj})$\\

 \texttt{CH1}
&  $\mathbf{1.000 \pm 0.001}$
&  $\mathbf{0.989 \pm 0.025}$
&  $\mathbf{1.000 \pm 0.000}$
&  ${0.998 \pm 0.011}$
&  $\mathbf{1.000 \pm 0.000}$
&  ${0.998 \pm 0.008}$ 
& $\times$
& $\mathcal{O}(nd)$\\

 \texttt{CH1+2}
&  $ {0.998 \pm 0.006}$
&  $ {0.963 \pm 0.044}$
&  $\mathbf{1.000 \pm 0.000}$
&  $\mathbf{0.999 \pm 0.006}$
&  $\mathbf{1.000 \pm 0.000}$
&  $\mathbf{1.000 \pm 0.000}$ 
& $\times$
& $\mathcal{O}(ndL)$\\

\bottomrule
\end{tabular}%
}
\vspace{-4mm}
\end{table}

\subsection{Controlled Cross-lingual Attribution}
\label{app:attribution:control_multi}

The cross-lingual experiment in \cref{tab:selection_main} is designed to reduce direct vocabulary overlap while preserving the underlying correspondence between examples. 
We consider two domains. For mathematical reasoning, we sample the first 50 training examples from \texttt{GSM8K}. 
For medical knowledge, we sample the first 10 test examples from each of four \texttt{MMLU} subjects: College Biology, Clinical Knowledge, College Medicine, and Medical Genetics. 
Each English example is translated into Chinese, French, Korean, and Spanish, producing four semantically corresponding candidates with different forms.

For each English target, all translated examples from the same domain are pooled into the candidate set. 
Each attribution method ranks this pool, and we retrieve the top four candidates. 
Since every English source has exactly four translated counterparts, attribution accuracy is defined as the fraction of the retrieved examples that originate from the same English source as the target. Under random ranking, the expected accuracy is \(1/N\), where \(N\) is the number of English source examples in the corresponding domain.
\begin{table}[h]
\centering
\caption{Cross-lingual data attribution accuracy on \texttt{MMLU} and \texttt{GSM8K} across different models.}
\label{tab:data_attr_crosslingual}
\resizebox{\linewidth}{!}{
\begin{tabular}{ccccccc}
\toprule
\multirow{2}{*}{Method} 
& \multicolumn{2}{c}{Qwen2.5 1.5B-Instruction} 
& \multicolumn{2}{c}{Qwen3-4B-Instruction} 
& \multicolumn{2}{c}{Llama3.2-3B-Instruction}  \\
\cmidrule(lr){2-3} \cmidrule(lr){4-5} \cmidrule(lr){6-7}
& MMLU Acc. & GSM8K Acc. 
& MMLU Acc. & GSM8K Acc. 
& MMLU Acc. & GSM8K Acc. \\
\midrule

\texttt{Random} & 0.025 & 0.020 & 0.025 & 0.020 & 0.025 & 0.020 \\
\texttt{Embd} & 0.125 & 0.085 & 0.131 & 0.150 & 0.475 & 0.370 \\
\hypi & \textbf{0.694} & 0.370 & 0.356 & 0.450 & 0.606 & 0.645 \\
\dfv & 0.419 & 0.070 & 0.138 & 0.115 & 0.313 & 0.175 \\
\texttt{LESS}  & 0.469 & 0.140 & 0.188 & 0.140 & 0.319 & 0.355 \\
 \texttt{CH1} & 0.488 & 0.445 & 0.450 & 0.689 & 0.506 & 0.830 \\
 \texttt{CH1+2} & 0.625 & \textbf{0.600} & \textbf{0.525} & \textbf{0.745} & \textbf{0.625} & \textbf{0.930} \\
\bottomrule
\end{tabular}
}
\end{table}

\textbf{Reduced prediction-vocabulary overlap.}
To verify that the multilingual construction indeed weakens the direct alignment available to \texttt{CH1}, we measure the overlap between the top-K prediction tokens of each English example and its translations. 
As shown in \cref{app:fig:data_overlap}, the overlap remains low across languages, with the average Jaccard similarity below $0.15$. 
The experiment therefore provides a controlled setting in which semantic correspondence is preserved while direct prediction-vocabulary overlap is substantially reduced. 
This makes it possible to test whether the additional coupling captured by \texttt{CH2} contributes attribution information beyond direct token alignment.

\textbf{Multilingual agent-memory retrieval.}
We additionally evaluate the same scores in an external-memory selection setting.
For each model, we run inference on the first 200 \texttt{GSM8K} examples and retain those that are answered incorrectly, yielding 62 hard queries for \texttt{Qwen2.5-1.5B-Instruct}, 51 for \texttt{Qwen3-1.7B}, and 36 for \texttt{Llama3.2-3B-Instruct}. 
For every hard query, its corresponding question-answer pair is translated into Chinese, French, Korean, and Spanish, and all translated examples are pooled into a shared multilingual memory bank.

\begin{wrapfigure}{r}{0.3\linewidth}
    \centering
    \vspace{-5mm}
    \includegraphics[width=\linewidth]{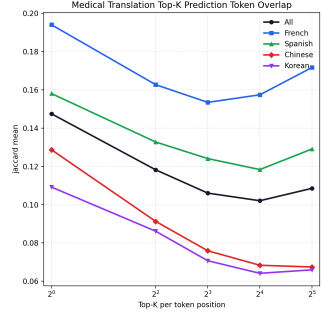}
    \vspace{-5mm}
    \caption{Top-K prediction token overlap between English and others.}
    \label{app:fig:data_overlap}
    \vspace{-4mm}
\end{wrapfigure}

At evaluation time, the original English question is used as the query.
Each method ranks the entries in the memory bank and retrieves the highest-scoring question-answer pair, which is then provided as additional context when answering the original query. 
We report two metrics. \textbf{Top-1 retrieval accuracy} measures whether the retrieved memory originates from the same underlying English example as the query, while \textbf{answer accuracy} measures whether the model answers the original question correctly after conditioning on the retrieved memory.

Random retrieval serves as a useful reference because even an unrelated \texttt{GSM8K} example can provide a valid reasoning and answer-format demonstration. 
We therefore interpret improvements over random retrieval as the additional benefit of selecting query-relevant experiences.
Embedding retrieval uses mean-pooled hidden-state similarity, whereas our scores retain token-wise learning-dynamics interactions when ranking candidate memories.
\begin{table}[h]
\centering
\caption{Multilingual agent-memory retrieval results on \texttt{GSM8K}.}
\label{tab:data_attr_crosslingual_agent}
\resizebox{1\linewidth}{!}{
\begin{tabular}{ccccccccc}
\toprule
\multirow{2}{*}{Method}
& \multicolumn{2}{c}{Qwen2.5-1.5B-Instruction} & \multicolumn{2}{c}{Qwen3-1.7B} & \multicolumn{2}{c}{Llama3.2-1B-Instruction}  
& \multicolumn{2}{c}{Llama3.2-3B-Instruction} \\
\cmidrule(lr){2-3}
\cmidrule(lr){4-5}
\cmidrule(lr){6-7}
\cmidrule(lr){8-9}
& Top-1 Retrieval Acc.
& Answer Acc.
& Top-1 Retrieval Acc.
& Answer Acc. 
& Top-1 Retrieval Acc.
& Answer Acc. 
& Top-1 Retrieval Acc.
& Answer Acc. \\
\midrule
\texttt{Base}
& -- & 0.000
& -- & 0.000
& -- & 0.000
& -- & 0.000 \\

\texttt{Random}
& 0.016 & 0.389
& 0.000 & 0.373
& 0.010 & 0.137
& 0.028 & 0.500 \\

\texttt{Embd}
& 0.339 & 0.532
& 0.157 & 0.353
& 0.248 & 0.206
& 0.861 & 0.889 \\

 \texttt{CH1}
 &  \textbf{0.484} &  \textbf{0.597}
  &  \textbf{0.431} &  \textbf{0.569}
    &  0.765 &  0.363
 &  0.861 &  0.889 \\

 \texttt{CH1+2}
 & 0.468  & \textbf{0.597}
   &  {0.392} &  {0.549}
    &  \textbf{0.902} &  \textbf{0.382}
 &  \textbf{1.000} &  \textbf{0.944} \\
\bottomrule
\end{tabular}
}
\end{table}

\subsection{Fine-tuning with Selected Data}
\label{app:attribution:finetune}
We further evaluate whether the attribution rankings translate into useful training subsets by fine-tuning on selected \texttt{GSM8K} examples. 
For each model, we compute the attribution score over the \texttt{GSM8K} training set, rank the candidate examples, and retain either the top 1\%, top 5\% or top 10\% according to each selection method. 
We compare \texttt{CH1} and \texttt{CH1+2} against random selection and \texttt{LESS}, in terms of both accuracy and GPU hours in \cref{tab:gsm8k_5epoch}.
We compute the scores after a 100-step in-distribution warm start on the probing data, while downstream fine-tuning starts from the original base checkpoint.

All selected subsets are used to fine-tune the corresponding base model for five epochs with a learning rate of \(1\times10^{-5}\). 
The same training setup is used across selection methods for each model.
For the random baseline, we independently sample the corresponding fraction of the training set with three random seeds and report the average downstream accuracy. 
The base-model \texttt{GSM8K} accuracy and training-set perplexity are also reported in the table to provide context for the model's initial fit to the target distribution.
As discussed in Section~\ref{sec:attribution}, the relative advantage of \texttt{CH1} and \texttt{CH1+2} varies across model-budget settings, with \texttt{CH2} sometimes providing additional gains.
We treat this model-dependent difference as an empirical observation rather than a general relationship between model scale or perplexity and the relative advantage of the two channels.

\begin{table*}[h]
\centering
\small
\setlength{\tabcolsep}{4.2pt}
\caption{
Closed-loop results.
Test accuracy after full-parameter fine-tuning on the top 1\%, 5\%, or 10\% selected examples. We reserve 500 examples from the \texttt{GSM8K} training split as a fixed probing set for score computation, leaving 6,973 examples as the candidate pool for selection and downstream fine-tuning.
Accuracy is reported on the $[0,1]$ scale as mean $\pm$ standard deviation
over three training seeds.
The corresponding base-model accuracies are 0.6194 for \texttt{Qwen2.5-1.5B}
and 0.1069 for \texttt{Llama-3.2-3B}.
GPU h denotes aggregate A100 GPU-hours for score computation, including all method-specific setup such as warm-up and gradient extraction; downstream fine-tuning is excluded.
}
\label{tab:gsm8k_5epoch}
\resizebox{\textwidth}{!}{
\begin{tabular}{ccccc|cccc}
\toprule
& \multicolumn{4}{c|}{Qwen2.5-1.5B}
& \multicolumn{4}{c}{Llama-3.2-3B} \\
\cmidrule(lr){2-5}
\cmidrule(lr){6-9}
Method
& 1\% (69) & 5\% (348) & 10\% (697) & GPU h
& 1\% (69) & 5\% (348) & 10\% (697) & GPU h \\
\midrule

\texttt{Random}
& 0.646$\pm$0.004
& 0.610$\pm$0.007
& 0.608$\pm$0.005
& --
& 0.2290$\pm$0.010
& 0.295$\pm$0.012
& 0.326$\pm$0.016
& -- \\

\texttt{LESS}
& 0.639$\pm$0.007
& 0.587$\pm$0.004
& 0.597$\pm$0.007
& 2.65
& \textbf{0.248$\pm$0.005}
& 0.301$\pm$0.005
& 0.330$\pm$0.007
& 4.16 \\

\texttt{CH1}
& 0.652$\pm$0.007
& 0.628$\pm$0.002
& \textbf{0.631$\pm$0.004}
& 0.09
& 0.202$\pm$0.011
& 0.307$\pm$0.007
& \textbf{0.348$\pm$0.007}
& 0.12 \\

\texttt{CH1+2}
& \textbf{0.654$\pm$0.013}
& \textbf{0.636$\pm$0.007}
& 0.623$\pm$0.007
& 0.87
& 0.202$\pm$0.002
& \textbf{0.313$\pm$0.006}
& 0.341$\pm$0.002
& 3.12 \\

\bottomrule
\end{tabular}
}
\end{table*}

\section{More On Forgetting}
\label{app:forgetting}

\subsection{Update Energy under Different Finetuning Objectives}
\label{app:forgetting:gkd_opd}

The collision analysis in \cref{sec:forgetting} also provides a common perspective on several finetuning objectives that are empirically more stable than standard off-policy SFT. 
Their mechanisms differ, but each reduces the chance of producing extremely large update forces.

\textbf{Generalized knowledge distillation} \citep[GKD,][]{agarwal2024policy}.
For GKD, the token-wise forward-KL objective can be written as
\[
    \mathcal{L}_\text{GKD}=\sum_u\mathsf{KL}(\pi_\text{teacher} || \pi_\theta)\propto -\sum_u\pi_\text{teacher}(\cdot\mid s_u)^\top\log\pi_\theta(\cdot\mid s_u).
\]
Compared with SFT, the one-hot target $\ve_{y_u}$ is replaced by the teacher distribution, giving
\[
    g_u^\text{GKD}=\pi_\text{teacher}(\cdot\mid s_u)-\pi_\theta(\cdot\mid s_u).
\]
Because the teacher target is generally softer than a one-hot label, the resulting force is less likely to enter the extreme-energy confident-conflict regime.

\textbf{On-policy distillation} \citep[OPD,][]{lu2025onpolicydistillation}.
OPD instead minimizes the reverse KL,
\[
    \mathcal{L}_\text{OPD}=\mathsf{KL}(\pi_\theta(\cdot\mid s_u) || \pi_\text{teacher}(\cdot\mid s_u))=\mathbb{E}_{y_u\sim\pi_\theta(\cdot\mid s_u)}[\log\pi_\theta(y_u\mid s_u)-\log\pi_\text{teacher}(y_u\mid s_u)].
\]
Define the log-ratio
\[
    A(y,s_u)\triangleq\log \frac{\pi_\theta(y\mid s_u)}{\pi_\text{teacher}(y\mid s_u)}.
\]
Treating the teacher as fixed, the gradient of the reverse KL can be written as
\[\begin{aligned}
\nabla_\theta \mathcal{L}_\text{OPD}
    &=\nabla_\theta\left(\sum_y{\pi_\theta}(y\mid s_u) A(y,s_u)\right)\\
    &=\sum_y\left(\pi_\theta(y\mid s_u) \nabla_\theta A(y,s_u) +  \mathsf{sg}(A(y,s_u)) \nabla_\theta \pi_\theta(y\mid s_u)\right)\\
    &=\sum_y\left( \pi_\theta(y\mid s_u) \nabla_\theta \log\pi_\theta(y\mid s_u) + \mathsf{sg}(A(y,s_u)) \nabla_\theta\pi_\theta(y\mid s_u)\right)\\
    &=\sum_y\left( \pi_\theta(y\mid s_u) \nabla_\theta \log\pi_\theta(y\mid s_u) +  \pi_\theta(y\mid s_u)\mathsf{sg}(A(y,s_u))\nabla_\theta\log\pi_\theta(y\mid s_u)\right)\\
    &=\mathbb{E}_{y\sim\pi_\theta(\cdot\mid s_u)}[\nabla_\theta\log\pi_\theta(y\mid s_u)] + \mathbb{E}_{y\sim\pi_\theta(\cdot\mid s_u)}[\mathsf{sg}(A(y,s_u))\nabla_\theta\log\pi_\theta(y\mid s_u)]\\
    &=\mathbb{E}_{y\sim\pi_\theta(\cdot\mid s_u)}[\mathsf{sg}(A(y,s_u))\nabla_\theta\log\pi_\theta(y\mid s_u)]
\end{aligned}\]
where the last equation comes from the fact that 
\[
\mathbb{E}_{y\sim\pi_\theta(\cdot\mid s_u)}[\nabla_\theta\log\pi_\theta(y\mid s_u)]=\nabla_\theta\sum_y\pi_\theta(y\mid s_u)=\nabla_\theta 1 = 0.
\]

Therefore, $\nabla_\theta \mathcal{L}_\text{OPD}=\mathbb{E}_{y\sim\pi_\theta(\cdot\mid s_u)}[\mathsf{sg}(A(y,s_u))\nabla_\theta\log\pi_\theta(y\mid s_u)]$, which has a similar form as the SFT loss, and the scalar $A(y,s_u)$ plays a role of effective learning rate (can be positive or negative).

Since gradient descent follows $-\nabla_\theta\mathcal{L}_\text{OPD}$, and our SFT force $g_u=\ve_{y_u}-\pi_\theta(\cdot\mid s_u)$ corresponds to the ascent direction of $\log\pi_\theta(y_u\mid s_u)$, the effective output-side force becomes
\[
    g_u^\text{OPD} = -A_u g_u.
\]
Accordingly, its effective token-wise energy is
\[
    E_u^\text{OPD} = \|g_u^\text{OPD}\|_2^2 = A_u^2\|g_u\|_2^2.
\]
The first factor and the second factor are controlled in different ways.
Because $y_u$ is sampled on-policy from $\pi_\theta$, tokens with small $\pi_\theta(y_u\mid s_u)$, and hence very large raw update energy $\|g_u\|_2^2$, are rarely sampled. 
This naturally suppresses many of the confident conflicts common in off-policy SFT. 
The remaining risk comes from the multiplier $A_u$: a large student-teacher likelihood ratio can amplify an otherwise moderate update. 
This provides a complementary explanation for why OPD becomes more sensitive when the student and teacher distributions are poorly matched \citep{li2026rethinking}.

\textbf{Label smoothing} \citep[LS,][]{muller2019does} also controls the same quantity more explicitly by replacing the one-hot target with $\ve_{y_u}\rightarrow (1-\epsilon)\ve_{y_u}+\epsilon\vq_{\mathrm{unif}}$,
where $\vq_{\mathrm{unif}} = \frac{1}{V}\mathbf{1}.$ is the uniform distribution over the vocabulary (or subset of it, as used in \cite{peng2025beyond}). 
The corresponding update force becomes
\[
    g_u^\text{LS} = (1-\epsilon)\ve_{y_u}+\epsilon\vq_{\mathrm{unif}} - \pi_\theta(\cdot\mid s_u).
\]
Its maximum energy is
\[
    \max_{\pi}\|g_u^\text{LS}\|_2^2\approx 2-2\epsilon+\epsilon^2(1-1/V).
\]
For the common choice $\epsilon=0.1$ and a sufficiently large vocabulary, this upper bound is approximately 1.81, explicitly excluding the most extreme collision regime discussed in \cref{sec:forgetting}.

\cref{tab:objective_energy} summarizes these objectives. 
Standard off-policy SFT combines one-hot supervision with potentially low-probability targets and can therefore generate forces approaching the maximum energy of 2. 
On-policy sampling, soft teacher targets, and label smoothing reduce this risk in different ways. 
This analysis primarily concerns the output-side force $g_u$.
How the corresponding contexts $s_u$ route these updates through the kernel is a separate question, which becomes central to the erosion mechanism studied in \cref{sec:forgetting}.

\begin{table}[h]
\centering
\caption{Comparison of update-force characteristics under different finetuning objectives.}
\label{tab:objective_energy}
\small
    \begin{tabular}{cccc}
    \toprule
    Objective & Update force & Supervision / sampling & Risk of extreme energy \\
    \midrule
    \texttt{SFT} 
    & $\ve_{y_u}-\pi_\theta$ 
    & Off-policy, one-hot 
    & Possible \\
    
    \texttt{RL} 
    & $\ve_{y_u}-\pi_\theta$ 
    & On-policy 
    & Uncommon \\
    
    \texttt{GKD} 
    & $\pi_{\mathrm{teacher}}-\pi_\theta$ 
    & Soft teacher target 
    & Suppressed \\
    
    \texttt{OPD} 
    & $A_u(\ve_{y_u}-\pi_\theta)$ 
    & On-policy 
    & Suppressed by sampling; scaled by $A_u$ \\
    
   \texttt{LS} 
    & $(1-\epsilon)\ve_{y_u}+\epsilon \vq_{\mathrm{unif}}-\pi_\theta$ 
    & Smoothed target 
    & Bounded by $\epsilon$ \\
    \bottomrule
    \end{tabular}
\end{table}

\subsection{Additional Energy-Control Experiments on OpenMathInstruct-2}
\label{app:forgetting:energy}

\begin{figure}[h]
    \centering
    \includegraphics[width=1\linewidth]{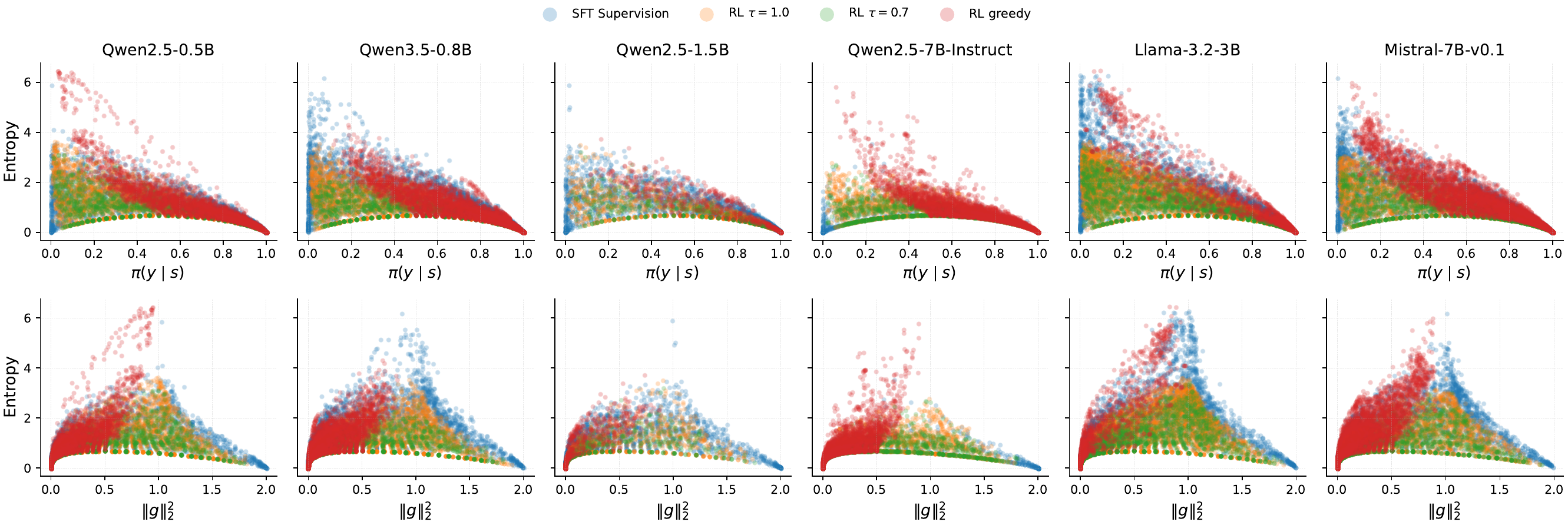}
    \caption{Cross-model validation of the entropy–energy geometry in \Cref{fig:results_collision_erosion}(a).
            Token entropy as a function of the target-token probability $\pi(y\mid s)$ (top) and gradient energy $\|g\|_2^2$ (bottom) across six models and four supervision settings. The characteristic geometry observed in the main text persists across model families and scales, further supporting gradient energy as a more intrinsic characterization of token-level uncertainty than probability alone.}
    \label{app:fig:entropy_vs_g_all}
\end{figure}

\textbf{Data and experimental setup.}
We conduct additional energy-control experiments on \texttt{OpenMathInstruct-2}, a mathematical instruction-tuning corpus containing 14M problem-solution pairs whose solutions were generated by \texttt{Llama-3.1-405B-Instruct}.
We uniformly sample 50,000 examples from its deduplicated train\_1M split with a fixed seed, preserving the original mixture of synthetic MATH-style \(85.5\%\), synthetic \texttt{GSM8K}-style \(12.1\%\), original \texttt{GSM8K} \(1.2\%\), and original MATH \(1.2\%\) examples.
Exact-match decontamination against the \texttt{GSM8K} test set and \texttt{MATH-500} removes no examples. 
Unlike the \texttt{GSM8K} setting in the main text, solutions in this corpus follow a consistent response style and report their final answers using \texttt{\textbackslash boxed\{\}}, providing a substantially larger and differently formatted substrate for testing our analysis.

We evaluate \texttt{Qwen2.5-1.5B}, \texttt{Qwen3-4B}, and \texttt{Llama3.2-3B}, all in their instruction-tuned versions.
All runs use full-parameter fine-tuning for one epoch with AdamW, cosine decay, a warmup ratio of \(0.1\), global batch size \(16\), bf16 precision, and a maximum sequence length of \(1{,}024\). 
Unless otherwise specified, the learning rate is \(10^{-5}\); the learning-rate sweep additionally considers \(2\times10^{-5}\) and \(5\times10^{-5}\). EAFT uses \(\alpha=1.0\). 
Under this common setup, we compare standard SFT, EAFT, and direct control of high-energy updates.

\textbf{Threshold sensitivity.}
We additionally sweep the threshold of $E_u$ to examine the trade-off between retention and downstream learning. 
Moderate thresholds consistently improve retention while preserving most of the learning signal, whereas overly aggressive masking eventually degrades downstream performance by removing useful updates together with harmful collisions. This confirms that the benefit does not come from indiscriminately reducing the number of gradient updates, but from selectively suppressing the extreme-energy tail. Bolded indicates best, underlined indicates second.

\begin{table}[h]
\centering
\caption{Downstream learning and general-capability retention under different \texttt{OpenMathInstruct-2} finetuning methods.
        We compare standard SFT, EAFT, and energy-controlled variants. 
        \texttt{GSM8K} reports generative accuracy, \texttt{MMLU} likelihood-based multiple-choice accuracy, \texttt{IFEval} prompt-level strict accuracy, and \texttt{Dolly classification/closed-QA} Rouge-L F1.
        Energy-controlled rows differ only in the masking threshold $E_u$.}
\label{tab:app_omi2_main}
\resizebox{0.9\linewidth}{!}{\begin{tabular}{ccccccc}
\hline
Model & Method & GSM8K & MMLU & IFEval & Dolly-CLS & Dolly-QA \\ \hline
\multirow{6}{*}{Qwen2.5-1.5B-Instruct}
 & base & 68.46 & 60.12 & 42.14 & 38.77 & 40.68 \\
 & SFT  & 70.20 & 58.88 & 26.99 & 28.78 & 34.12 \\
 & EAFT & 72.18 & 58.86 & \underline{27.36} & 27.46 & 33.68 \\
 & Mask $E_u{>}1.8$ & 71.42 & \underline{58.89} & 26.80 & \underline{27.71} & 33.92 \\
 & Mask $E_u{>}1.5$ & 68.31 & \textbf{59.03} & \textbf{28.65} & \textbf{29.68} & \textbf{36.98} \\
 & Mask $E_u{>}1.0$ & 63.46 & 58.98 & 26.25 & 29.17 & \underline{34.33} \\ \hline
\multirow{6}{*}{Qwen3-4B-Instruct}
 & base & 85.21 & 70.67 & 83.18 & 17.26 & 33.19 \\
 & SFT  & 89.23 & 69.94 & 69.87 & 34.55 & 37.55 \\
 & EAFT & 89.31 & \textbf{70.03} & \textbf{74.49} & 36.04 & 37.58 \\
 & Mask $E_u{>}1.8$ & 88.48 & 69.69 & \underline{72.64} & 34.24 & 37.79 \\
 & Mask $E_u{>}1.5$ & 87.79 & \underline{69.72} & 68.58 & \textbf{36.78} & \textbf{41.38} \\
 & Mask $E_u{>}1.0$ & 79.45 & 69.18 & 71.16 & \underline{37.91} & \underline{41.20} \\ \hline
\multirow{6}{*}{Llama3.2-3B-Instruct}
 & base & 62.55 & 62.28 & 71.16 & 31.43 & 40.32 \\
 & SFT  & 76.80 & 60.66 & 58.23 & 30.24 & 40.62 \\
 & EAFT & 76.50 & \textbf{60.4}5 & 59.89 & \underline{30.31} & \textbf{41.26} \\
 & Mask $E_u{>}1.8$ & 76.57 & \underline{60.23} & 58.41 & \textbf{30.75} & 40.66 \\
 & Mask $E_u{>}1.5$ & 74.22 & 60.21 & \underline{60.81} & 30.31 & \underline{41.24} \\
 & Mask $E_u{>}1.0$ & 62.24 & 59.02 & \textbf{61.00} & 24.33 & 39.32 \\ \hline
\end{tabular}}
\end{table}

\textbf{Larger updates amplify the benefit of energy control.}
We also test whether the relative benefit of suppressing high-energy updates changes with the update magnitude. 
On \texttt{Qwen2.5-1.5B} trained on \texttt{OpenMathInstruct-2}, we compare SFT and EAFT across learning rates $10^{-5}$, $2\times 10^{-5}$, and $5\times10^{-5}$.
As the learning rate increases, the EAFT-minus-SFT retention gap moves consistently in EAFT's favor across \texttt{MMLU}, \texttt{IFEval}, and the two \texttt{Dolly} subsets. 
For example, the \texttt{MMLU} gap increases from \(-0.02\) to \(+0.53\) and $+1.10$ points across the three learning rates, while the \texttt{IFEval} gap increases from $+0.37$ to $+0.74$ and $+0.93$.

This trend is consistent with the collision mechanism: when individual updates are larger, extreme-energy conflicts become more consequential, increasing the relative value of suppressing them. 
We treat this result as supporting rather than isolating evidence, since changing the learning rate scales all update components rather than energy alone.

\begin{table}[h]
\centering
\caption{EAFT $-$ SFT difference (percentage points) on \texttt{Qwen2.5-1.5B-Instruct} trained on \texttt{OpenMathInstruct-2}, across three learning rates. Positive values favour \texttt{EAFT}. The gap increases monotonically with learning rate on all four retention metrics, and \texttt{EAFT} is ahead on every metric only at the largest learning rate: at $10^{-5}$ it is behind on \texttt{MMLU}, \texttt{Dolly-CLS}, and \texttt{Dolly-QA}. The benefit of suppressing high-energy updates therefore grows with update magnitude rather than being uniform, consistent with the collision mechanism of \cref{sec:forgetting}. Since changing the learning rate scales all update components rather than energy alone, we treat this as supporting rather than isolating evidence.}
    \label{tab:app_omi2_lr}
    \begin{tabular}{cccc}
    \toprule
    metric / learning rate & $10^{-5}$ & $2{\times}10^{-5}$ & $5{\times}10^{-5}$ \\
    \midrule
    \texttt{MMLU}       & $-0.02$ & $+0.53$ & $+1.10$ \\
    \texttt{IFEval}     & $+0.37$ & $+0.74$ & $+0.93$ \\
    \texttt{Dolly-CLS}  & $-1.32$ & $-0.74$ & $+0.38$ \\
    \texttt{Dolly-QA}  & $-0.44$ & $+0.11$ & $+1.75$ \\
    \bottomrule
    \end{tabular}
\end{table}

\subsection{Additional Evidence for Erosion on OpenMathInstruct-2}
\label{app:forgetting:erosion_if}

The \texttt{OpenMathInstruct-2} experiments also provide an independent test of the erosion mechanism in \cref{sec:forgetting}. 
Unlike the \texttt{GSM8K} setting in the main text, this dataset contains substantially more training examples and uses a different response convention, with final answers reported in \texttt{\textbackslash boxed\{\}}. 
If erosion reflects accumulated behavioral drift rather than a \texttt{GSM8K}-specific artifact, we should therefore expect the model to drift toward this new response pattern after fine-tuning.

\textbf{Generative MMLU evaluation.}
We evaluate the fine-tuned models on the full \texttt{MMLU} test set using the following template. 

\begin{tcolorbox}[colback=black!3,colframe=black!40,boxrule=0.4pt,arc=1mm,left=1mm,right=1mm,top=0.5mm,bottom=0.5mm]
\ttfamily\small
Question:\\
Please only return A/B/C/D.\\
A: \{choice$_A$\}; B: \{choice$_B$\}; C: \{choice$_C$\}; D: \{choice$_D$\}\\
Answer:
\end{tcolorbox}

Each question explicitly instructs the model to return only A/B/C/D, and decoding is greedy.
We report five complementary measurements in \cref{app:tab:app_omi2_erosion_full}. 
\textbf{Non-IF} is the fraction of generations that do not directly follow the requested A/B/C/D format. 
\textbf{Boxed} measures the fraction containing \texttt{\textbackslash boxed\{\}}, the characteristic final-answer convention of \texttt{OpenMathInstruct-2}. 
\(P(A\ldots D)\triangleq\sum_{c \in \{A,B,C,D\}} \pi_\theta(c \mid s_o)\) is shorthand for the first-token probability mass \(\pi_\theta(\mathcal C\mid s_o)\). 
\textbf{Acc.} is extraction-based answer accuracy, which allows the answer letter to be recovered even from non-compliant responses. 
Finally, \textbf{Worked} denotes generations longer than 200 characters, which we use as a simple indicator that the model produces a worked mathematical solution rather than directly returning an answer.

The last measurement is useful for distinguishing two levels of behavioral transfer. 
An increase in Boxed alone could indicate that the model has merely copied a new final-answer marker.
In contrast, an increase in Worked indicates that the drift extends to the broader response pattern of producing a mathematical derivation before the final answer. 
Although the 200-character threshold is only a coarse heuristic, it provides a simple diagnostic of this more substantial behavioral change.
Consistent with this interpretation, Worked rises sharply after SFT across all three models, from 4.0\% to 75.6\% on \texttt{Qwen-1.5B}, 1.8\% to 36.6\% on \texttt{Qwen-4B}, and 4.1\% to 50.1\% on \texttt{Llama-3B}. 
This shows that the transferred behavior often extends beyond the final-answer marker to the broader solution-generation pattern.
We treat Worked only as a supporting diagnostic rather than a continuous measure of erosion, since its fixed length threshold can saturate or vary with response verbosity.

\begin{table}[h]
\centering
\caption{Generative \texttt{MMLU} evaluation after \texttt{OpenMathInstruct-2} fine-tuning. We compare base, SFT, and EAFT across three models, together            with a learning-rate sweep on \texttt{Qwen2.5-1.5B}. LR ×2 and LR ×5 denote two and five times the default learning rate, respectively.}
\label{app:tab:app_omi2_erosion_full}
    \resizebox{0.85\linewidth}{!}{
    \begin{tabular}{cccccccc}
    \hline
Model & Method & Non-IF & Boxed & P(A...D) & Acc. & Worked \\ \hline
\multirow{6}{*}{Qwen2.5-1.5B}
 & base            & 20.5\%  & 1.8\%  & 0.954 & 54.5\% & 4.0\%  \\
 & SFT             & 99.9\%  & 67.6\% & 0.201 & 47.9\% & 75.6\% \\
 & EAFT            & 100.0\% & 58.1\% & 0.202 & 48.3\% & 78.1\% \\
 & Mask $E_u{>}1.8$ & 99.9\%  & 73.7\% & 0.200 & 47.0\% & 79.1\% \\
 & Mask $E_u{>}1.5$ & 99.7\%  & 60.5\% & 0.240 & 47.6\% & 69.2\% \\
 & Mask $E_u{>}1.0$ & 99.0\%  & 42.3\% & 0.247 & 46.7\% & 76.6\% \\ \hline
\multirow{6}{*}{Qwen3-4B}
 & base            & 11.4\%  & 0.0\%  & 0.981 & 69.4\% & 1.8\%  \\
 & SFT             & 63.7\%  & 53.3\% & 0.399 & 68.1\% & 36.6\% \\
 & EAFT            & 60.7\%  & 44.7\% & 0.418 & 69.2\% & 34.5\% \\
 & Mask $E_u{>}1.8$ & 65.4\%  & 46.6\% & 0.385 & 66.6\% & 46.3\% \\
 & Mask $E_u{>}1.5$ & 55.1\%  & 40.8\% & 0.444 & 66.1\% & 34.4\% \\
 & Mask $E_u{>}1.0$ & 51.5\%  & 11.9\% & 0.521 & 66.7\% & 31.6\% \\ \hline
\multirow{6}{*}{Llama3.2-3B}
 & base            & 59.9\%  & 0.0\%  & 0.962 & 55.4\% & 4.1\%  \\
 & SFT             & 88.1\%  & 21.3\% & 0.371 & 52.2\% & 50.1\% \\
 & EAFT            & 67.9\%  & 12.8\% & 0.476 & 52.1\% & 35.9\% \\
 & Mask $E_u{>}1.8$ & 75.7\%  & 15.0\% & 0.455 & 49.9\% & 41.5\% \\
 & Mask $E_u{>}1.5$ & 70.2\%  & 21.1\% & 0.506 & 50.9\% & 42.3\% \\
 & Mask $E_u{>}1.0$ & 65.9\%  & 13.7\% & 0.847 & 49.8\% & 30.8\% \\ \hline
\multirow{4}{*}{Qwen2.5-1.5B (Big LR)}
 & SFT, LR $\times 2$  & 100.0\% & 76.2\% &  0.156 & 45.6\% & 72.5\% \\
 & EAFT, LR $\times 2$ & 100.0\% & 75.8\% &  0.151 & 46.3\% & 75.4\% \\
 & SFT, LR $\times 5$  & 100.0\% & 88.3\% &  0.118 & 35.7\% & 69.7\% \\
 & EAFT, LR $\times 5$ & 100.0\% & 85.2\% & 0.111 & 37.1\% & 76.2\% \\ \hline
    \end{tabular}}
\end{table}

\begin{table}[h]
\centering
\caption{
\textbf{Response-format separation redirects erosion toward the fine-tuning format.}
We evaluate all 14,042 \texttt{MMLU} questions using both the
\texttt{Question:/Answer:} (Q/A) and \texttt{Problem:/Result:} (P/R) prompts.
Compared with standard SFT, using an alternative P/R response format during
\texttt{OpenMathInstruct-2} fine-tuning produces strongly format-dependent erosion:
behavior in the deployed Q/A format is better preserved, while interference is
redirected toward P/R.
Downstream \texttt{GSM8K} performance remains similar, indicating that format
separation redirects interference rather than simply suppressing learning.
}
\label{app:tab:omi2_format_separation}

\resizebox{\linewidth}{!}{
\begin{tabular}{ccccccccc}
\hline
Model & SFT form
& Non-IF Q/A $\downarrow$
& Non-IF P/R 
& Boxed Q/A $\downarrow$
& Boxed P/R 
& $P(A\ldots D)$ Q/A $\uparrow$
& $P(A\ldots D)$ P/R 
& GSM8K $\uparrow$ \\
\hline

\multirow{3}{*}{Qwen3-4B}
& base
& 11.4\% & 7.7\% & 0.0\% & 0.0\% & 0.980 & 0.984 & 85.21 \\
& SFT (standard)
& 63.7\% & 64.0\% & 53.3\% & 60.3\% & 0.399 & 0.388 & 89.23 \\
& SFT (P/R)
& 31.5\% & 94.9\% & 26.2\% & 92.7\% & 0.588 & 0.201 & 89.84 \\
\hline

\multirow{3}{*}{Llama3.2-3B}
& base
& 59.9\% & 55.4\% & 0.0\% & 0.0\% & 0.962 & 0.954 & 62.55 \\
& SFT (standard)
& 88.1\% & 86.2\% & 21.3\% & 26.2\% & 0.371 & 0.351 & 76.80 \\
& SFT (P/R)
& 42.5\% & 99.9\% & 7.6\% & 34.4\% & 0.743 & 0.011 & 76.88 \\
\hline

\multirow{3}{*}{Qwen2.5-1.5B}
& base
& 21.0\% & 19.6\% & 1.9\% & 0.3\% & 0.953 & 0.982 & 69.22 \\
& SFT (standard)
& 99.9\% & 99.5\% & 67.6\% & 70.0\% & 0.201 & 0.283 & 70.20 \\
& SFT (P/R)
& 99.5\% & 100.0\% & 53.0\% & 79.7\% & 0.321 & 0.163 & 70.74 \\
\hline

\end{tabular}}
\end{table}

\textbf{Fine-tuning induces drift toward the training distribution.}
As shown in \cref{app:tab:app_omi2_erosion_full}, SFT produces substantial instruction-following degradation across all three models.
At the same time, the resulting responses increasingly resemble \texttt{OpenMathInstruct-2} rather than arbitrary malformed outputs. 
For \texttt{Qwen3-4B-Instruct}, for example, Non-IF rises from \(11.4\%\) to \(63.7\%\), Boxed rises from \(0.0\%\) to \(53.3\%\), and Worked rises from \(1.8\%\) to \(36.6\%\). 
The corresponding first-token choice mass drops from \(0.981\) to \(0.399\). 
Thus, fine-tuning not only reduces the tendency to directly follow the A/B/C/D instruction, but systematically reallocates behavior toward the response convention repeatedly reinforced during training.
The same pattern appears under the other models, although its severity is model-dependent. 
In particular, \texttt{Qwen2.5-1.5B} reaches an almost completely non-compliant regime after fine-tuning, while \texttt{Llama3.2-3B} already has a relatively high Non-IF rate before fine-tuning. 
These differences motivate using continuous distributional measurements in addition to the binary generation-based criterion.

\textbf{Probability-level measurements remain informative after greedy behavior saturates.}
The learning-rate sweep on \texttt{Qwen2.5-1.5B} makes this distinction particularly clear. 
Its Non-IF rate is already approximately \(100\%\) at the default learning rate and therefore cannot reflect further degradation. 
However, the underlying output distribution continues to move monotonically away from the requested behavior as the learning rate increases.
Under SFT, \(P(\mathrm{A\ldots D})\) decreases from \(0.201\) to \(0.156\) and \(0.118\) for learning rates \(10^{-5}\), \(2\times10^{-5}\), and \(5\times10^{-5}\), respectively. 
Over the same sweep, the Boxed rate increases from \(67.6\%\) to \(76.2\%\) and \(88.3\%\).

Therefore, saturation of the greedy response does not imply that erosion has saturated internally. 
Even when nearly every generation already violates the instruction, the model can continue shifting probability mass away from the protected behavior and toward the fine-tuning distribution.
This also illustrates why probability-level measurements such as \(\pi_\theta(\mathcal C\mid s_o)\) are useful for tracking erosion when benchmark accuracy or binary generation metrics have already saturated.

\textbf{Erosion is largely dissociated from answer correctness.}
Despite these large behavioral changes, answer accuracy is considerably more stable. 
For example, \texttt{Qwen3-4B-Instruct} changes from \(69.4\%\) to \(68.1\%\) confidence-based \texttt{MMLU} accuracy under SFT, despite the large changes in Non-IF, Boxed, Worked, and first-token probability mass.

\textbf{EAFT and energy control only partially mitigates erosion.}
EAFT can improve retention in some settings, but it does not eliminate the behavioral drift. 
For example, on \texttt{Llama3.2-3B}, EAFT reduces Non-IF from \(88.1\%\) under SFT to \(67.9\%\), reduces Boxed from \(21.3\%\) to \(12.8\%\), and increases \(P(\mathrm{A\ldots D})\) from \(0.371\) to \(0.476\). 
On \texttt{Qwen3-4B}, the corresponding improvements are smaller but follow the same direction. Moreover, for the energy-control runs, as $E_u$ the energy cap decreases (from 1.8 to 1.0), instruction following also improves. For example, for \texttt{Qwen3-4B}, Non-IF decreases from \(63.7\%\) under SFT to \(55.1\%\) with masking $E_u > 1.5$, and reduces Boxed from \(53.3\%\) under SFT to \(40.8\%\) with masking $E_u > 1.5$, while maintaining downstream and retention metrics (as seen in \cref{tab:app_omi2_main}).
\texttt{Qwen2.5-1.5B} remains close to saturation under both objectives. 
These results are consistent with our distinction between collision and erosion: suppressing high-energy updates can remove some harmful interactions, but erosion can still accumulate through many individually mild updates that are not targeted by the energy criterion.

\subsection{Additional Analyses of Context Separation}
\label{app:forgetting:erosion_mitigate}

We provide additional results for the context-separation intervention in \cref{sec:forgetting}. 
We first report the raw performance underlying the changes summarized in the main text, then verify the same effect on instruction-tuned models, and finally examine how this intervention modifies \texttt{CH1} and \texttt{CH2}.

\textbf{Raw results of the controlled experiment.}
\Cref{fig:results_collision_erosion}(e) reports the before- and after-finetuning results underlying \cref{app:tab:context_separation}. 
Starting from the same base checkpoint, we first establish the \texttt{MMLU} behavior using either the \texttt{Question:/Answer:} (Q/A) or \texttt{Problem:/Result:} (P/R) format, and then fine-tune on \texttt{GSM8K} using the alternative format. 
The raw results show the same pattern as in the main text: interference is substantially stronger when \texttt{MMLU} is evaluated in the response format used by the incoming \texttt{GSM8K} updates, while downstream \texttt{GSM8K} learning remains comparable.

\begin{table*}[h]
\centering
\small
\setlength{\tabcolsep}{3.5pt}
\renewcommand{\arraystretch}{1.08}
\caption{Controlled test of context-dependent erosion under matched and mismatched response formats. 
        Starting from base models, we first establish the \texttt{MMLU} behavior using either the \texttt{Question:/Answer:} (Q/A) or \texttt{Problem:/Result:} (P/R) format, and then fine-tune on \texttt{GSM8K} using the alternative format.
        We report instruction-following accuracy on \texttt{MMLU} together with \texttt{MMLU} and \texttt{GSM8K} task accuracy before and after the second-stage fine-tuning.
}
\label{app:tab:context_separation}
\resizebox{\textwidth}{!}{
\begin{tabular}{lcccccccccccc}
\toprule
\multirow{3}{*}{\textbf{Model}}
& \multicolumn{4}{c}{\textbf{Instruction Following Acc.}}
& \multicolumn{4}{c}{\textbf{MMLU Acc.}}
& \multicolumn{4}{c}{\textbf{GSM8K Acc.}} \\
\cmidrule(lr){2-5}
\cmidrule(lr){6-9}
\cmidrule(lr){10-13}

& \multicolumn{2}{c}{\textbf{Before FT}}
& \multicolumn{2}{c}{\textbf{After FT}}
& \multicolumn{2}{c}{\textbf{Before FT}}
& \multicolumn{2}{c}{\textbf{After FT}}
& \multicolumn{2}{c}{\textbf{Before FT}}
& \multicolumn{2}{c}{\textbf{After FT}} \\
\cmidrule(lr){2-3}
\cmidrule(lr){4-5}
\cmidrule(lr){6-7}
\cmidrule(lr){8-9}
\cmidrule(lr){10-11}
\cmidrule(lr){12-13}

\textbf{Eval. format}

& Q/A & P/R
& Q/A & P/R
& Q/A & P/R
& Q/A & P/R
& Q/A & P/R
& Q/A & P/R \\
\midrule

Qwen2.5-1.5B-QA
& 100.00 & 99.85
& {100.00} & 93.13
& 59.31 & 58.41
& {58.89} & 55.73
& 8.42 & 9.25
& 39.27 & {41.62} \\

Llama-3.2-3B-QA
& 99.98 & 99.94
& {99.53} & 46.51
& 55.28 & 53.95
& {53.81} & 42.44
& 5.69 & 5.53
& 33.43 & {35.78} \\

Qwen3-4B-QA
& 100.00 & 99.97
& {99.99} & 13.30
& 71.04 & 70.74
& {70.65} & 23.81
& 22.37 & 22.29
& {67.63} & 67.10 \\

Qwen2.5-1.5B-PR
& 100.00 & 99.99
& 90.14 & {99.01}
& 59.76 & 59.76
& 56.97 & {59.04}
& 0.00 & 12.66
& 56.71 & 58.45 \\

Llama-3.2-3B-PR
& 99.99 & 99.99
& 70.89 & {99.34}
& 55.52 & 55.32
& 44.97 & {52.79}
& 0.00 & 5.46
& 45.56 & 45.72 \\

Qwen3-4B-PR
& 100.00 & 100.00
& 71.72 & {96.08}
& 71.36 & 71.15
& 58.67 & {68.93}
& 0.00 & 31.08
& 79.38 & 79.83 \\

\bottomrule
\end{tabular}
}
\vspace{-3mm}
\end{table*}

\begin{table}[h]
  \centering
  \small
  \caption{Format separation removes context-dependent erosion. Each model first
  establishes MMLU behavior in one response format, then is fine-tuned on
  \texttt{GSM8K} in either the same format (\emph{default}) or the alternative one
  (\emph{protected}). We report instruction-following accuracy on MMLU in the
  established format, i.e.\ the format the model is deployed in.}
  \label{tab:format-separation}
  \begin{tabular}{llcccc}
    \toprule
    Established & Model & After est. & Default & Protected & $\Delta$ \\
                &       &            & (FT $=$ deploy) & (FT $\neq$ deploy) & \\
    \midrule
    Q/A & Qwen2.5-1.5B & 100.00 & 97.35 & 100.00 & $+2.65$ \\
        & Llama-3.2-3B &  99.98 & 44.77 &  99.53 & $+54.76$ \\
        & Qwen3-4B     & 100.00 & 68.64 &  99.99 & $+31.35$ \\
    \midrule
    P/R & Qwen2.5-1.5B &  99.99 & 85.98 &  99.01 & $+13.03$ \\
        & Llama-3.2-3B &  99.99 & 82.35 &  99.34 & $+16.99$ \\
        & Qwen3-4B     & 100.00 & 88.35 &  96.08 & $+7.73$ \\
    \bottomrule
  \end{tabular}
\end{table}

\textbf{The same effect appears in instruction-tuned models.}
The controlled experiment above manually establishes a response-format preference, whereas instruction-tuned checkpoints already contain preferences inherited from their unknown training mixtures. 
We therefore repeat the intervention by first identifying which of Q/A and P/R each instruction-tuned model follows more reliably before \texttt{GSM8K} fine-tuning, and then comparing \texttt{GSM8K} training under the preferred and alternative formats.

As shown in \cref{app:tab:instruct_pr_qa}, the same qualitative tendency remains.
Fine-tuning through a response format that is more strongly aligned with the model's existing behavior generally produces greater instruction-following degradation than using the alternative format. 
The effect is not perfectly symmetric across models, as expected given their uncontrolled instruction-tuning histories, but it provides a complementary robustness check that the main result is not an artifact of manually constructing the initial response preference.

\begin{table}[t]
\centering
\caption{Instruction Following Accuracy comparison on MMLU under different prompt formats and \texttt{GSM8K} fine-tuning settings.
Q/A denotes the \texttt{MMLU} \texttt{Question/Answer:} format, while P/R denotes the \texttt{MMLU} \texttt{Problem/Result:} format.}
\label{app:tab:instruct_pr_qa}
\resizebox{0.8\linewidth}{!}{
\begin{tabular}{lcccccc}
\toprule
\multirow{2}{*}{\textbf{Model}}
& \multicolumn{2}{c}{\textbf{Base}}
& \multicolumn{2}{c}{\textbf{GSM8K Q/A FT}}
& \multicolumn{2}{c}{\textbf{GSM8K P/R FT}} \\
\cmidrule(lr){2-3}
\cmidrule(lr){4-5}
\cmidrule(lr){6-7}
& \textbf{Q/A} & \textbf{P/R}
& \textbf{Q/A} & \textbf{P/R}
& \textbf{Q/A} & \textbf{P/R} \\
\midrule
Qwen2.5-1.5B-Ins
& 75.27\% &  \textbf{96.14\%}
& \cellcolor{blue!10} \textbf{73.20\%} & \cellcolor{blue!10}72.52\%
& \cellcolor{red!10} \textbf{52.63\%} & \cellcolor{red!10} 31.06\% \\

Llama-3.2-3B-Ins
& 92.34\% & \textbf{97.02\%}
& \cellcolor{blue!10}\textbf{47.59\%} &\cellcolor{blue!10} 44.85\%
& \cellcolor{red!10} \textbf{44.10\%} & \cellcolor{red!10} 39.87\% \\

Qwen3-4B-Ins
& \textbf{85.89\%} & 76.93\%
& \cellcolor{red!10} 0.00\% & \cellcolor{red!10} \textbf{0.05\%}
& \cellcolor{blue!10}\textbf{32.12\%} & \cellcolor{blue!10} 0.00\% \\
\bottomrule
\end{tabular}
}
\label{tab:ins_follow_acc}
\vspace{-3mm}
\end{table}

\begin{figure}[h]
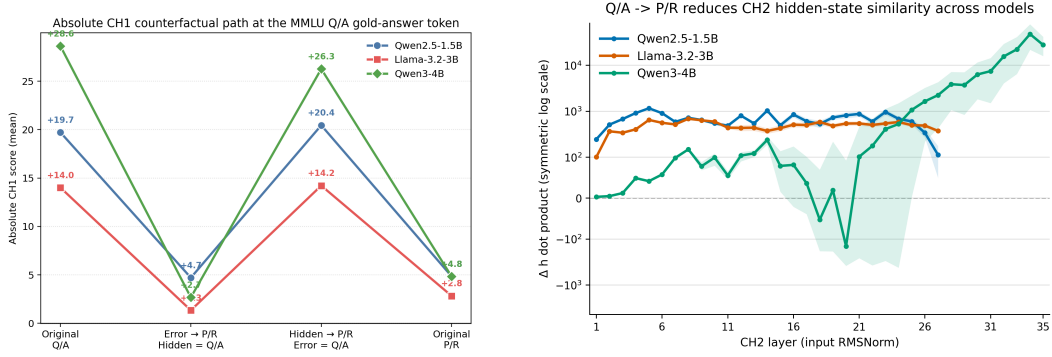

    \centering
    \begin{subfigure}[t]{0.45\linewidth}
        \centering
        \includegraphics[width=\linewidth]{figures/CH1_hgab.png}
        \caption{\textbf{Counterfactual intervention of \texttt{CH1}.} Replacing only prediction-error vector accounts for most of the \texttt{CH1} reduction, whereas replacing $\vh_{L,u}$ has little effect. Using \texttt{GSM8K} P/R format for training introduces the smallest interference on \texttt{MMLU} Q/A.}
        \label{fig:ch1_ab}
    \end{subfigure}
    \hfill
    \begin{subfigure}[t]{0.50\linewidth}
        \centering
        \includegraphics[width=\linewidth]{figures/ch2_hab.png}
        \caption{\textbf{Layer-wise change in \texttt{CH2} hidden similarity.}
                We report per layer hidden embedding similarity difference $\Delta_\ell$. Positive values indicate lower similarity under P/R.}
        \label{fig:ch2_ab}
    \end{subfigure}
    \caption{\textbf{Effects of template separation on \texttt{CH1} and \texttt{CH2}.}
The Q/A-to-P/R shift reduces \texttt{CH1} mainly through weaker prediction-error alignment and reduces \texttt{CH2} mainly through lower layer-wise hidden-state similarity.}
    \label{app:fig:ch_ab}
\end{figure}

\textbf{Template separation weakens \texttt{CH1} mainly through force alignment.}
We next examine which factors in our decomposition change when the \texttt{GSM8K} format is switched. 
For each paired \texttt{GSM8K} example, the underlying question, reasoning trace, and target answer are unchanged; only the prompt and answer-field labels differ between Q/A and P/R. 
We keep the observing \texttt{MMLU} example fixed and construct two counterfactual \texttt{CH1} scores: one replaces only the \texttt{GSM8K} update force \(g_u\), while the other replaces only its final-layer representation \(\vh_{L,u}\).

As shown in \cref{app:fig:ch_ab}-(a), replacing \(g_u^{\mathrm{Q/A}}\) with \(g_u^{\mathrm{P/R}}\) accounts for most of the reduction in \texttt{CH1} and closely approaches the score obtained when the complete P/R example is used. 
Replacing only \(h_{L,u}\) produces much less change. 
Thus, for \texttt{CH1}, format separation primarily weakens the direct vocabulary-space alignment \(g_o^\top g_u\).

\textbf{\texttt{CH2} is weakened mainly through contextual representation alignment.}
\texttt{CH2} exhibits a complementary pattern. Its output-side component is the readout-mediated force alignment $g_o^\top\vw\vw^\top g_u$.
In contrast to \texttt{CH1}, this term remains nearly unchanged under the Q/A-to-P/R shift, with its 95\% bootstrap confidence interval containing zero for all three models.
The dominant and systematic change instead appears in the layer-wise $\vh$ similarity.
\Cref{app:fig:ch_ab}-(b) reports the layer-wise difference
\[
    \Delta_\ell=\langle \tilde{\vh}_{\ell, o}^\text{Q/A}, \tilde{\vh}_{\ell, u}^\text{Q/A}\rangle - 
                \langle \tilde{\vh}_{\ell, o}^\text{Q/A}, \tilde{\vh}_{\ell, u}^\text{P/R}\rangle,
\]
where positive values indicate weaker \texttt{MMLU–GSM8K} alignment after switching the update examples from Q/A to P/R. 
The difference is predominantly positive across layers for different models. 
\texttt{Qwen3-4B} is noisier in intermediate layers, but still exhibits reduced alignment in most layers, with the separation becoming clearer toward the later layers. 
Thus, template separation attenuates \texttt{CH2} primarily through reduced contextual inner-product alignment, while leaving the readout-mediated force alignment largely unchanged.

\section{More on Plasticity Loss}
\label{app:exp_plasticity}

This appendix provides additional details and results for the plasticity-loss experiments in \cref{sec:plasticity}, including task-support statistics, sequential-SFT dynamics, cross-model results, and reset experiments. We also define the relative quantities used below. The relative readout degeneration is $D_{\mathcal D}^{t}=1-R_{\mathcal D}^{t}/R_{\mathcal D}^{0}$, and the relative AUC gain from resetting is $G_{\text{reset}}=(A_{\text{std}}-A_{\text{reset}})/A_{\text{std}}$. In \cref{fig:app:retention_last_layer}, retained plasticity is defined as $P_{\text{retain}}=A_{\text{base}}/A_\text{setting}$, where $A_\text{setting}$ is the downstream loss AUC under the corresponding long-horizon/reset setting.

\textbf{Controlled readout reshaping under sequential SFT.}
We first construct a controlled setting in which task-specific changes in $\vw\vw^\top\in\mathbb{R}^{V\times V}$ can be inspected directly.
We sequentially fine-tune on \texttt{GSM8K}, \texttt{MBPP}, and \texttt{Dolly-QA}, denoted by tasks $A$, $B$, and $C$.
To reduce overlap between their token supports, we translate \texttt{GSM8K} into Chinese and \texttt{Dolly-QA} into French while keeping \texttt{MBPP} in English.
This transformation is used only as a diagnostic device to make task-specific regions of $\vw\vw^\top$ easier to isolate; support-overlap statistics are reported in \cref{tab:app:plasticity_details}.

For each task, we select its 300 most frequent training tokens and denote the resulting support sets by $S_A,S_B,S_C$. 
We then extract the corresponding task-specific readout blocks. 
For task A, $M_A \triangleq [\vw]_{S_A}[\vw]_{S_A}^{\top} = [\vw\vw^\top]_{AA}$, with $M_B$ and $M_C$ defined analogously. 
These three $300\times300$ matrices provide a simple probe of how the shared readout geometry evolves along the token coordinates emphasized by each task.

\begin{figure}[h]
    \centering
    \includegraphics[width=0.8\linewidth]{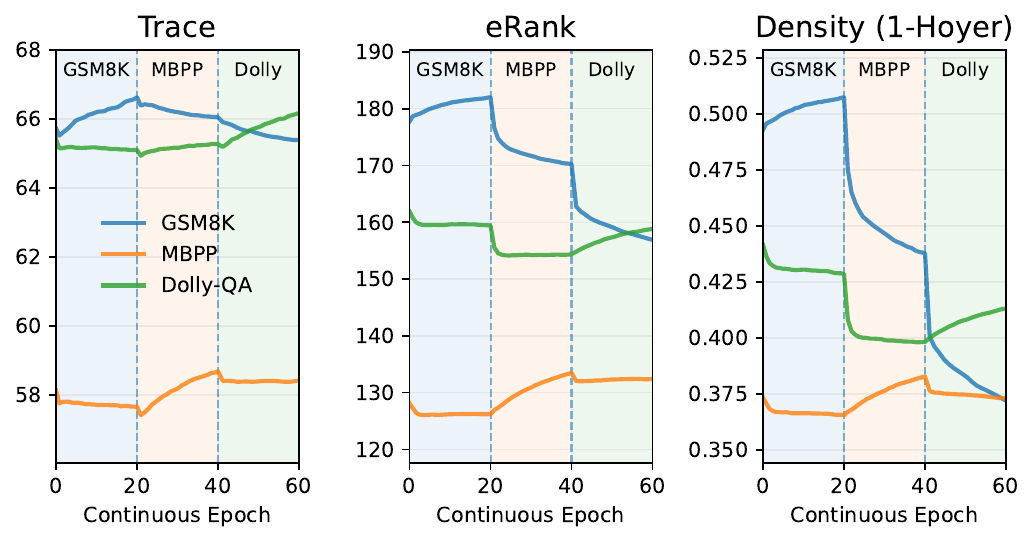}
    \caption{Sequential SFT empirically exhibits this predicted block-wise reshaping: the geometry associated with the current task is strengthened while that of the other tasks deteriorates. The model is trained sequentially on \texttt{GSM8K} (epochs 0-20), \texttt{MBPP} (20-40), and \texttt{Dolly-QA} (40-60).}
    \label{fig:app:seq_sft}
\end{figure}

We characterize each task-specific block from three complementary aspects of its spectrum:
\begin{itemize}
    \item $\mathsf{Trace}(M_A)\triangleq\sum_i [M_A]_{ii}$, which measures the total spectral mass of the block and therefore its overall transmission strength. A smaller trace indicates weaker amplification along the corresponding task support.
    \item $\mathsf{eRank}(M_A)\triangleq \exp(-\sum_i p_i\log p_i)$, where $p_i=\frac{\sigma_i}{\sum_j \sigma_j}$ and $\sigma_i$ are the eigenvalues of $M_A$. It measures the effective number of directions carrying substantial spectral mass. A smaller effective rank indicates stronger concentration into fewer directions.
    \item $\mathsf{Density}(M_A)\triangleq 1-\mathsf{Hoyer}(M_A)=\frac{{\mathsf{Trace}(M_A)}/{\sqrt{\mathsf{Trace}(M^2_A)}}-1}{\sqrt{|S_A|}-1}$. This provides a complementary sparsity-based measure of the eigenvalue spectrum. A smaller density indicates that the spectral mass is distributed over a smaller fraction of available directions.
\end{itemize}

The corresponding quantities for $M_B$ and $M_C$ are defined analogously. 
Together, $\mathsf{Trace}$ captures changes in overall spectral mass, while the scale-invariant $\mathsf{eRank}$ and $\mathsf{Density}$ distinguish genuine geometric reshaping from uniform contraction.

We sequentially fine-tune \texttt{Qwen2.5-0.5B-Instruct} on tasks A, B, and C for 20 epochs\footnote{Twenty epochs exceeds typical SFT schedules; we use it here to amplify the degeneration signal.} each using standard SFT hyperparameters.
The AdamW optimizer state, learning-rate schedule, and warm-up are reinitialized at every task transition, preventing optimizer momentum from carrying information across tasks.
Throughout training, we continuously track the three metrics for each probing data distribution mentioned above.

The results are shown in \cref{fig:app:seq_sft}. 
During training on task $A$, the statistics of $M_A$ move in the direction of stronger and broader task-specific transmission, while those of $M_B$ and $M_C$ decrease.
After switching to task $B$, the same task-dependent redistribution reappears, now favoring $M_B$; training on task $C$ produces an analogous transition.
Thus, continual fine-tuning does not reshape $\vw\vw^\top$ uniformly.
Instead, readout geometry is repeatedly redistributed toward directions emphasized by the current task, while weakly supported task directions lose spectral mass or directional coverage.

This experiment is designed primarily to isolate the predicted geometric redistribution rather than to model a realistic long-horizon continual learner.
Although later tasks already become mildly harder to fit, the induced plasticity loss remains limited.
We therefore next move to a substantially longer training trajectory, where the behavioral consequence of this geometry can accumulate.

\begin{table}[h]
\centering
\footnotesize
\begin{subtable}{\linewidth}
\centering
\begin{tabular}{l cc cc}
\toprule
& \multicolumn{2}{c}{\textbf{Original}}
& \multicolumn{2}{c}{\textbf{Multilingual}} \\
\cmidrule(lr){2-3}\cmidrule(lr){4-5}
\textbf{Task pair}
& \textbf{Intersection} & \textbf{Jaccard}
& \textbf{Intersection} & \textbf{Jaccard} \\
\midrule
GSM8K--MBPP
& 41  & 0.073
& 26  & 0.045 \\

GSM8K--Dolly-QA
& 113 & 0.232
& 18  & 0.031 \\

MBPP--Dolly-QA
& 44  & 0.079
& 39  & 0.070 \\
\midrule
All three
& 27 & --
& 15 & -- \\
\bottomrule
\end{tabular}
\caption{Pairwise overlap between the top-300 token supports of the three sequential SFT tasks.
The multilingual construction substantially reduces overlap between task-specific supports.}
\end{subtable}

\vspace{16pt}

\begin{subtable}{\linewidth}
\centering
\begin{tabular}{
    l
    p{0.18\linewidth}
    p{0.22\linewidth}
    p{0.20\linewidth}
    p{0.22\linewidth}
}
\toprule
\textbf{Setting}
& \textbf{Shared by all}
& \textbf{GSM8K-specific}
& \textbf{MBPP-specific}
& \textbf{Dolly-QA-specific} \\
\midrule

Original
&
\texttt{0, 1, 2, 3, 4, 5, 8} \newline
\texttt{., in, and, is}
&
\texttt{>>, <<, \#\#\#\#, So, \$} \newline
\texttt{Thus, If, apples, dollars}
&
\texttt{def, return, ==, [i} \newline
\texttt{range, \_list, [, arr}
&
\texttt{It, world, University, American} \newline
\texttt{New, city, where, known}
\\[2pt]

Multilingual
&
\texttt{0, 1, 2, 3, 4, 5} \newline
\texttt{., ,, -, (, )}
&
\texttt{>>, <<, \#\#\#\#} \newline
\textit{Chinese words and punctuation}
&
\texttt{return, def, if, [i} \newline
\texttt{range, ==, \_list, [}
&
\texttt{de, la, et, \`a, est, en} \newline
\texttt{des, les, dans, pour}
\\

\bottomrule
\end{tabular}
\caption{Representative tokens from the task-specific support sets before and after transformation.}
\end{subtable}

\vspace{16pt}

\begin{subtable}{\linewidth}
\centering
\begin{tabular}{l l l l}
\toprule
\textbf{Hyperparameter} & \textbf{Value}
& \textbf{Hyperparameter} & \textbf{Value} \\
\midrule
Models
& Qwen2.5-0.5B/1.5B, OLMo-1B, Llama3.2-3B
& Epochs/task
& 20 \\

Task order
& GSM8K $\rightarrow$ MBPP $\rightarrow$ Dolly-QA
& Fine-tuning
& Full-parameter \\

Learning rate
& $1\times10^{-4}$
& Optimizer
& AdamW \\

LR scheduler
& Constant + warmup
& Warmup ratio
& 0.05 \\

Weight decay
& 0.01 
& Effective batch size
& 8 \\

Batch size/device
& 1
& Gradient accumulation
& 8 \\

Max sequence length
& 512
& Precision
& BF16 \\
\bottomrule
\end{tabular}
\caption{Hyperparameters for the sequential SFT experiments.}
\end{subtable}

\vspace{16pt}

\begin{subtable}{\linewidth}
\centering
\begin{tabular}{ll|ll|ll}
\toprule
\textbf{Setting} & \textbf{Value}
& \textbf{Setting} & \textbf{Value}
& \textbf{Setting} & \textbf{Value} \\
\midrule

CPT data
& PubMed
& CPT tokens
& $\sim$100M
& CPT checkpoints
& 0/5/10/20/40/99M \\

CPT LR
& $1\times10^{-4}$
& CPT wd
& 0.1
& CPT schedule
& constant \\

CPT warmup
& 1M tokens
& SFT tasks
& GSM8K/MBPP/Dolly
& SFT epochs
& 5 \\

SFT LR
& $2\times10^{-5}$
& SFT wd
& 0
& SFT schedule
& linear \\

SFT warmup
& 5\%
& Readout reset
& base readout
& Last-4 reset
& readout + last 4 layers \\
\bottomrule
\end{tabular}
\caption{Key hyperparameters for the plasticity experiments (CPT+SFT).}
\end{subtable}
\caption{Additional experimental details for \cref{sec:plasticity}.}
\label{tab:app:plasticity_details}
\end{table}

\begin{figure}[htbp]
    \centering
    \begin{subfigure}[b]{0.85\textwidth}
        \centering
        \includegraphics[width=\textwidth]{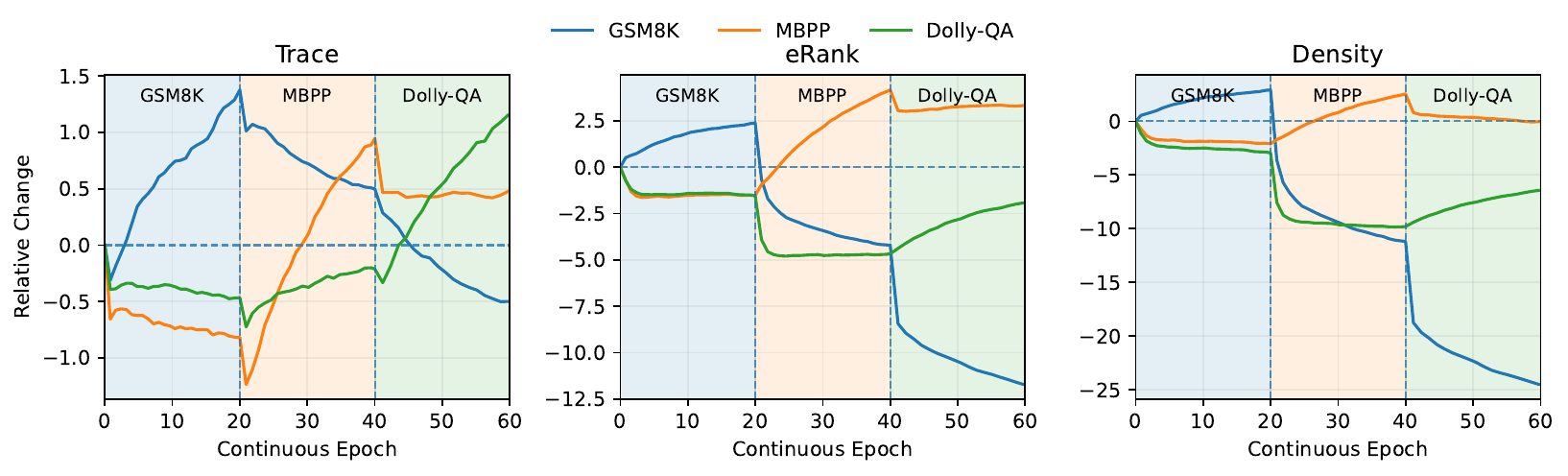} 
        \caption{\texttt{Qwen2.5-0.5B-Instruct}}
    \end{subfigure}
    \vspace{1em} 
    
    \begin{subfigure}[b]{0.85\textwidth}
        \centering
        \includegraphics[width=\textwidth]{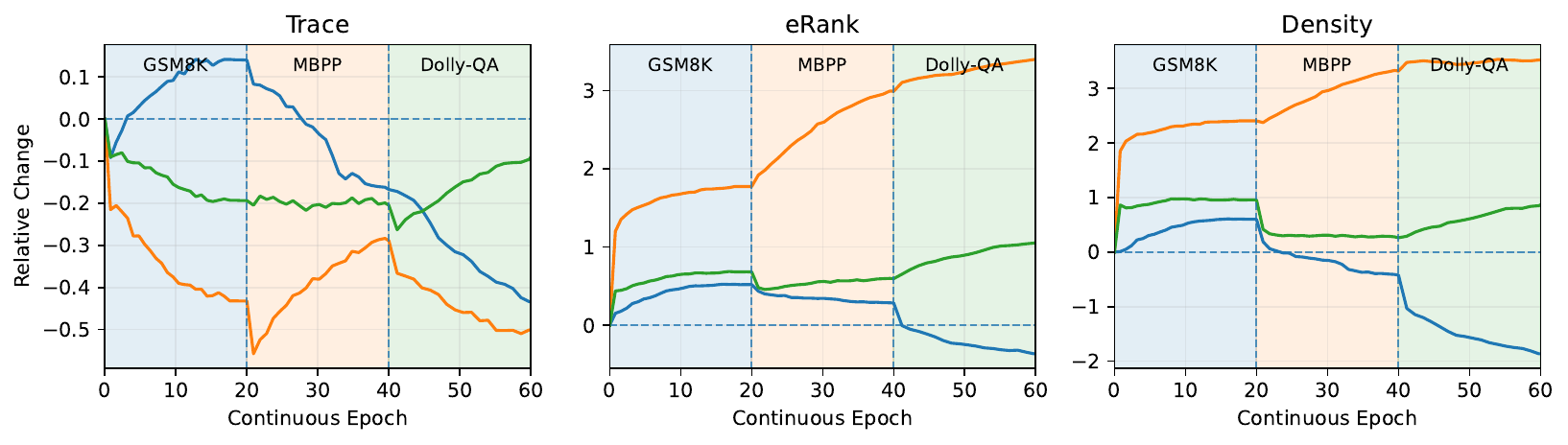} 
        \caption{\texttt{Qwen2.5-1.5B-Instruct}}
    \end{subfigure}

    \begin{subfigure}[b]{0.85\textwidth}
        \centering
        \includegraphics[width=\textwidth]{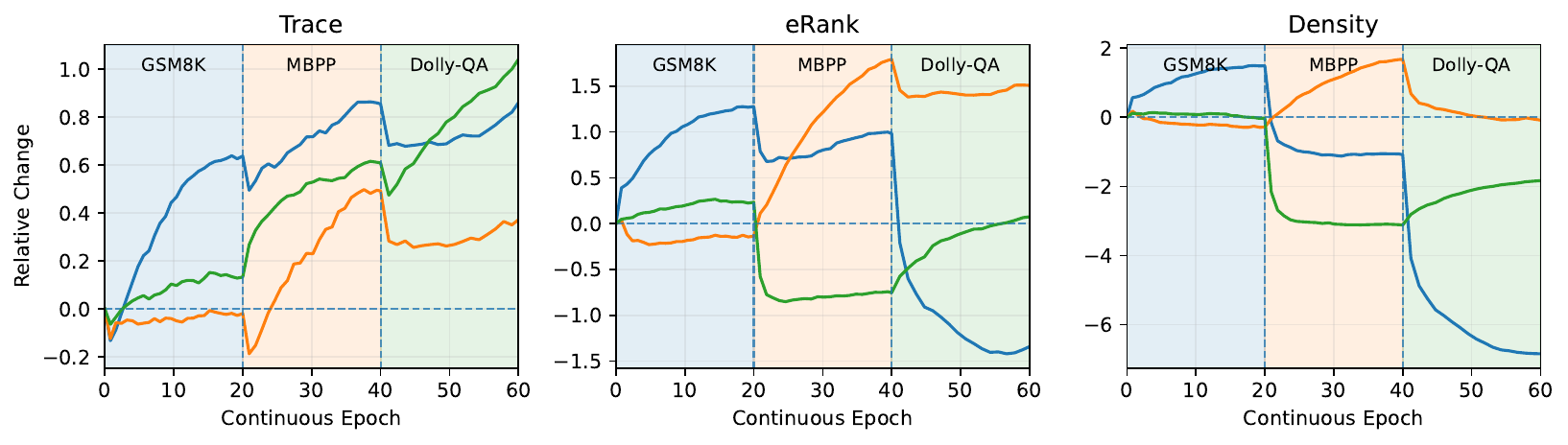} 
        \caption{\texttt{Llama3.2-3B-Instruct}}
    \end{subfigure}

    \caption{Relative evolution of task-specific readout geometry during sequential SFT across models. Each metric is reported relative to its value before training. The model is trained on \texttt{GSM8K}, \texttt{MBPP}, and \texttt{Dolly-QA} during epochs 0-20, 20-40, and 40-60, respectively. Despite model-specific differences, the overall evolution is consistent with task-dependent reshaping of the readout geometry.}
    \label{fig:app:readoutgeo}
\end{figure}

\begin{figure}[h]
    \centering
    \includegraphics[width=0.9\linewidth]{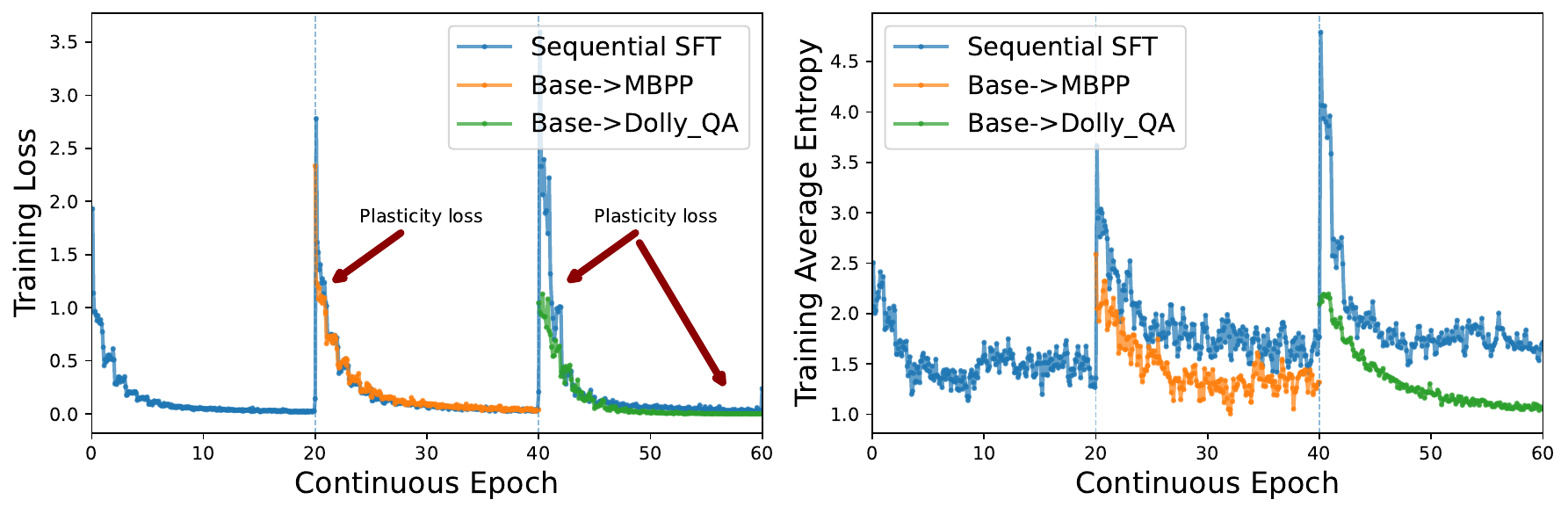}
    \caption{Training dynamics under sequential SFT, same settings as \cref{fig:plastic_theory}-(c). A \texttt{Qwen2.5-0.5B-Instruct} model is trained on \texttt{GSM8K}, \texttt{MBPP}, and \texttt{Dolly-QA} for epochs 0-20, 20-40, and 40-60, respectively. While plasticity loss is weak and difficult to distinguish from the training-loss curves alone, the average token entropy shows a pronounced deviation from training the latter tasks directly from the base model. This suggests that distribution-level learning dynamics can change substantially before the loss-level effect becomes clearly visible.}
    \label{app:fig:plasticity:train_loss}
\end{figure}

\begin{figure}
    \centering
    \includegraphics[width=1\linewidth]{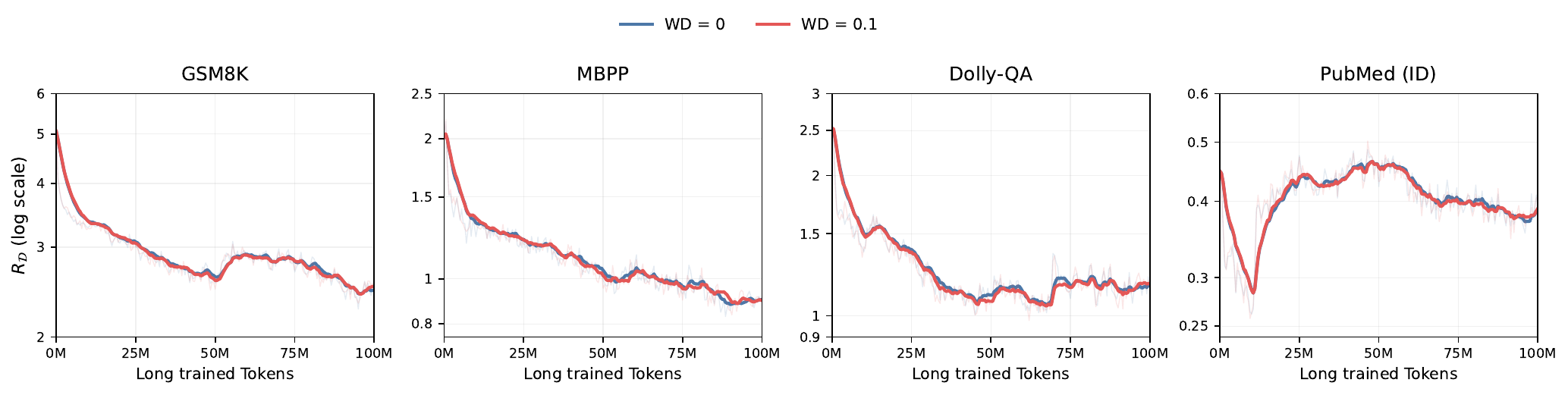}
    \caption{Readout-transmission dynamics are largely insensitive to weight decay.
            We track $R_{\mathcal D}$ during 100M-token long-horizon training on \texttt{PubMed2025} with weight decay $0$ or $0.1$, using \texttt{GSM8K}, \texttt{MBPP}, \texttt{Dolly-QA}, and held-out \texttt{PubMed} as probes.
            The two trajectories are nearly identical across all probing tasks, suggesting that the observed degeneration cannot be explained simply by uniform contraction induced by weight decay.}
    \label{fig:app:cpt_weight_decay}
\end{figure}

\begin{figure}[h]
    \centering
    \includegraphics[width=1\linewidth]{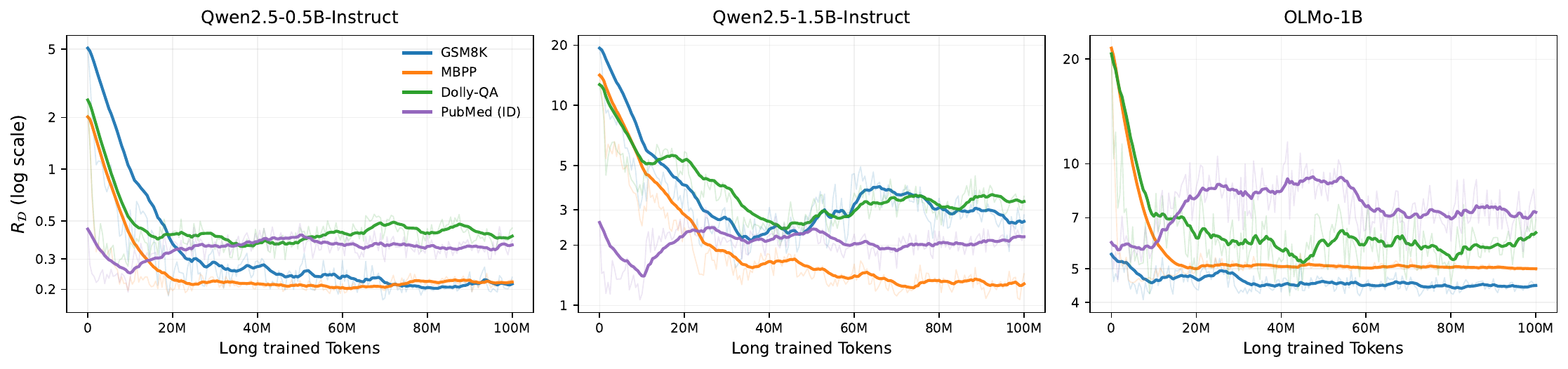}
    \caption{Evolution of $R_\mathcal{D}$ during long-horizon training for different probing tasks. The $R_\mathcal{D}$ generally decreases for the OOD tasks, while the in-distribution \texttt{PubMed} task does not show the same systematic decay. This suggests that long training progressively reshapes the shared readout geometry toward the current task distribution at the expense of gradient transmission for other tasks.}
    \label{fig:cpt_R}
\end{figure}

\begin{figure}[htbp]
    \centering
    \begin{subfigure}[b]{1\textwidth}
        \centering
        \includegraphics[width=\textwidth]{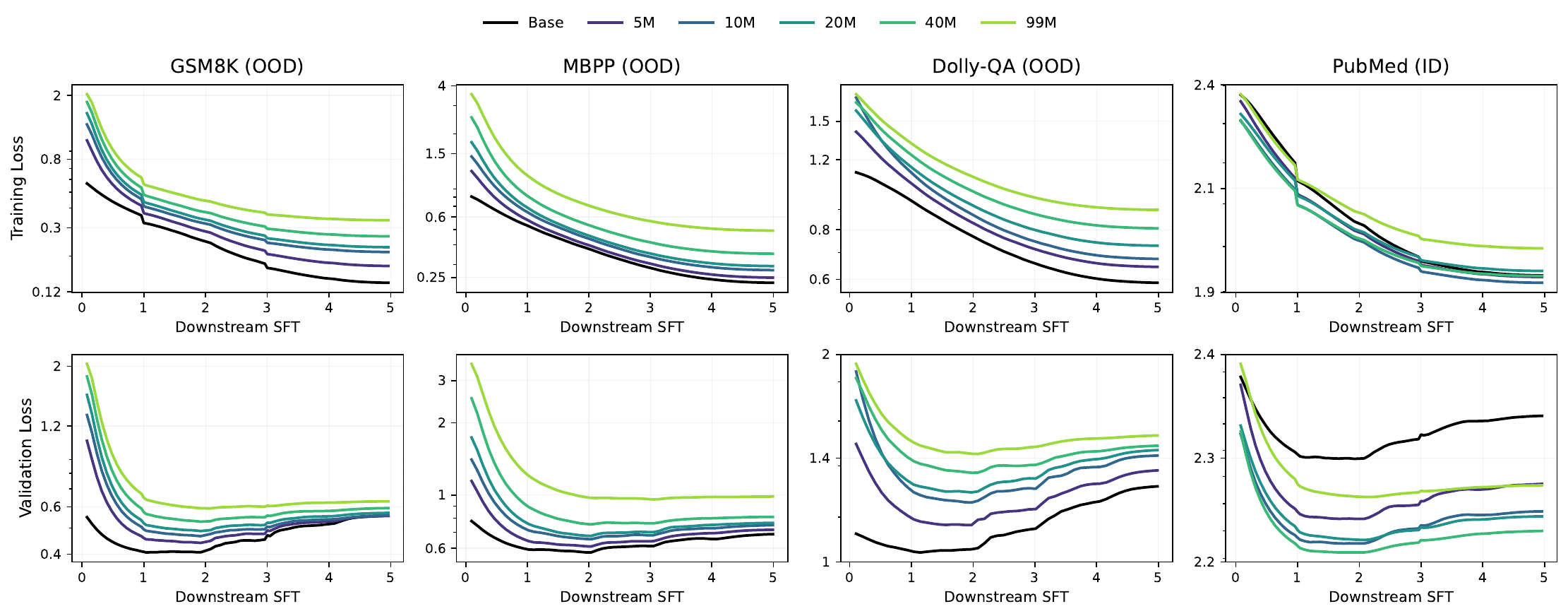} 
        \caption{\texttt{Qwen2.5-0.5B-Instruct}}
    \end{subfigure}
    
    \centering
    \begin{subfigure}[b]{1\textwidth}
        \centering
        \includegraphics[width=\textwidth]{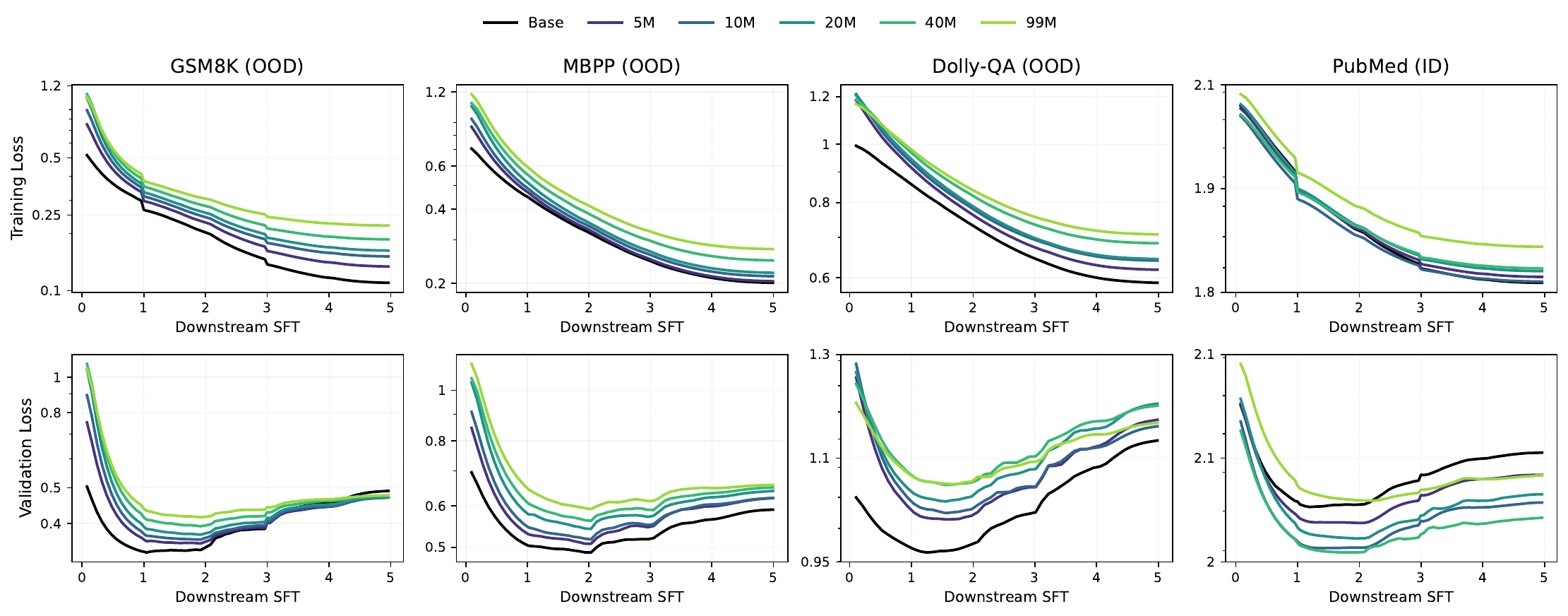} 
        \caption{\texttt{Qwen2.5-1.5B-Instruct}}
    \end{subfigure}

    \begin{subfigure}[b]{1\textwidth}
        \centering
        \includegraphics[width=\textwidth]{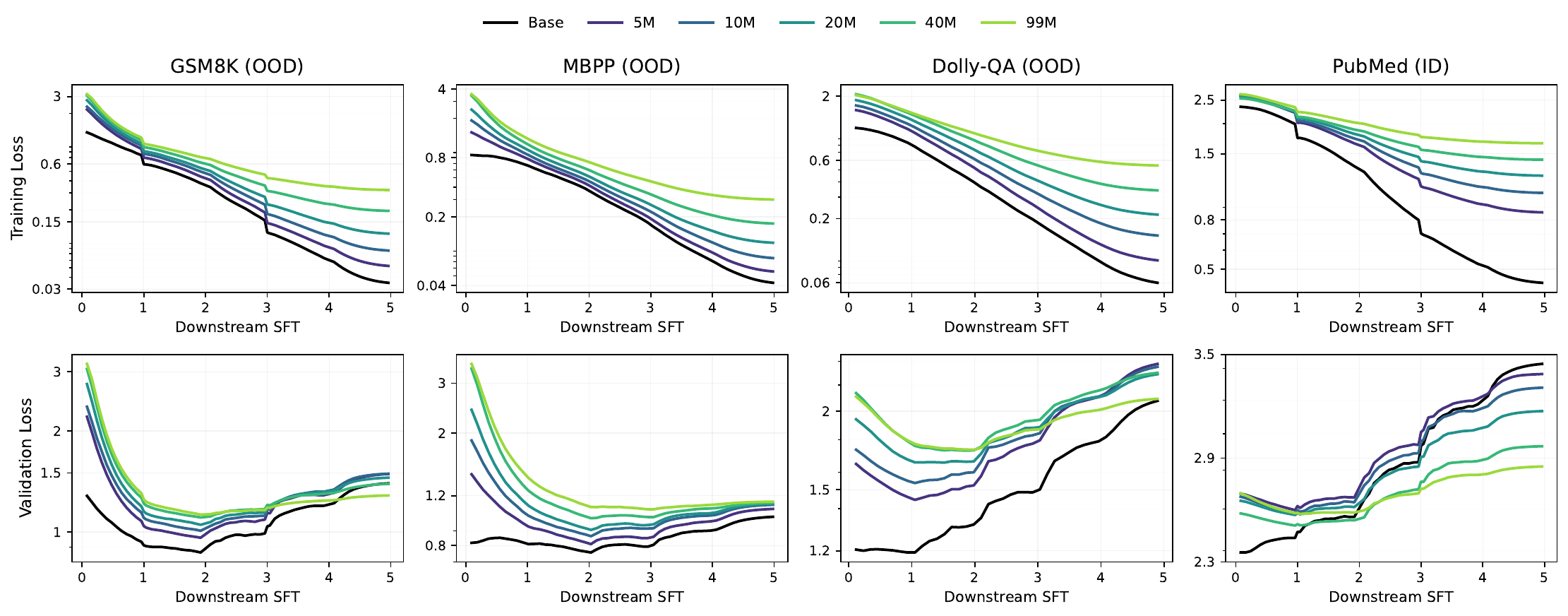} 
        \caption{\texttt{OLMo-1B}}
    \end{subfigure}

    \caption{The training and validation loss during SFT on different models.}
    \label{fig:app:ctp_training_curves}
\end{figure}

\begin{figure}[htbp]
    \centering
    \includegraphics[width=1\linewidth]{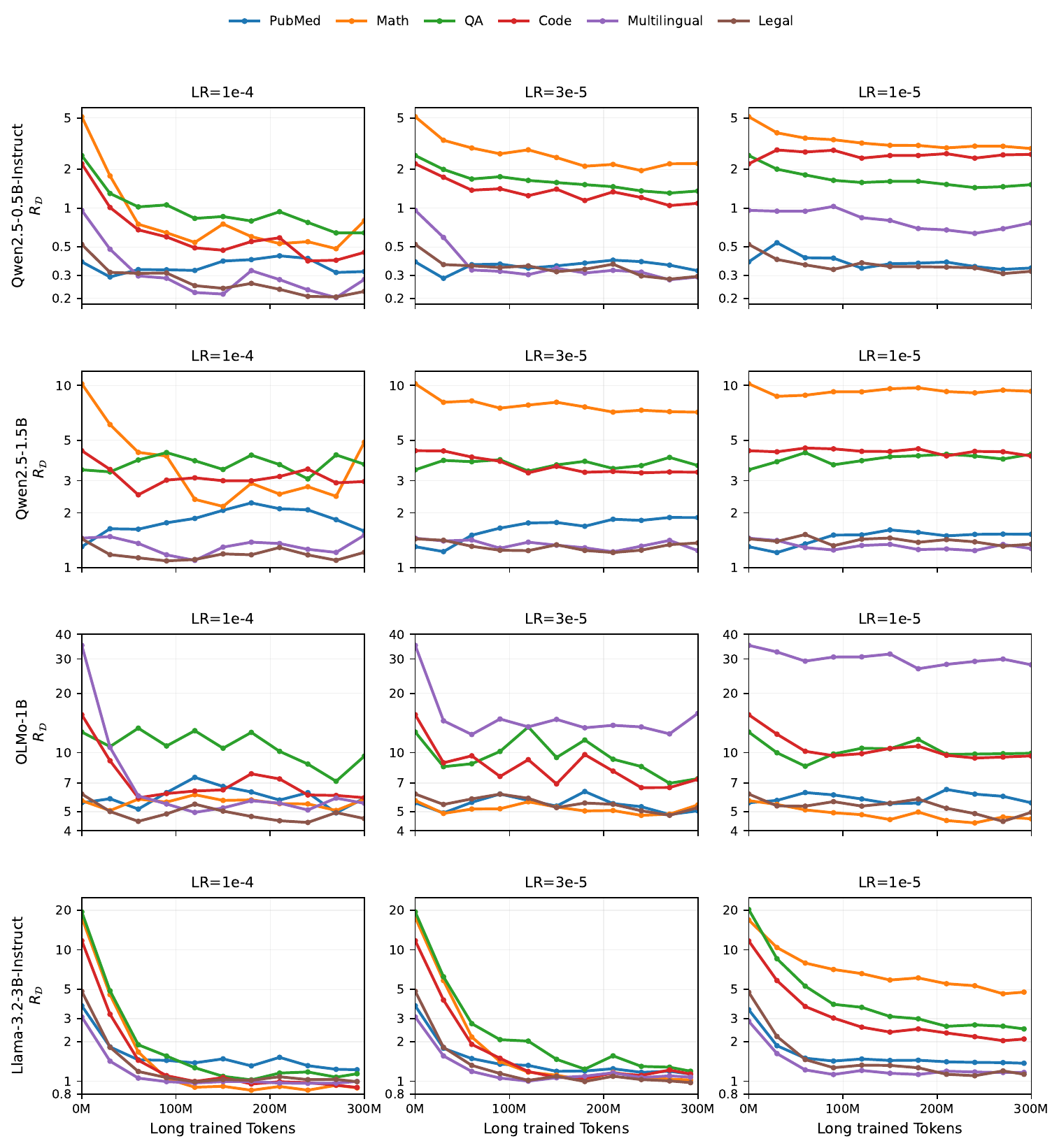}
    \caption{Readout degeneration across models, learning rates, and a broader training mixture.
            We track $R_{\mathcal D}$ during 300M-token long-horizon training on a hybrid corpus consisting of 40\% general web text, 20\% scientific/PubMed text, 15\% mathematics, 15\% code, and 10\% QA/instruction-like data, across four model families and multiple learning rates.
            The long-horizon training data are trained in causal-LM format, whereas the probing datasets use downstream task formats, such as mathematical QA, code generation, and instruction-style responses; thus, semantic domain overlap does not imply exact overlap in the token-level learning signals being probed.
            Task-conditioned transmission generally decreases during training, but the magnitude and timescale of the effect depend strongly on the model and optimization regime, with larger learning rates typically producing faster and stronger changes.
            These results show that readout degeneration persists under a broader and more heterogeneous training distribution, although its strength is highly setting-dependent.}
    \label{fig:app:cpt_R_300}
\end{figure}

\begin{figure}[htbp]
    \centering
    \begin{subfigure}[b]{0.98\textwidth}
        \centering
        \includegraphics[width=\textwidth]{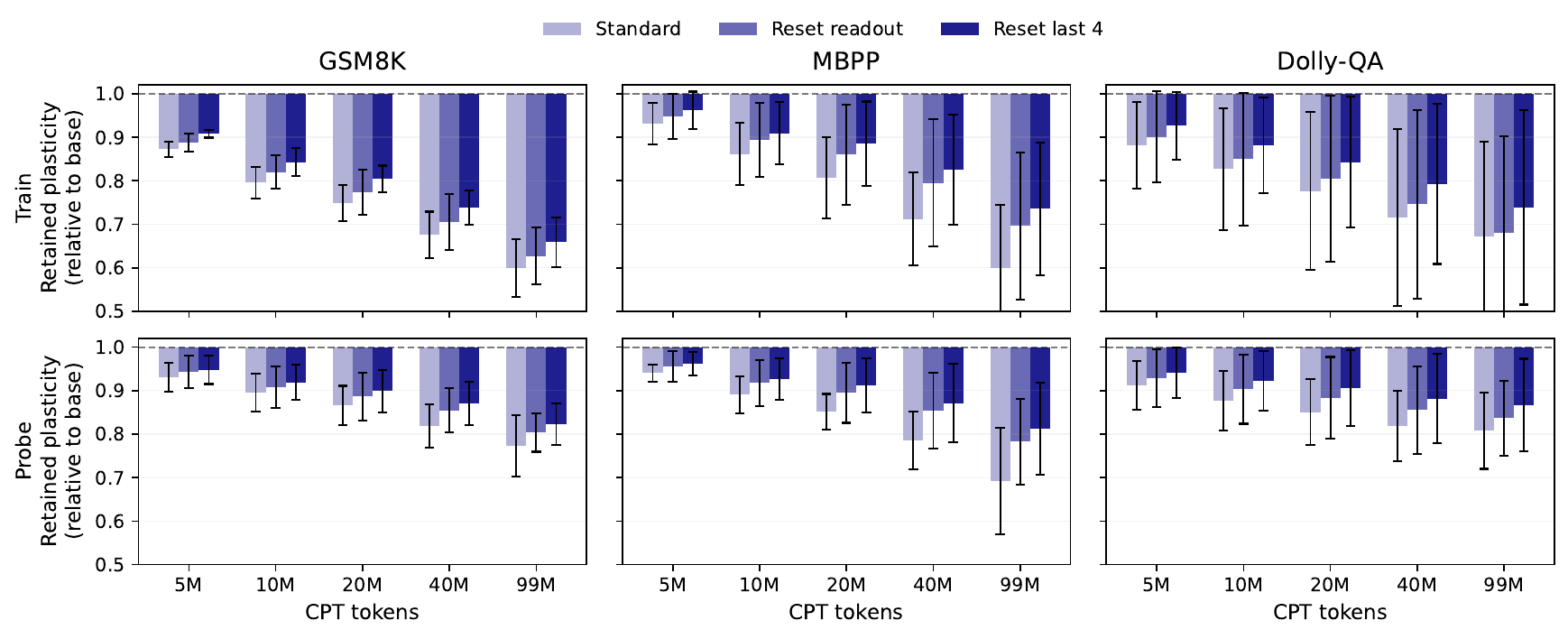} 
    \end{subfigure}
    
    \vspace{1em} 
    
    \begin{subfigure}[b]{0.98\textwidth}
        \centering
        \includegraphics[width=\textwidth]{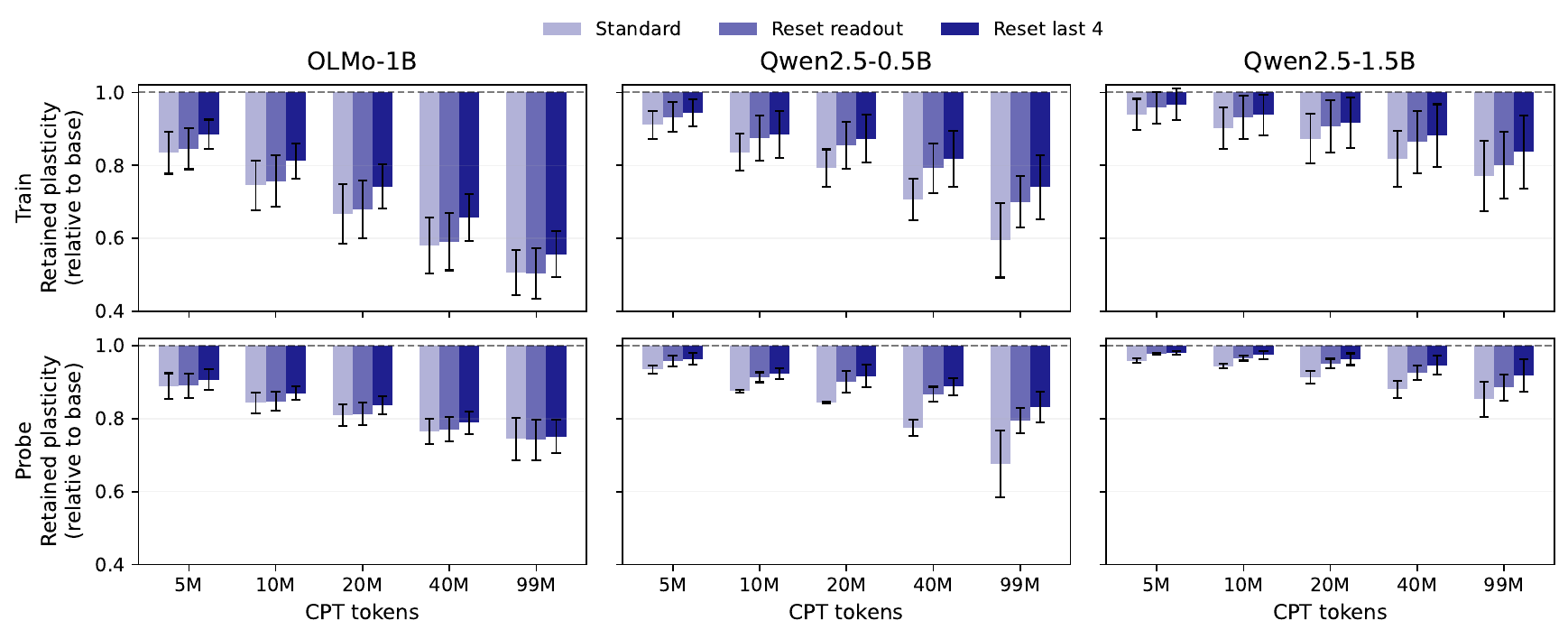} 
    \end{subfigure}

    \begin{subfigure}[b]{0.98\textwidth}
        \centering
        \includegraphics[width=\textwidth]{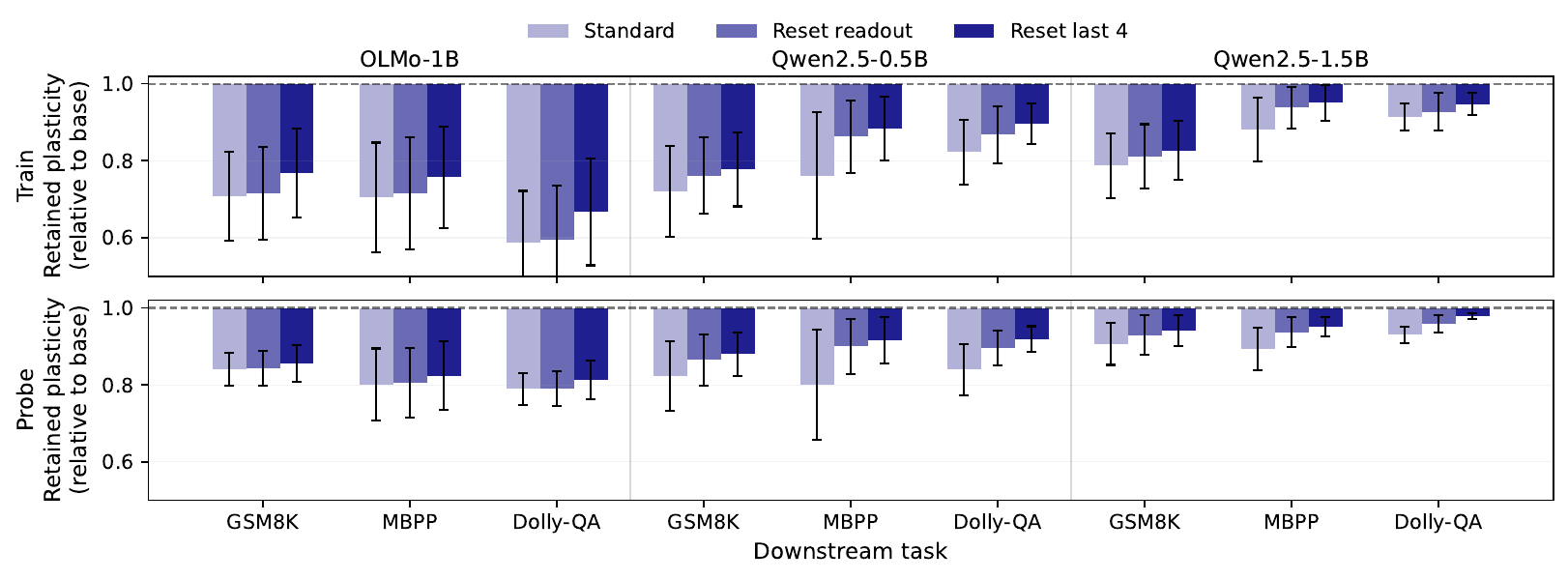} 
    \end{subfigure}
    \caption{Plasticity retention under different reset strategies (reset-last-4 means reset last four transformer blocks together with the readout), viewed under three complementary aggregations. 
    Retained plasticity is measured by downstream loss AUC relative to the corresponding base model (integrated over downstream optimizer steps), such that 1 denotes the pre-training plasticity level; lower values indicate larger plasticity loss. Bars extend downward from this base reference. We compare standard long-training, resetting the readout layer, and resetting the last four layers, with error bars denoting one standard deviation over the aggregated dimension. Top and bottom rows report training and probe-set retention, respectively. (a) Results aggregated over models, shown separately for \texttt{GSM8K}, \texttt{MBPP}, and \texttt{Dolly-QA} across different checkpoints. (b) Results aggregated over downstream tasks, shown separately for each model across different checkpoints. (c) Results aggregated over checkpoints, shown for each model-task pair. Resetting the readout generally recovers part of the lost plasticity, while resetting additional upper layers can provide further recovery; the magnitude of this effect is model- and task-dependent.}
    \label{fig:app:retention_last_layer}
\end{figure}

\end{document}